\documentclass[11pt]{article}
\usepackage[letterpaper,margin=1in]{geometry}
\usepackage{amsmath,amssymb,amsthm}
\usepackage{enumitem}
\usepackage{hyperref}
\usepackage{booktabs}
\usepackage{multirow}
\usepackage{listings}
\usepackage[table]{xcolor}

\definecolor{icalmrow}{gray}{0.94}
\newcommand{\best}[1]{\textbf{#1}}
\usepackage{graphicx}
\usepackage{caption}
\usepackage{subcaption}
\usepackage{adjustbox}
\usepackage{placeins}
\usepackage{pifont}
\usepackage{titlesec}
\usepackage{titletoc}
\usepackage{float} 
\usepackage[most]{tcolorbox}
\usepackage{bbm}
\usepackage{longtable}
\usepackage{wrapfig}
\usepackage[numbers]{natbib}
\usepackage[bottom]{footmisc}
\usepackage{tabularx}

\newtheorem{proposition}{Proposition}

\newcounter{prompt}
\renewcommand{\theprompt}{\arabic{prompt}}

\hypersetup{colorlinks=true,linkcolor=blue,citecolor=blue,urlcolor=blue}

\title{I-CALM: Incentivizing Confidence-Aware Abstention for LLM Selective Answering}

\author{
Haotian Zong$^{*\,\dagger}$ \and
Binze Li$^{*\,\ddagger}$ \and
Yufei Long$^{\ddagger}$ \and
Sinyin Chang$^{\ddagger}$ \and
Jialong Wu$^{\ddagger}$ \and
Gillian K. Hadfield$^{\S}$
}
\date{}

\begin{document}
\maketitle

\begingroup
\renewcommand{\thefootnote}{\fnsymbol{footnote}}
\footnotetext[1]{Co-first authors.}
\footnotetext[2]{Department of Applied Mathematics and Statistics, Johns Hopkins University, Baltimore, MD 21218, USA; Email: \texttt{hzong4@jh.edu}.}
\footnotetext[3]{Department of Computer Science, Johns Hopkins University, Baltimore, MD 21218, USA; Emails: \texttt{\{bli91, ylong17, schan106, jwu235\}@jh.edu}.}
\footnotetext[4]{Department of Computer Science, Johns Hopkins University, Baltimore, MD 21218, USA; School of Government and Policy, Johns Hopkins University, Washington, D.C. 20001, USA; Vector Institute for Artificial Intelligence, Toronto, ON, Canada; Email: \texttt{ghadfield@jhu.edu}.}
\endgroup

\setcounter{footnote}{0}
\renewcommand{\thefootnote}{\arabic{footnote}}

\begin{abstract}
Large language models (LLMs) often produce confident but incorrect answers, in part because standard evaluation incentives reward guessing over expressing uncertainty. We study epistemic abstention for factual questions with verifiable answers, where the goal is to improve selective answering, making LLMs abstain when they are likely to be wrong while preserving correct answers. Inspired by human behavioral decisions in question answering, we introduce I-CALM, a prompt-level framework for black-box LLMs. I-CALM combines elicited verbal confidence, announced answer/abstain payoffs, and normative guidance emphasizing truthfulness, humility, evidential support, and responsibility. To distinguish targeted abstention from indiscriminate refusal, we use a two-stage evaluation protocol, in which LLMs first choose whether to answer or abstain, and are then forced to provide a best guess for the abstained ones. Across models and factual QA datasets, I-CALM improves selective answering by reducing false-answer rate among surfaced responses and improving abstention quality, shifting error-prone cases into abstention while retaining answers the model would have answered correctly. Overall, I-CALM offers a lightweight way to improve inference-time selective answering without retraining or access to model's internal states. 
Code is available at the following \href{https://github.com/FayLONG03/hallucinationControl}{link}.
\end{abstract}

\section{Introduction}

Suppose you were injured in an accident and need to know how long you have to file a complaint for damages in your jurisdiction. You ask someone, and they answer confidently: you have a year. Did they do anything wrong?

Human societies have a surprising number of norms---and laws---governing the giving of answers. Whether that answer was wrong depends on who gave it, what it rested on, and what was at stake. A friend who has never been near a courtroom and has simply
watched a lot of lawyer shows should have said ``I don't know.'' A friend who just filed their own case in the same court might reasonably answer, especially if the deadline is looming and you cannot afford a lawyer. But your corporate lawyer friend who does billion-dollar mergers has not merely erred: answering outside their expertise, in a jurisdiction where they are not authorized to practice, may violate professional ethics rules and possibly local law. In each case, the relevant question is not just \emph{what} is the best candidate answer, but also \emph{whether that answer should be surfaced at all}.

LLMs face the same question with every factual query, but answer it far more crudely. Even strong LLMs can produce fluent and plausible false answers while expressing high apparent confidence \cite{huang2025survey,lin2022truthfulqa}. Recent work argues that hallucination is shaped not only by limitations in pretraining data, but also by post-training and evaluation incentives. In particular, standard binary evaluation often rewards guessing while providing little or no benefit for acknowledging uncertainty, thereby encouraging models to answer even when their support is weak \cite{kalai2025language}. This problem is especially salient for long-tail factual questions and questions associated with common falsehoods \cite{lin2022truthfulqa,mallen2023not}. 

We study this problem through \emph{epistemic abstention} \citep{shen2025seat}: in our setting, where every question has a verifiable answer, the model withholds a candidate answer when it is insufficiently supported to be presented as fact.
This differs from safety refusal and from abstention on queries that are themselves unanswerable, underspecified, or unsupported by the provided context \citep{kirichenko2025abstentionbench,lavi2025detecting,madhusudhan2025llms,wu2025robots}.
For a fixed LLM, the objective is therefore \emph{selective answering}: preferentially withholding candidates likely to be incorrect while retaining as many correct answers as possible.

Human decisions to withhold an answer provide a useful functional analogy. When deciding whether to assert a factual claim, people may assess whether they have sufficient evidence, consider the consequences of answering and abstaining, and apply social norms of responsible factual assertion. 
These considerations correspond to epistemic monitoring \cite{nelson1990metamemory,koriat1996monitoring}, consequence-sensitive decision making \cite{vonNeumann1947theory,kahneman1979prospect,becker1968crime}, and epistemic norms governing warranted assertion \citep{williamson1996knowing}.
We adapt this perspective to LLM abstention: abstention is a decision to avoid surfacing a factual claim when support is insufficient, the stakes of error is high, or asserting it would violate a norm of responsible assertion.

We introduce I-CALM, a prompt-level framework for confidence-aware selective answering in fixed, black-box LLMs. I-CALM organizes the answer-or-abstain decision around three functional dimensions. The \emph{epistemic} component elicits verbal confidence for the candidate answer, following established confidence-elicitation prompting \citep{tian2023just,xiong2023can}. The \emph{consequential} component announces explicit payoffs for correctness, falsehood, and abstention, building on prior risk-framing studies of answer/refuse behavior \citep{wu2025answer,wang2026llm}. The \emph{normative} component, which to our knowledge has not been examined for selective answering, builds on the human normative model to supply a compact set of principles emphasizing truthfulness, humility, attention to evidential support, and responsibility in factual communication. Neither confidence elicitation nor payoff
framing is new in isolation; our claim is that jointly framing the answer-or-abstain decision as an epistemic, consequential, and normative judgment yields \textit{better} targeted abstention than these signals alone, in a single forward pass with no weight updates, auxiliary classifiers, or access to model internals.

This positions I-CALM among inference-time interventions: each prior method addresses the answer-surfacing decision along only one or two of the dimensions above, but none addresses all of them jointly. Explicit IDK permission tells the model it may respond ``I don't know'' when uncertain \citep{cheng2024can,qin2026large}, but leaves the threshold for doing so unspecified. Uncertainty-based methods threshold candidates on verbal confidence \citep{madhusudhan2025llms} or token probabilities \citep{lin2024contextualized}, as do selective-generation and risk-control wrappers \citep{lee2024selective,oehri2025trusted,wang2025safer}; these supply an epistemic signal, but apply the decision rule post hoc rather than leaving it to the model. 
Prior work has used natural-language principles to generate fine-tuning data \citep{sun2023principle}, guide inference-time response generation and revision \citep{bell2026reflect}, and define honesty criteria for prompting and fine-tuning \citep{gao2024honestllm}; although these methods may yield refusals or disclaimers, they do not isolate normative guidance within selective answering for answerable factual QA or evaluate its effects on coverage, surfaced risk, and abstention quality.


This paper makes three contributions. 
First, we introduce and operationalize a decision-centered, natural language formulation of abstention in LLM factual question answering, decomposing the answer-versus-abstain decision into epistemic support, consequence sensitivity, and normative standards governing factual assertion. 
Across models and factual QA datasets, I-CALM improves abstention quality by 28.7\% over verbal-confidence thresholding, the strongest baseline, at matched coverage.
Second, we propose a two-stage evaluation protocol that measures abstention quality, distinguishing targeted abstention from indiscriminate coverage reduction.
Third, we characterize \textit{how} and \textit{when} I-CALM improves selective-answering behavior in fixed, black-box LLMs. 
Component ablations, sensitivity experiments across payoff schemes, and evaluation across 15 model--dataset settings show that the announced abstention payoff is the strongest selective-answering lever, confidence elicitation prevents indiscriminate over-abstention, and normative guidance further concentrates abstention on error-prone candidates.
Together, these results position I-CALM as a lightweight mechanism for controlling answer surfacing in black-box LLMs.

\section{Problem Formulation}
\label{sec:problem_protocol}

We study factual question answering with a fixed, black-box LLM.
For each question, the model must either surface a factual answer or abstain with an ``I don't know'' response.
Unlike unanswerability benchmarks, every question has a verifiable answer. 
The challenge is therefore \emph{selective answering} for a model with imperfect factual competence, in the spirit of selective prediction with a reject option \citep{elyaniv2010foundations,geifman2017selective}: withholding error-prone candidate answers while retaining as many correct ones as possible.
The resulting reduction in coverage is thus not inherently a cost---when abstention is \textit{targeted}, it is desirable; the relevant failure mode is indiscriminate abstention that also withholds correct answers.
I-CALM does not update model parameters or seek to improve underlying factual competence; it uses prompt-only inference-time decision framing to shape the answer-surfacing policy and select an appropriate selective-answering operating point on the coverage--risk curve.


\subsection{I-CALM: A Functional Decomposition of Abstention Decisions}

\begin{table}[t]
\centering
\small
\setlength{\tabcolsep}{3pt}
\renewcommand{\arraystretch}{1.05}

\begin{tabularx}{\linewidth}{
@{}
>{\raggedright\arraybackslash}p{0.19\linewidth}
>{\raggedright\arraybackslash}p{0.16\linewidth}
>{\raggedright\arraybackslash}X
>{\raggedright\arraybackslash}p{0.20\linewidth}
@{}
}
\toprule
\textbf{Human Decision Question} &
\textbf{Functional Role} &
\textbf{Conceptual Grounding} &
\textbf{I-CALM Component} \\
\midrule

Do I know enough? &
Epistemic assessment &
Metacognitive monitoring and control; strategic regulation of memory accuracy
\citep{nelson1990metamemory,koriat1996monitoring} &
Verbal confidence \\

\addlinespace[3pt]

What happens if I am wrong or abstain? &
Consequence-sensitive action selection &
Expected utility theory; behavioral decision theory; deterrence theory
\citep{vonNeumann1947theory,kahneman1979prospect,becker1968crime} &
Announced answer/abstain payoffs \\

\addlinespace[3pt]

Should I assert this as fact? &
Norm-governed factual assertion &
Epistemic norms of assertion; knowledge as a condition for warranted assertion
\citep{williamson1996knowing} &
Normative guidance \\

\bottomrule
\end{tabularx}

\caption{Functional motivation for the three I-CALM components. Each component
corresponds to a distinct role in deciding whether to surface a factual answer.}
\label{tab:functional_decomposition}
\end{table}

In ordinary communication, withholding an assertion is rarely a reflexive response to uncertainty alone.  A person typically asks whether they know enough, how costly it would be to be wrong, and whether it would be responsible to present the claim as fact.  We use this observation to motivate I-CALM as a functional decomposition of abstention decisions, and the components correspond to different roles in the decision of whether a factual answer should be surfaced. 
Table~\ref{tab:functional_decomposition} summarizes this logic.  
The following explains the conceptual motivation for each role and then describes its implementation in I-CALM.

\paragraph{Verbal Confidence.}
Classic work on metacognition distinguishes \emph{monitoring}, in which a person assesses the state of their own knowledge, from \emph{control}, in which that assessment guides subsequent behavior \citep{nelson1990metamemory}.  This distinction is especially relevant to factual reporting.  In the strategic regulation of memory accuracy, people can improve the accuracy of what they report by withholding answers for which their internal monitoring signal is weak, trading informativeness for reliability \citep{koriat1996monitoring}.  Translated to LLM selective answering, confidence is not itself an abstention policy.  It is an epistemic assessment that can support a downstream control decision about whether an answer should be surfaced.

I-CALM implements this monitoring role by asking the model to provide a verbal confidence score together with its factual answer.  
We use verbal confidence because it is available in black-box settings where model weights, hidden states, and token probabilities may be unavailable.  We do not assume that verbal confidence is a perfectly faithful readout of the model's internal uncertainty.  
Instead, we assess and treat it as a practical, elicited uncertainty signal that can be evaluated empirically and combined with natural-language decision framing (see Appendix \ref{sec:confidence_metric} for the assessment). 

\paragraph{Announced Answer/Abstain Payoffs.}
Anyone making a decision under uncertainty needs a criterion for evaluating the consequences of alternative actions. Rational actors choose the action that maximizes expected utility \citep{vonNeumann1947theory}. Behavioral decision theory shows that choices under risk are shaped by perceived gains, losses, and framing \citep{kahneman1979prospect}.  Deterrence theory similarly models behavior as responsive to expected benefits and costs, including the cost of being wrong or sanctioned \citep{becker1968crime}.  
While we do not claim that LLMs literally implement expected-utility maximization, prospect-theoretic weighting, or deterrence-theoretic reasoning, these classic theories establish the motivation that an answer-or-abstain decision depends not only on epistemic confidence, but also on the stakes of answering, making an error, and declining to answer.

I-CALM implements a role for consequences by announcing the payoffs associated with each possible outcome. In our prompt, a correct answer receives positive payoff, an incorrect answer receives negative payoff, and an ``I don't know'' response receives a positive payoff. Formally, this can be summarized as $r(\mathrm{correct})=+\alpha$, $r(\mathrm{wrong})=-\beta$, and $r(\mathrm{abstain})=+\gamma$ for $\alpha,\beta,\gamma \ge 0$, although these payoffs are only described in natural language and are not used to update model parameters.  This design makes abstention an explicit action with value, allowing us to study whether different announced stakes shift the model's selective-answering policy.

\paragraph{Normative Guidance.}

Payoffs make the consequences of answering and abstaining salient, whereas normative guidance specifies standards for when surfacing a factual claim is appropriate. 
Philosophical work on epistemic norms of assertion asks what epistemic standing licenses a factual assertion. 
On Williamson's influential knowledge-norm account, one should assert a proposition only if one knows it \citep{williamson1996knowing}.
This perspective motivates treating abstention not merely as risk avoidance, but as a normative obligation not to mislead others by making an insufficiently supported assertion.

I-CALM operationalizes this idea through a compact set of principles emphasizing truthfulness, humility, evidential support, responsibility, and clarity (Figure~\ref{fig:norms}).
These principles are not claimed to be exhaustive, universal, or optimal. 
Unlike prior prompting that directly instructs the model to use reported confidence to choose between answering and abstaining \citep{wang2026llm}, the normative component adds general standards for responsible factual assertion rather than merely specifying a confidence-to-action rule. 

\begin{figure}[ht]
  \centering
  \fbox{%
    \begin{minipage}{0.6\linewidth}
      \begin{list}{\arabic{enumi}.}{%
          \usecounter{enumi}
          \leftmargin2em
          \itemsep0.1em}
          {\small
          \item Always tell the truth (do not lie or deceive).
          \item Treat all topics with respect and humility.
          \item Do not make assumptions without evidence.
          \item Be responsible for your words and their consequences.
          \item Honor reason, clarity, and understanding.}
      \end{list}
    \end{minipage}
  }
  \caption{Normative principles oriented toward truthfulness, humility, evidential support, responsibility, and clarity.}
  \label{fig:norms}
\end{figure}

\subsection{Two-Stage Evaluation Protocol}
A central challenge is evaluating whether the induced abstention is \textit{targeted} rather than \textit{indiscriminate}. A lower false-answer rate among answered questions is not, by itself, sufficient evidence of better abstention: a model can reduce the risk of surfaced answers simply by refusing more often. 
Our two-stage evaluation protocol (Figure~\ref{fig:two_stage_protocol}) therefore lets the model first answer or abstain, then elicits a forced best guess after each abstention.
This allows us to estimate whether abstentions are concentrated on questions that would otherwise have been answered incorrectly, while also measuring how often correct answers are unnecessarily withheld.

\begin{figure}[ht]
\centering
\includegraphics[width=0.72\linewidth]{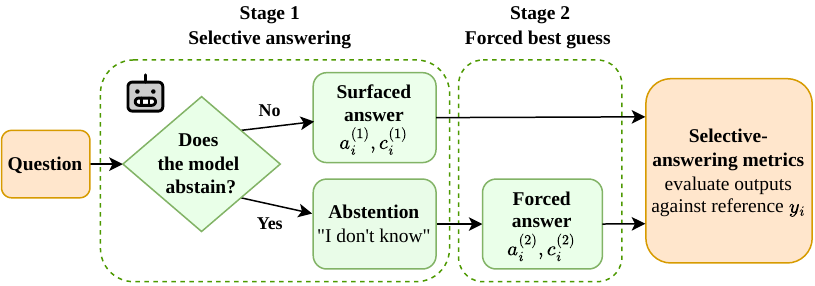}
\caption{Two-stage evaluation. Stage 1 allows answering/abstention; after abstention, Stage 2 elicits a forced best guess.}
\label{fig:two_stage_protocol}
\end{figure}

\paragraph{Notation.}
Consider a factual QA dataset of $n$ questions.
Let $y_i$ be the reference answer for $i$-th question, and $S_i=1$ if Stage~1 surfaces an answer ($S_i=0$ otherwise).
If $S_i=1$, the first-stage answer and confidence are $(a_i^{(1)},c_i^{(1)})$.  If $S_i=0$, the second-stage forced answer and confidence are $(a_i^{(2)},c_i^{(2)})$.  
Define the
eventual candidate answer and its associated confidence jointly as
$$
(\bar a_i,p_i)=
\begin{cases}
(a_i^{(1)},c_i^{(1)}), & S_i=1,\\
(a_i^{(2)},c_i^{(2)}), & S_i=0.
\end{cases}
$$
The eventual correctness indicator is
$C_i=\mathbf{1}\{\bar a_i\equiv y_i\}$, where
$\equiv$ denotes correctness under
the benchmark's evaluator. 
Write $\mathbb{P}_n$ for the empirical distribution of $(S_i,C_i)$, with conditional proportions defined on nonempty sets.
The metrics below collectively answer three questions: (i) how accurate the surfaced answers are, (ii) whether the abstentions are targeted, and (iii) whether the model's expressed confidence tracks correctness under the induced decision policy.

\paragraph{Selective-Answering Metrics.}
Coverage is the fraction of questions answered in the first stage:
$\mathbb{P}_n(S=1)$.
The false-answer rate among surfaced answers is
$\mathrm{FAR}_{\mathrm{answered}}=\mathbb{P}_n(C=0\mid S=1)$,
while the forced-answer false-answer rate over all questions is
$\mathrm{FAR}_{\mathrm{overall}}=\mathbb{P}_n(C=0)$.

\paragraph{Abstention-Quality Metrics.}
To determine whether abstentions are targeted toward unreliable cases, we evaluate the first-stage abstention decision as a classifier for eventual forced-answer error.  
The error capture rate $\mathrm{ECR}=\mathbb{P}_n(S=0\mid C=0)$ measures the fraction of likely errors that were intercepted by an initial abstention.  
The correct answer retention $\mathrm{CAR}=\mathbb{P}_n(S=1\mid C=1)$ measures the fraction of eventual correct answers surfaced in the first stage.  

Viewing the abstention decision as a prediction that the eventual candidate answer is incorrect, ECR is the sensitivity of the abstention decision and CAR is its specificity. 
We summarize the balance between these two metrics  using abstention informedness, corresponding to Youden's $J$ \cite{youden1950index}:
$J_{\mathrm{abs}}
=\mathrm{ECR}+\mathrm{CAR}-1
=\mathbb{P}_n(S=0\mid C=0)-\mathbb{P}_n(S=0\mid C=1)
\in[-1,1].$ 
A value of $1$ denotes perfect targeting of erroneous candidates, $0$ denotes no population-level association between abstention and eventual correctness, and positive values indicate that the policy abstains more frequently on erroneous than correct candidates.


\paragraph{Calibration Metrics.}

We report Brier score \citep{glenn1950verification} and empirical expected calibration error ($\widehat{\text{ECE}}$)  \citep{degroot1983comparison,guo2017calibration} to assess the forecasting accuracy and calibration quality. 
Appendix \ref{appendix:sec2} gives their definitions.

\section{I-CALM Improves Inference-Time Selective Answering}
\label{sec:baseline}

We next assess the empirical effectiveness of I-CALM. Our analysis examines whether the proposed framework yields a more reliable selective-answering policy than established baseline alternatives.

\subsection{Experiment Setup}

\paragraph{Evaluation Setting.}
We evaluate selective-answering performance on PopQA \citep{mallen2023not}, a factual question-answering benchmark that emphasizes long-tail knowledge. We report results for GPT-5 mini \citep{openai_gpt5mini}, Gemini-3.1-Flash-Lite \citep{googledeepmind2026gemini31flashlite}, and Meta-Llama-3-8B-Instruct \citep{grattafiori2024llama}. 
We additionally perform the comparisons on SimpleQA-Verified \cite{haas2025simpleqa}, reported in Appendix~\ref{appendix:sec3_results_simpleqa}. 
For the main comparison, we use a representative I-CALM configuration that combines verbal confidence elicitation, normative guidance, and the announced payoffs 
$\alpha=1$,
$\beta=1$, and
$\gamma=0.4$.

\paragraph{Baselines.}
We compare I-CALM against four baselines. All methods except Pure Eval use the two-stage evaluation protocol introduced in Section~\ref{sec:problem_protocol}.
Comparisons between I-CALM and verbal-confidence and token-probability thresholding baselines are coverage-matched.
\begin{itemize}
    \item \textbf{Pure Eval.}
    The model is asked to provide an answer without additional instructions.

    \item \textbf{Explicit IDK Prompting.}
    The model is explicitly asked to respond with ``I don't know''
    when it is uncertain about its answer \cite{qin2026large,cheng2024can}.

    \item \textbf{Verbal-Confidence Thresholding.}
    Each question contributes the candidate answer $\bar a_i$ and its associated verbal confidence $p_i$. 
    We rank and surface answers in descending order of $p_i$ \cite{madhusudhan2025llms, tian2023just}.

    \item \textbf{Token-Probability Thresholding.}
    We use the geometric mean probability of the generated answer tokens as the confidence signal and rank and surface answers in its descending order. \cite{lin2024contextualized} 
\end{itemize}

\subsection{Results}

Table~\ref{tab:method_comparison} shows that I-CALM provides the most consistently useful balance of selective-answering performance, abstention quality, calibration, and deployment requirements. 
Pure Eval retains nearly full coverage but surfaces errors frequently, with high $\mathrm{FAR}_{\mathrm{answered}}$ values and negligible $J_{\mathrm{abs}}$ across all models. 
Explicit IDK prompting lowers risk in some cases but is highly model-dependent, does not produce confidence scores for calibration, and often does so by excessive abstention: coverage falls to $0.290$ for GPT-5 mini and nearly zero ($0.002$) for Llama, making it impractical. 
At nearly identical coverage on Gemini, I-CALM improves over both IDK and verbal-confidence thresholding in $\mathrm{FAR}_{\mathrm{answered}}$, $J_{\mathrm{abs}}$, and calibration. 
On Llama, I-CALM likewise provides the strongest matched-coverage selective-answering performance as well as abstention quality, compared to verbal-confidence and token-probability thresholding. 
Token probabilities are better calibrated on Llama, but this baseline requires token-level probability or logit access and is unsupported for the two other API models evaluated here. 
GPT-5 mini is the only case in which verbal-confidence thresholding has modestly stronger selective-answering metrics at matched coverage; however, I-CALM substantially improves calibration. Averaged across the three models, I-CALM improves $J_{\mathrm{abs}}$ by $28.7\%$ relative to the strongest coverage-matched baseline.
Overall, I-CALM offers a strong black-box, prompt-only alternative that avoids the severe over-abstention of IDK prompting and the access requirements of token-probability thresholding while maintaining competitive or superior selective-answering performance and meaningful confidence calibration; similar patterns hold on SimpleQA-Verified (Appendix~\ref{appendix:sec3_results_simpleqa}).

\begin{table}[ht]
\centering
\scriptsize
\setlength{\tabcolsep}{2.0pt}
\renewcommand{\arraystretch}{1.05}

\begin{adjustbox}{max width=\linewidth}
\begin{tabular}{
@{}
llccccc cccc
@{}
}
\toprule
\multirow{2}{*}{\textbf{Model}} &
\multirow{2}{*}{\textbf{Method}} &
\multirow{2}{*}{\textbf{Coverage}} &
\multirow{2}{*}{$\mathbf{FAR}_{\mathrm{answered}} \downarrow$} &
\multirow{2}{*}{\textbf{ECR}} &
\multirow{2}{*}{\textbf{CAR}} &
\multirow{2}{*}{$\boldsymbol{J}_{\mathrm{abs}} \uparrow$} &
\multicolumn{2}{c}{$\widehat{\textbf{ECE}} \downarrow$} &
\multicolumn{2}{c}{\textbf{Brier Score} $\downarrow$} \\
\cmidrule(lr){8-9}
\cmidrule(lr){10-11}

&
&
&
&
&
&
&
\textbf{Answered} &
\textbf{Overall} &
\textbf{Answered} &
\textbf{Overall} \\
\midrule

\multirow{5}{*}{GPT-5 mini}
& Pure Eval
& 0.97
& \mbox{0.533 {\scriptsize (0.008)}}
& 0.050
& 0.994
& 0.044
& \multicolumn{4}{c}{\cellcolor{red!10}\textit{Not supported}} \\

& IDK
& \cellcolor{red!10}0.290
& \mbox{\best{0.220} {\scriptsize (0.013)}}
& 0.877
& 0.468
& 0.345
& \multicolumn{4}{c}{\cellcolor{red!10}\textit{Not supported}} \\

& Verbal Conf.
& 0.564
& \mbox{0.323 {\scriptsize (0.010)}}
& 0.660
& 0.822
& \best{0.482}
& 0.202
& 0.172
& \mbox{0.238 {\scriptsize (0.007)}}
& \mbox{0.218 {\scriptsize (0.005)}} \\

& Token Prob.
& \multicolumn{9}{c}{\cellcolor{red!10}\textit{Not supported}} \\

\rowcolor{icalmrow}
& \textbf{I-CALM}
& 0.564
& \mbox{0.355 {\scriptsize (0.011)}}
& 0.636
& 0.808
& 0.444
& \best{0.147}
& \best{0.104}
& \mbox{\best{0.206} {\scriptsize (0.009)}}
& \mbox{\best{0.186} {\scriptsize (0.007)}} \\
\midrule

\multirow{5}{*}{Gemini-3.1-Flash-Lite}
& Pure Eval
& 0.999
& \mbox{0.444 {\scriptsize (0.008)}}
& 0.001
& 0.999
& 0.000
& \multicolumn{4}{c}{\cellcolor{red!10}\textit{Not supported}} \\

& IDK
& 0.704
& \mbox{0.325 {\scriptsize (0.009)}}
& 0.470
& 0.836
& 0.306
& \multicolumn{4}{c}{\cellcolor{red!10}\textit{Not supported}} \\

& Verbal Conf.
& 0.706
& \mbox{0.330 {\scriptsize (0.009)}}
& 0.476
& 0.852
& 0.328
& 0.327
& 0.389
& \mbox{0.326 {\scriptsize (0.009)}}
& \mbox{0.381 {\scriptsize (0.007)}} \\

& Token Prob.
& \multicolumn{9}{c}{\cellcolor{red!10}\textit{Not supported}} \\

\rowcolor{icalmrow}
& \textbf{I-CALM}
& 0.706
& \mbox{\best{0.322} {\scriptsize (0.009)}}
& 0.501
& 0.879
& \best{0.380}
& \best{0.308}
& \best{0.332}
& \mbox{\best{0.308} {\scriptsize (0.009)}}
& \mbox{\best{0.326} {\scriptsize (0.007)}} \\
\midrule

\multirow{5}{*}{Meta-Llama-3-8B-Instruct}
& Pure Eval
& 0.963
& \mbox{0.683 {\scriptsize (0.008)}}
& 0.051
& 0.997
& 0.048
& \multicolumn{4}{c}{\cellcolor{red!10}\textit{Not supported}} \\

& IDK
& \cellcolor{red!10}0.002
& \mbox{\best{0.121} {\scriptsize (0.113)}}
& 1.000
& 0.006
& 0.006
& \multicolumn{4}{c}{\cellcolor{red!10}\textit{Not supported}} \\

& Verbal Conf.
& 0.447
& \mbox{0.571 {\scriptsize (0.012)}}
& 0.636
& 0.639
& 0.275
& 0.564
& 0.533
& \mbox{0.561 {\scriptsize (0.012)}}
& \mbox{0.501 {\scriptsize (0.007)}} \\

& Token Prob.
& 0.447
& \mbox{0.605 {\scriptsize (0.012)}}
& 0.601
& 0.547
& 0.148
& \best{0.454}
& \best{0.443}
& \mbox{\best{0.425} {\scriptsize (0.008)}}
& \mbox{\best{0.404} {\scriptsize (0.004)}} \\

\rowcolor{icalmrow}
& \textbf{I-CALM}
& 0.447
& \mbox{0.515 {\scriptsize (0.012)}}
& 0.685
& 0.805
& \best{0.490}
& 0.497
& \best{0.446}
& \mbox{0.496 {\scriptsize (0.012)}}
& \mbox{0.423 {\scriptsize (0.009)}} \\

\bottomrule
\end{tabular}
\end{adjustbox}

\caption{
Baseline comparisons on PopQA. Values in parentheses denote 95\% confidence-interval half-widths. Bold values indicate the best reported value within each model block. Blue rows highlight I-CALM; red-shaded cells indicate unsupported metrics or results with impractically low coverage.
}
\label{tab:method_comparison}
\end{table}

\section{How and When Does I-CALM Work?}
\label{sec:how_and_when}

Section~\ref{sec:baseline} showed that I-CALM yields stronger selective-answering performance than standard baselines.
We now examine how and when it works: component ablations identify the role of each component (Section~\ref{subsec:ablations}), while incremental experiments trace how payoff design and normative guidance improve abstention quality, transfer across models and datasets, reshape confidence routing, and shift the selective-answering operating point (Section~\ref{subsec:incremental}).
Additional analyses test sensitivity to payoff magnitude and fact popularity and examine post-hoc confidence thresholding for deployment (Section~\ref{sec:further_analyses}).
Unless stated otherwise, results use GPT-5 mini on PopQA.

\subsection{Component Ablations}
\label{subsec:ablations}

We perform component-wise ablation to assess the role and effect of each each I-CALM component.
Table~\ref{tab:ablation} reveals distinct failure modes: removing payoffs or normative guidance leads the model to surface less reliable answers, whereas removing confidence induces broad over-abstention.
Every ablation lowers $J_{\mathrm{abs}}$, showing that all three components contribute to I-CALM's error-targeting abstention behavior.

\paragraph{Announced Payoffs.} Removing the announced payoffs increases coverage from $0.56$ to $0.76$, but also raises $\mathrm{FAR}_{\mathrm{answered}}$ from $0.36$ to $0.43$, lowers $J_{\mathrm{abs}}$ from $0.44$ to $0.33$, and worsens every reported calibration metric.
Conversely, announced payoffs move the model to a more effective selective-answering operating point, with lower surfaced risk and better error targeting, improving how the model distinguishes candidates that should be surfaced from those that should be withheld.

\paragraph{Verbal Confidence.}
Removing verbal confidence produces the opposite failure mode.
Although $\mathrm{FAR}_{\mathrm{answered}}$ falls, coverage drops sharply from $0.56$ to $0.25$, and $J_{\mathrm{abs}}$ falls from $0.44$ to $0.30$.
Worsened abstention targeting suggests that the lower surfaced risk is achieved through broad withholding rather than improved separation of erroneous and correct candidates.
Restoring verbal confidence allows the model to answer many more questions while making abstention more selective.
Verbal confidence therefore provides an important self-assessment signal that prevents the model from excessive over-abstention.

\paragraph{Normative Guidance.}
Removing normative guidance raises coverage, but also increases $\mathrm{FAR}_{\mathrm{answered}}$ from $0.36$ to $0.42$ and lowers $J_{\mathrm{abs}}$ from $0.44$ to $0.38$.
These changes show that the additional abstention brought by normative guidance is concentrated on error-prone candidates rather than applied indiscriminately.
Calibration remains essentially stable, indicating that normative guidance contributes primarily by improving which candidate answers are surfaced, beyond what is achieved through payoffs and verbal confidence alone.
Finer-grained ablations on individual principles in Appendix~\ref{subsec:norm_ablation}  show that no single principle accounts for the full effect and the complete set performs best.

\begin{table}[ht]
\centering
{\fontsize{9pt}{8pt}\selectfont
\setlength{\tabcolsep}{3.5pt}
\renewcommand{\arraystretch}{1}
\begin{tabular}{@{}lccccccc@{}}
\toprule
& \multirow{2}{*}{\textbf{C}}
& \multirow{2}{*}{\textbf{FAR}}
& \multirow{2}{*}{$\boldsymbol{J}_{\mathrm{abs}}$}
& \multicolumn{2}{c}{$\widehat{\textbf{ECE}}$}
& \multicolumn{2}{c}{\textbf{Brier Score}} \\
\cmidrule(lr){5-6}
\cmidrule(lr){7-8}
&
&
&
&
\textbf{A}
& \textbf{O}
& \textbf{A}
& \textbf{O} \\
\midrule

\rowcolor{icalmrow}
\textbf{I-CALM}
& 0.56
& 0.36
& \best{0.44}
& \best{0.15}
& \best{0.10}
& \best{0.21}
& \best{0.19} \\

\midrule

w/o Payoffs
& 0.76
& 0.43
& 0.33
& 0.21
& 0.17
& 0.25
& 0.24 \\

w/o Confidence
& 0.25
& 0.26
& 0.30
& \multicolumn{4}{c}{\textit{Not applicable}} \\

w/o Norms
& 0.68
& 0.42
& 0.38
& 0.15
& 0.12
& 0.21
& 0.19 \\

\bottomrule
\end{tabular}
}
\caption{Component ablations relative to the full I-CALM configuration $(+1,-1,+0.4)$ for GPT-5 mini on PopQA. C: coverage; FAR: $\mathrm{FAR}_{\mathrm{answered}}$; A and O: answered and overall, respectively.}
\label{tab:ablation}
\end{table}

\subsection{Incremental Effects of Payoff Design and Normative Guidance}
\label{subsec:incremental}

To isolate the incremental effects of rewarding abstention and normative guidance, we compare the following three nested prompt configurations, each of which elicits verbal confidence.
For $R,\beta,\gamma\geq 0$, \emph{Scheme~A} announces a payoff of $+R$ for a correct answer and $-\beta$ for an incorrect answer, without rewarding abstention; \emph{Scheme~B} adds a payoff of $+\gamma$ for abstention in Stage 1; and \emph{I-CALM} retains Scheme~B's payoffs while adding normative guidance.
Thus, comparing Scheme~A with Scheme~B isolates the effect of adding a positive abstention payoff while holding the correct- and incorrect-answer payoffs fixed; prior payoff-framing studies instead fix the abstention payoff at zero and vary either only one or both of the correct- and incorrect-answer payoffs \citep{wang2026llm,wu2025answer}.
Comparing Scheme~B with I-CALM isolates the effect of normative guidance. 
We denote the payoff specifications by $(+R,-\beta)$ for Scheme~A and $(+R,-\beta,+\gamma)$ for Scheme~B and I-CALM, and unless stated otherwise use $R=1$, $\beta=1$, and $\gamma=0.4$.
Appendix~\ref{subsec:E2_full_result} reports the full sensitivity results for $\beta\in\{0,1\}$ and $\gamma\in\{0.2,0.4,0.6,0.8\}$.

\paragraph{Abstention Quality Improves Monotonically.}

Figure~\ref{fig:model_comparison} compares the progression from Pure Eval to Scheme~A, Scheme~B, and I-CALM for three models on PopQA.
The clearest cross-model pattern is the monotonic improvement in abstention quality: $J_{\mathrm{abs}}$ increases at every step, indicating that each successive configuration better concentrates abstention on erroneous candidates.
I-CALM attains the lowest $\mathrm{FAR}_{\mathrm{answered}}$ among the three schemes, whereas $\mathrm{FAR}_{\mathrm{overall}}$ changes only modestly.
This suggests that the improvement primarily comes from better selecting which candidate answers to surface rather than increasing candidate accuracy across all questions, which aligns with our hypothesis and intended goal.
Notably, for Gemini and Llama the largest single-step gains in $J_{\mathrm{abs}}$, $\mathrm{FAR}_{\mathrm{answered}}$, and $\widehat{\mathrm{ECE}}$ all come from adding norms to Scheme~B, indicating that normative guidance is the dominant lever for models that respond weakly to announced payoffs.

\begin{figure}[ht]
    \centering
    \includegraphics[width=0.55\linewidth]{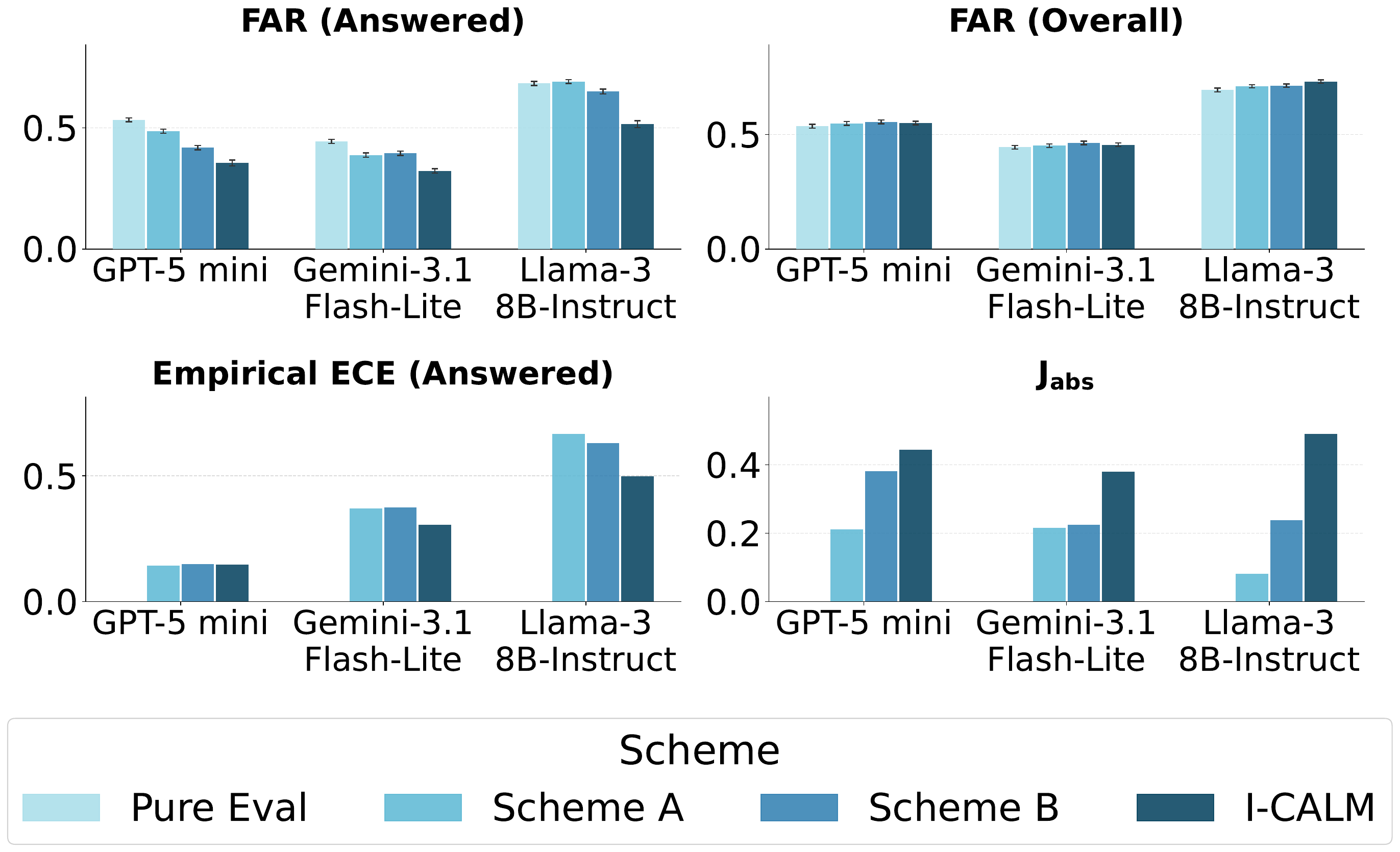}
    \caption{Evaluation on PopQA as I-CALM's components are added incrementally.}
    \label{fig:model_comparison}
\end{figure}

\paragraph{Transfer Experiments.}

We evaluate transfer across model families and dataset difficulty by adding GPT-4o mini \cite{openai_gpt4omini}, Meta-Llama-3-8B-Instruct, Gemini-3.1-Flash-Lite, and Qwen3-4B-Instruct-2507 \cite{yang2025qwen3} on PopQA, and by evaluating these models together with GPT-5 mini on TriviaQA and SimpleQA-Verified. Across the resulting 15 model--dataset settings, ECR increases at every step from A $\rightarrow$ B $\rightarrow$ I-CALM, while $J_{\mathrm{abs}}$ increases monotonically and is highest under I-CALM in 14 settings. I-CALM also lowers $\mathrm{FAR}_{\mathrm{answered}}$ relative to Scheme~A in 14 settings and ties in the remaining one. Overall, the effect is strongest and most consistent on PopQA, becomes more dependent on baseline headroom on TriviaQA, and still remains visible under the severe knowledge pressure of SimpleQA-Verified. Full results are in Appendix~\ref{subsec:additional_exp}.

\paragraph{Confidence Distributions Reveal Error-Targeted Routing.}

Aggregate metrics show that abstention becomes better targeted, but not how confidence and correctness differ between surfaced cases and cases initially withheld.
Figure~\ref{fig:confidence_distribution} decomposes GPT-5 mini's Stage~1 surfaced answers (left) and Stage~2 forced best guesses after abstention (right) by confidence and correctness across Scheme~A, Scheme~B, and I-CALM.
As the components are added, low- and medium-confidence incorrect mass diminishes in the surfaced distribution, while the best-guess distribution gains predominantly incorrect mass at confidence levels of $0.1$--$0.4$, with a pronounced peak near $0.3$ under I-CALM.
Meanwhile, surfaced correct answers remain concentrated at high confidence ($0.8$--$1.0$).
Thus, I-CALM does not merely increase abstention; they increasingly route low-confidence, error-prone cases away from the surfaced channel.
This distributional pattern is consistent with, and helps explain, the monotonic increase in abstention quality $J_{\mathrm{abs}}$.

\begin{figure}[ht]
    \centering

    \begin{subfigure}{0.7\linewidth}
        \centering
        \includegraphics[width=\linewidth]{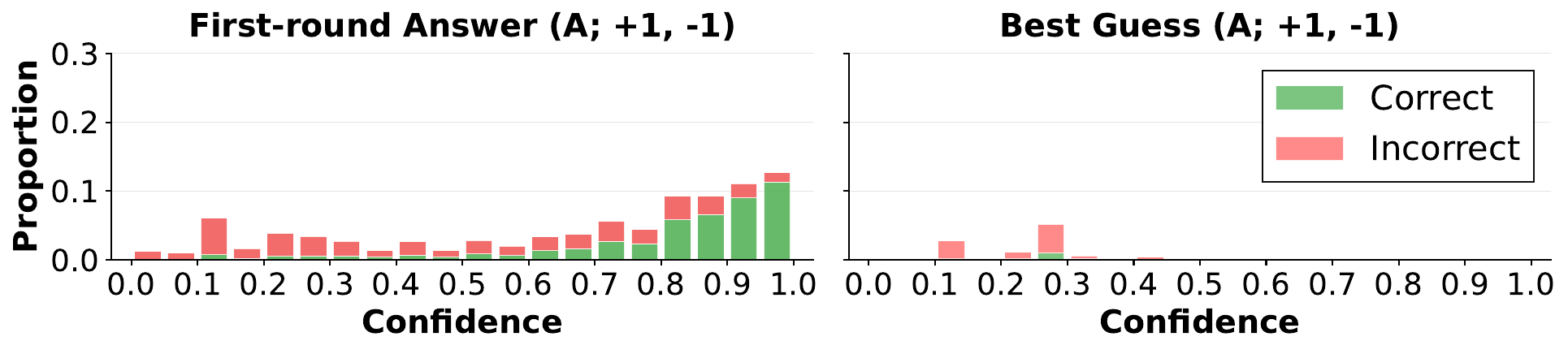}
        \label{fig:A_conf}
    \end{subfigure}

    \vspace{-1em}

    \begin{subfigure}{0.7\linewidth}
        \centering
        \includegraphics[width=\linewidth]{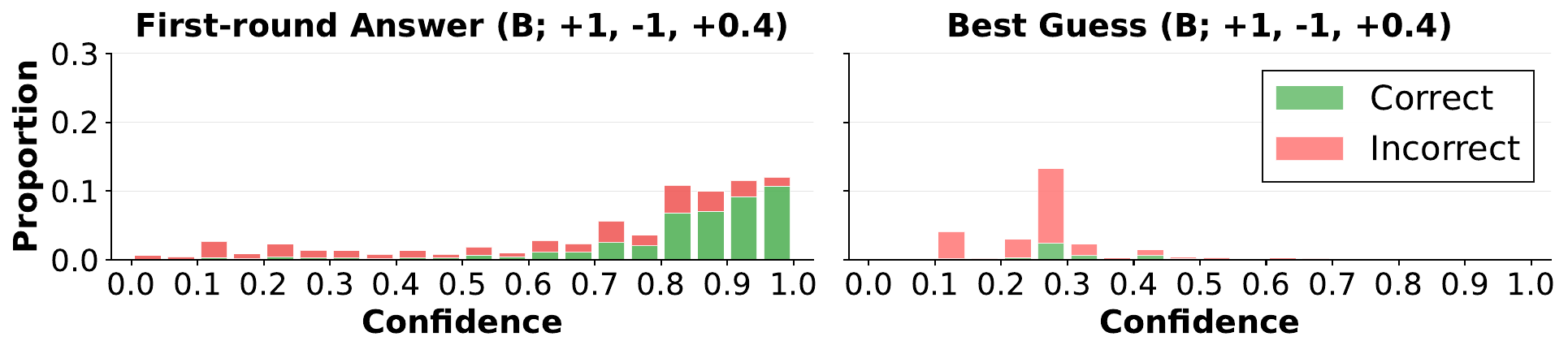}
        \label{fig:B_conf}
    \end{subfigure}

    \vspace{-1em}

    \begin{subfigure}{0.7\linewidth}
        \centering
        \includegraphics[width=\linewidth]{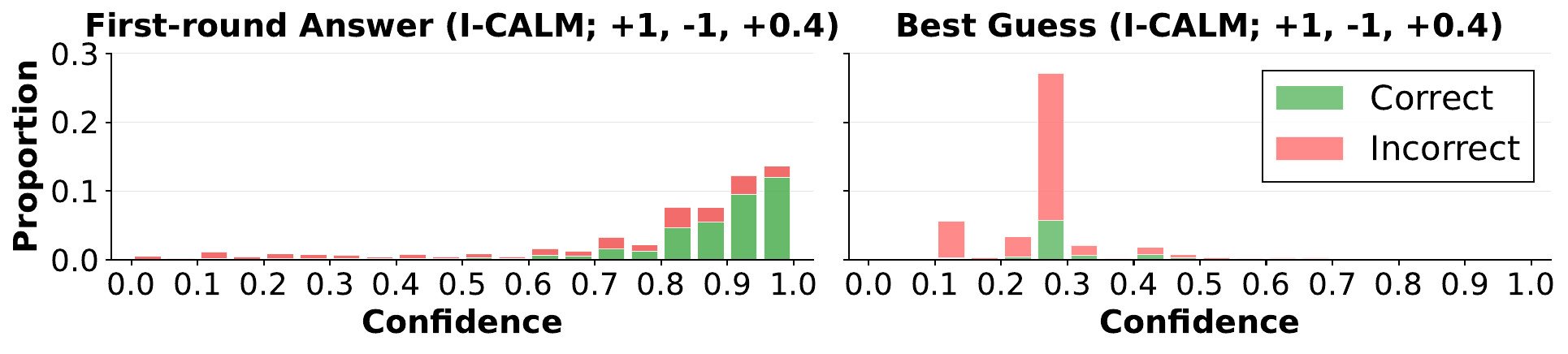}
        \label{fig:B_norm_conf}
    \end{subfigure}
    \vspace{-1em}
    \caption{Confidence distributions for GPT-5 mini under representative schemes
    on PopQA. Left: Stage 1 surfaced answers. Right: Stage 2 best guesses.}
    \label{fig:confidence_distribution}
\end{figure}

\paragraph{I-CALM Components Shift the Selective-Answering Operating Point.}

Figure~\ref{fig:pareto_frontier} plots $\mathrm{FAR}_{\mathrm{answered}}$ against abstention rate across all tested configurations for GPT-5 mini on PopQA.
The configurations lie near a common monotone curve: higher abstention rates correspond to lower $\mathrm{FAR}_{\mathrm{answered}}$.
In selective answering, the accompanying reduction in coverage is \textit{not} inherently undesirable; because abstention is targeted---as evidenced by the increasing $J_{\mathrm{abs}}$ and the confidence distributions in Figure~\ref{fig:confidence_distribution}---it is the \textit{intended} mechanism for preventing likely errors from being surfaced.
For each fixed abstention reward $\gamma$, penalizing incorrect answers ($\beta=1$) moves the model toward both greater abstention and lower surfaced risk than $\beta=0$, while at matched payoffs I-CALM shifts the metrics further in the same direction relative to Scheme~B.
Thus, I-CALM enables deployment-time selection of selective-answering operating point without retraining or access to model internals.

\begin{figure}[ht]
    \centering
    \includegraphics[width=0.5\linewidth]{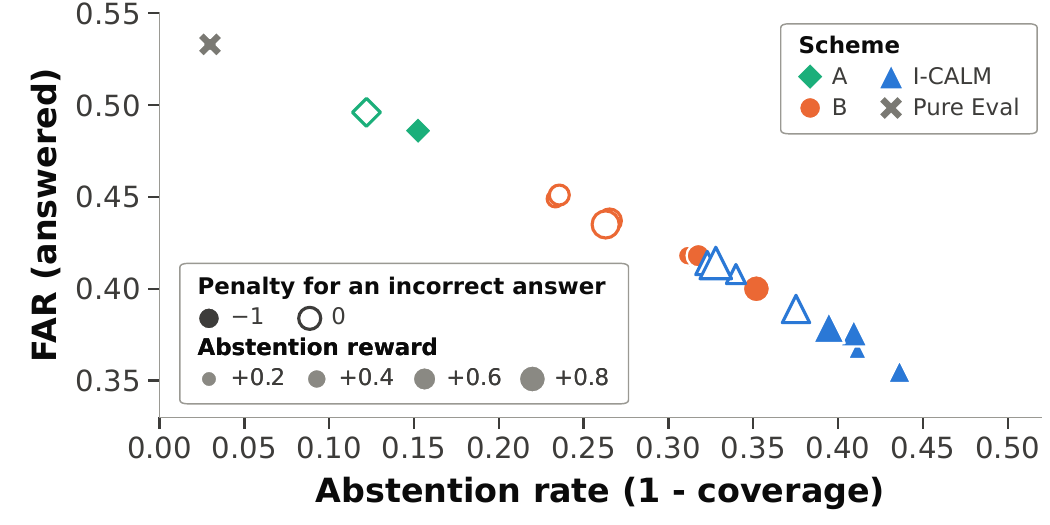}
    \caption{The abstention--FAR curve for GPT-5 mini on PopQA across all experimented configurations. }
    \label{fig:pareto_frontier}
\end{figure}

\subsection{Additional Analyses}
\label{sec:further_analyses}


Table~\ref{tab:extreme_case_results} shows that selective answering is far more sensitive to the abstention reward than to the correct-answer reward or the false-answer penalty. This complements prior evidence that LLMs often fail to adjust abstention behavior even when false-answer penalties vary substantially \citep{wang2026llm}.
Moreover, scaling the entire payoff vector leaves performance nearly unchanged (Appendix~\ref{subsec:_robust_reward_scaling}), indicating that behavior depends on relative rather than absolute payoff magnitudes.

\begin{table}[ht]
\centering
{\fontsize{9pt}{6pt}\selectfont
\setlength{\tabcolsep}{4pt}
\begin{tabular}{@{}ccc cccc@{}}
\toprule
\multicolumn{3}{c}{\textbf{Announced Payoff}} &
\multirow{2}{*}{\textbf{C}} &
\multirow{2}{*}{\textbf{FAR}} &
\multirow{2}{*}{\textbf{ECR}} &
\multirow{2}{*}{\textbf{CAR}} \\
\cmidrule(lr){1-3}
\textbf{Correct} & \textbf{Wrong} & \textbf{Abstain} & & & & \\
\midrule

\rowcolor{icalmrow}
+1   & -1   & +0.4 & 0.564 & 0.418 {\scriptsize (0.01)} & 0.481 & 0.900 \\

\midrule

+100 & -1   & +0.4 & 0.752 & 0.445 {\scriptsize (0.01)} & 0.400 & 0.942 \\
+1   & -100 & +0.4 & 0.694 & 0.425 {\scriptsize (0.01)} & 0.474 & 0.906 \\
+1   & -1   & +40  & 0.429 & 0.285 {\scriptsize (0.11)} & 0.776 & 0.677 \\
\bottomrule
\end{tabular}
}
\caption{Sensitivity to extreme payoffs;
GPT-5 mini on PopQA. C: coverage; FAR: $\mathrm{FAR}_{\mathrm{answered}}$.}
\label{tab:extreme_case_results}
\end{table}


Splitting PopQA by entity popularity, we compare performance on the top and bottom popularity terciles. 
The progression A $\rightarrow$ B $\rightarrow$ I-CALM consistently lowers $\mathrm{FAR}_{\mathrm{answered}}$ and improves abstention quality in both strata, with larger absolute risk reductions on rare facts. This suggests that I-CALM may be particularly useful for long-tail factual questions. Full results are in Appendix~\ref{subsec:rare_vs_common_facts}.


The elicited confidence signal can also be converted into a post-hoc acceptance rule with a finite-sample guarantee on $\mathrm{FAR}_{\mathrm{answered}}$. Appendix~\ref{subsec:new_false_answer_rate_control} develops a Bonferroni CP-UCB rule and a multistart fixed-sequence alternative.

\section{Related Work}
\label{sec:related_works}

\paragraph{LLM Hallucination and Abstention Mechanisms.}

Hallucinations remain a central reliability problem for LLMs \citep{huang2025survey}.
Recent work emphasizes that evaluation incentives matter: 
binary evaluations and leaderboards reward guessing and can discourage abstention \citep{kalai2025language}.
Abstention is now treated as a first-class capability, with its own taxonomies, metrics, and benchmarks \citep{wen2025know}. 
A substantial part of this literature focuses on settings where abstention is itself the correct behavior because the query is unanswerable, unsupported, or underspecified \citep{kirichenko2025abstentionbench,lavi2025detecting,madhusudhan2025llms,wu2025robots}. 
Closer to our setting, knowledge-gap recognition work asks whether models can tell when they do not know the answer to a question that has a verifiable answer \citep{cheng2024can,qin2026large}.
Existing solutions include training-time abstention via RL, explicit abstention tokens, or semantic uncertainty \citep{an2025teaching, cohen2024don, jain2024refusal, jha2026rewarding, mohamadi2025honesty,tjandra2024fine, wei2025truthrl,wu2025mitigating} and inference-time risk-control wrappers \citep{oehri2025trusted, wang2025safer}.

\paragraph{Confidence Estimation and Calibration in LLMs.}

Our method requires an uncertainty signal that is available at inference time. 
Surveys and systematic evaluations cover logit-, representation-, semantic-, and consistency-based measures and highlight calibration challenges in generative settings \citep{geng2024survey,hobelsberger2025systematic}.
Prior work comparing token-based and verbalized confidence finds both alignment and gaps
\citep{ji2025calibrating, kadavath2022language, kumar2024confidence}, and RLHF can amplify overconfident language \citep{chhikara2025mind, leng2024taming}.
A growing line studies how to elicit and calibrate expressed confidence in LLMs. 
On the prompting side, verbal confidence can be elicited with usable calibration, though prompt design matters \citep{tian2023just,xiong2023can,zhao2024fact}. 
On the training side, methods such as ConfTuner, ADVICE, and LACIE directly improve verbal confidence calibration \citep{li2025conftuner,seo2025advice,stengel2024lacie}.
See extended discussion in Appendix~\ref{sec:related_works_extended}.

\section{Conclusion}

We formulate epistemic abstention on factual questions with verifiable answers as selective answering and introduce I-CALM, a prompt-only framework that combines verbal confidence, announced answer/abstain payoffs, and normative guidance for fixed, black-box LLMs.
Across models and datasets, I-CALM generally reduces surfaced-answer risk and raises $J_{\mathrm{abs}}$, indicating that abstention becomes better targeted at erroneous candidates rather than merely more frequent. Across the 15 transfer settings, I-CALM attains the highest $J_{\mathrm{abs}}$ in 14 of them. 
Ablations reveal complementary roles: announced payoffs set the operating point, confidence elicitation limits indiscriminate over-abstention, and normative guidance further refines error targeting.

These findings should be interpreted within I-CALM's intended scope.
As a prompt-only intervention for fixed models, it is designed to improve answer selection rather than underlying factual competence; the modest changes in $\mathrm{FAR}_{\mathrm{overall}}$ are therefore consistent with its intended role.
It thus does not fully resolve severe underlying knowledge deficits, as reflected by the case for SimpleQA-Verified.
Also, our experiments focus on short-answer factual QA with one verifiable answer, leaving multi-hop, long-form, and multi-claim responses for future work.
We therefore view I-CALM as a \textit{lightweight complement} to training-based competence improvements and deployment-time risk-control methods.

\section*{Acknowledgments}
This work is supported by the AI2050 program at Schmidt Sciences (Grant G-25-67962). We are grateful to the members of the Normativity Lab, Matthew Renze, and Adam Tauman Kalai for their valuable feedback. We also thank the participants in EN.601.669 AI Safety, Alignment, \& Governance (Fall 2025) at Johns Hopkins University for their valuable feedback.

\bibliography{ref-20260805}
\bibliographystyle{plainnat}

\clearpage
\appendix
\startcontents[appendix]

\pdfbookmark[1]{Appendix Contents}{appendix-contents}

\section*{Appendix Contents}
\printcontents[appendix]{}{1}{\setcounter{tocdepth}{3}}
\bigskip

\newpage

\section{Extended Related Work}
\label{sec:related_works_extended}

This appendix expands the related-work discussion. 
We organize the literature along two axes: 
(i) hallucination mitigation, abstention, and selective answering; and 
(ii) confidence estimation and calibration.
Our setup sits at their intersection: we keep the base model fixed, elicit self-reported confidence, and use prompt-level payoff framing plus humility- and truthfulness-oriented norms to shape answer/abstain behavior.

\subsection{Hallucination Mitigation, Abstention, and Selective Answering}

Hallucination and abstention are now well-developed subliteratures. 
Broad surveys frame hallucination as a central reliability problem and abstention as a capability that must be evaluated in its own right \citep{huang2025survey,wen2025know}. 
\citet{kalai2025language} provide a key motivation for our paper by arguing that hallucinations reflect both pretraining statistical pressures and post-training evaluation incentives, with mainstream binary grading discouraging abstention.
Our setting focuses on factual questions with a verifiable answer, but the model may not know it.

Much existing abstention work focuses instead on cases where withholding an answer is itself the gold behavior because the query is unanswerable, unsupported, or underspecified. Representative benchmarks and evaluations study precisely this regime \citep{kirichenko2025abstentionbench,madhusudhan2025llms,wu2025robots}. 
Closer to our setup, knowledge-gap recognition work asks whether models can detect when they lack the knowledge needed to answer questions with a verifiable answer \citep{cheng2024can,qin2026large}, and uncertainty-based abstention work shows that abstaining on the right cases can reduce hallucinations and improve reliability \citep{tomani2024uncertainty}.

Methodologically, prior work addresses abstention in two main ways. Training-time approaches teach it directly, for example through explicit abstention outputs or abstention-aware reward shaping \citep{cohen2024don,mohamadi2025honesty,wei2025truthrl,wu2025mitigating, zhang2024r}. Inference-time approaches instead wrap a fixed model with calibrated refusal or risk-control procedures \citep{lee2024selective, oehri2025trusted,wang2025safer}. 
Our method sits between these lines: like inference-time wrappers, it keeps the base model fixed; like \citet{kalai2025language}'s explicit-confidence-target proposal, it uses prompt-level payoff framing; unlike their benchmark-redesign agenda, it studies a black-box deployment intervention.

Finally, we treat short truthfulness- and humility-oriented norms as a lightweight additive intervention rather than a claim about a unique constitution. 
This design is closest in spirit to principle-based alignment work: Constitutional AI provides the basic constitutional framing but is training-based \citep{bai2022constitutional}, while Reflect studies an inference-time constitutional alternative \citep{bell2026reflect}; it is also consistent with evidence that system-level prompting can materially change abstention behavior without updating model weights \citep{kirichenko2025abstentionbench}.

\subsection{Confidence Estimation, Verbal Confidence, and Calibration in LLMs}

Confidence estimation in LLMs spans logit-based, representation-based, semantic, and consistency-based methods, and recent surveys show that calibration remains difficult in generative settings \citep{geng2024survey,hobelsberger2025systematic}. Classic QA calibration work already showed that model probabilities can be miscalibrated and brittle under domain shift \citep{jiang2021can}, while black-box settings motivate generation-only calibration methods that do not require logits \citep{ulmer2024calibrating}. These constraints matter for our setting, where we want a signal that is both observable and actionable at inference time.

A growing line of work studies verbal confidence directly. Prompt-based work shows that useful confidence estimates can often be elicited in natural language, although performance depends on prompt design \citep{tian2023just,xiong2023can}. Related prompting work such as Fact-and-Reflection improves calibration by structuring reflection before the final answer \citep{zhao2024fact}, and training-based approaches such as Teaching Models to Express Their Uncertainty in Words, ConfTuner, ADVICE, and LACIE aim to make verbal confidence itself better calibrated \citep{li2025conftuner,lin2022teaching,seo2025advice,stengel2024lacie}.

At the same time, internal probabilities and verbalized confidence are not identical. Prior work finds partial alignment between the two signals, but also systematic gaps \citep{kadavath2022language,kumar2024confidence,ji2025calibrating}. Other work shows that RLHF can amplify verbal overconfidence \citep{chhikara2025mind, leng2024taming}, while recent mechanistic evidence suggests that verbal confidence reflects richer answer-quality evaluations than token probability alone \citep{kumaran2026llms}. 
This combination of usefulness and imperfection is exactly why we treat self-reported confidence as a practical signal for inference-time decisions rather than as a fully faithful readout of model belief.

That distinction also matters for decision-making. Risk-controlled abstention methods formalize how confidence should map to coverage and error \citep{oehri2025trusted,wang2025safer}, but recent evidence suggests that current models may still fail to convert verbal confidence into risk-sensitive answer/abstain behavior \citep{wang2026llm}. Our work therefore studies both pieces together: whether verbal confidence is informative enough to use, and whether prompt-level payoff framing can better align answer/abstain decisions with that expressed uncertainty.

\section{Preliminary: Assessing Self-Reported Verbal Confidence}
\label{sec:confidence_metric}

A common way to estimate LLM uncertainty is to aggregate token-level log probabilities after generation \citep{jiang2021can,geng2024survey}. 
However, our setting requires an uncertainty signal the model can express in its response to support abstention. 
Post-hoc scores also depend on the aggregation rule and can be distorted by answer realization, relevance, and length \citep{duan2024shifting,  gupta2024language,wang2024my}. 
We therefore use self-reported verbal confidence $\tau^{self}$ as an inference-time uncertainty proxy---it reflects automatic, sophisticated self-evaluation during generation \citep{kumaran2026llms}, is easy to elicit in black-box settings \citep{chhikara2025mind}, and can be directly acted upon.

We evaluate $\tau^{self}$ against post-hoc geometric mean token probability $\tau_{avg}^{token}$ along two dimensions: 
(i) \textbf{reliability}, whether $\tau^{self}$ functions as a confidence signal comparable to $\tau_{avg}^{token}$; 
and (ii) \textbf{robustness}, whether it remains stable under prompt paraphrasing and confidence scales, so abstention decisions do not depend on particular wording.
Because token-probability-based estimates can be task-sensitive \citep{xia2025survey}, we focus on free-response QA to match the main intervention experiments in Section~\ref{sec:baseline} and \ref{sec:how_and_when}. Section~\ref{subsec:verbal_confidence_mcqa} reports results for the multiple-choice setting.

\subsection{Validity of Self-Reported Verbal Confidence in Free-Response QA}
\label{subsec:experiment_1_setup}

\subsubsection{Experiment Setup}
\label{subsubsec:experiment_setup_freeresponse}

\paragraph{Dataset}
We evaluate GPT-4o mini\footnote{GPT-4o mini is used because GPT-5 mini does not expose token-level log probabilities. We also evaluate GPT-4o mini in the main experiments and observe consistent trends across models.} \citep{openai_gpt4omini} on PopQA \citep{mallen2023not}, a set of 14,267 factual questions constructed from Wikidata knowledge triples. 
Because all questions are factual queries with a verifiable answer, we operationalize hallucinations as false factual assertions, and abstention reflects the model's lack of knowledge rather than the absence of a correct answer.


\paragraph{Confidence Elicitation}
We elicit self-reported verbal confidence $\tau^{self}$ in $[0,1]$ using the two-part abstention-conditioned prompt framework (Figure~\ref{fig:two_stage_protocol}); for prompt see the template \ref{prompt:confidence} below. 
For each question, the LLM either (i) provides a direct answer with a confidence score, or (ii) abstains by first responding ``I don't know'' and then providing a best guess with its confidence. 
We define $\tau^{self}$ as the confidence attached to the evaluated answer: the first-round confidence if the model answers immediately, or the best-guess confidence if it initially abstains.

\paragraph{Reproducibility Notes}
We call the OpenAI Chat Completions API with model \texttt{gpt-4o-mini-2024-07-18}, temperature $0$ and token log-probabilities enabled (\texttt{logprobs:true}). 
The model's final answer is normalized by lowercasing, stripping punctuation, and trimming whitespace, and compared to the PopQA reference set \texttt{possible\_answers}. 
An output is marked correct if the normalized prediction is a substring of any reference string or vice versa.

\paragraph{Performance Metrics.} We report four metrics: 

(1) \textbf{False-answer rate (FAR)} evaluates task performance and operationalizes hallucination rate. For each prompt template, we compute FAR as the proportion of questions answered incorrectly. We report prompt-level FAR with its 95\% confidence interval.

    $$
    \text{FAR}=\frac{\text{N}_{\text{incorrect}}}{\text{N}_{\text{total}}}.
    $$

(2) \textbf{Pearson's $r$} measures the correlation between $\tau^{self}$ and $\tau_{avg}^{token}$. For each prompt template, we compute the correlation across all questions and report the corresponding 95\% confidence interval.
 
(3) \textbf{Brier score} quantifies forecasting loss. Probabilistic forecast accuracy is commonly evaluated using proper scoring rules \citep{gneiting2007strictly}, which measure the discrepancy between predicted probabilities and observed outcomes. 
We employ the Brier score \citep{glenn1950verification}, one of the most widely used proper scoring rules. The Brier score computes the mean squared error between predicted probabilities and observed binary outcomes, with lower values indicating more accurate probabilistic forecasts. For a single prediction with confidence $p$ and correctness label $y \in \{0,1\}$:
    $$
    \text{Brier}(p,y)=(p-y)^2
    $$

For each prompt template, we compute the Brier score for every question and report the average Brier score across all questions, together with its 95\% confidence interval.

(4) \textbf{Expected Calibration Error ($\widehat{\text{ECE}}$)} assesses calibration error. Beyond the overall probabilistic forecasting accuracy, we also care about whether these forecasts are trustworthy, i.e., whether stated probabilities reliably correspond to empirical frequencies of correctness. 
This property is formalized as \textit{calibration} \citep{degroot1983comparison,guo2017calibration}, which assesses how well confidence estimates track observed correctness across the probability spectrum. 
We adopt the Expected Calibration Error ($\widehat{\text{ECE}}$) to measure such calibration quality. $\widehat{\text{ECE}}$ is defined as the expected absolute gap between a model's probabilistic confidence and its empirical accuracy at that confidence level:
    $$       
    \widehat{\text{ECE}}=\mathbb{E}_{\hat{p}}[|\mathrm{Pr}(\hat{y}=y|\hat{p})-\hat{p}|].  
    $$

Here, $\hat{p}\in[0,1]$ denotes the model's confidence estimates and $\mathrm{Pr}(\hat{y}=y|\hat{p})$ denotes the true correctness probability among predictions made with confidence $\hat{p}$. 
A lower $\widehat{\text{ECE}}$ indicates closer alignment between predicted probabilities and empirical accuracy, although such alignment may still coexist with substantial forecasting loss. 
    
Since this expectation cannot be computed exactly from finite samples, we estimate $\widehat{\text{ECE}}$ using a binned approximation. The unit interval $[0,1]$ is partitioned into $K$ equal-width bins:
    $$
    B_k =
        \begin{cases}
            \left[\dfrac{k}{K}, \dfrac{k+1}{K}\right), & k = 0, 1, \ldots, K-2, \\
            \left[\dfrac{K-1}{K}, 1\right], & k = K-1.
        \end{cases}
    $$
    
The empirical ECE ($\widehat{\text{ECE}}$), which evaluates the consistency between reported probabilities and empirical correctness, is then computed as:
$$
\widehat{\text{ECE}}=\sum_{k=0}^{K-1}\frac{|B_k|}{N}|\mathrm{acc}(B_{k})-\mathrm{conf}(B_k)|,
$$
where $N$ denotes the total number of predictions. $|B_k|$ is the number of predictions whose confidence falls into bin $B_k$. $\mathrm{acc}(B_{k})=\frac{1}{|B_k|}\sum_{i\in B_k}\mathbf{1}\{\hat{y}_i=y_i\}$ is the empirical accuracy within bin $B_k$, and $\mathrm{conf}(B_k)=\frac{1}{|B_k|}\sum_{i\in B_k}\hat{p}_i$ is the average confidence in that bin. In our experiments, we set $K=10$.

\subsubsection{Prompt Template \ref{prompt:confidence}}
\label{subsec:prompt_tem_conf}

\refstepcounter{prompt}
\label{prompt:confidence}

\begin{tcolorbox}[
  enhanced jigsaw,
  colback=gray!5,
  colframe=gray!40,
  title=Prompt \ding{172},
  boxsep=3pt,
  left=3pt,
  right=3pt,
  top=3pt,
  bottom=3pt,
  parskip=0pt,
  breakable,
  break at=-\baselineskip/0pt,
  height fixed for=middle,
  pad at break=1mm
]

Give your confidence about your answer as a decimal between 0 and 1, where 0 means completely unsure and 1 means completely sure. Use four decimal places (e.g., 0.8765).

If you truly don't know, first say ``I don't know'' and do not give a confidence (leave it blank). If you choose to answer, provide both your answer and confidence.

Next, if you said ``I don't know'', also give your best possible guess and its confidence.

Format your response as follows:

\medskip
{\ttfamily
Answer: <your first answer>\\
Confidence: <0--1>\\
If you answered ``I don't know'', also include:\\
Best Guess: <your best possible answer>\\
Best Guess Confidence: <0--1>
}
\medskip

Question: \texttt{\{q\}}

\end{tcolorbox}

\subsubsection{Definition of Geometric Mean Token Probability}
\label{subsec:mean_token_logprob}
We calculate $\tau_{avg}^{token}$ from the token-level log probabilities $\log p(t)$ returned by the API for the final answer. Specifically, we compute the geometric mean probability over answer tokens:
$$
\tau_{avg}^{token} = \exp\left(\frac{1}{|\text{tokens}|} \sum_{t \in \text{tokens}} \log p(t)\right).
$$

\subsubsection{Results}
\label{subsec:verb_conf_results_free_response}

\paragraph{Calibration and Reliability.}
We construct four semantics-preserving prompt variants (Table~\ref{tab:popQA_paraphrase}) that keep the abstention-conditioned response structure fixed while varying surface wording. Across templates, $\tau^{self}$ and $\tau_{avg}^{token}$ exhibit a moderate, consistent positive correlation around 0.54. This suggests $\tau^{self}$ tracks the token-probability-based confidence reasonably well, even though the two are not identical.
The two signals also show comparable probabilistic forecasting and calibration quality. Their Brier scores fall in a similar range (approximately $0.33$ -- $0.35$), and $\widehat{\text{ECE}}$ values are also close.
Although $\tau^{self}$ attains marginally lower Brier scores and $\widehat{\text{ECE}}$ than $\tau_{avg}^{token}$, the differences are small and the 95\% CIs overlap substantially. 
We also note that the Brier score and $\widehat{\text{ECE}}$ values, which deviate from perfect calibration, are consistent with known patterns of imperfect calibration in verbal confidence \citep{mielke2022reducing,xiong2023can,chhikara2025mind,seo2025advice}. 

\begin{table}[h]
\centering
{\fontsize{9pt}{9pt}\selectfont
\setlength{\tabcolsep}{0.7pt}
\setlength{\heavyrulewidth}{1.2pt}

\begin{tabular}{ccccccc}
\toprule

\multirow{2}{*}{\textbf{Prompt}} &
\multirow{2}{*}{\textbf{FAR}} &
\multirow{2}{*}{\textbf{Pearson's $r$}} &
\multicolumn{2}{c}{\textbf{Brier Score}} &
\multicolumn{2}{c}{$\widehat{\textbf{ECE}}$} \\

\cmidrule(lr){4-5}
\cmidrule(lr){6-7}

& & &
\textbf{$\tau^{self}$} &
\textbf{$\tau_{avg}^{token}$} &
\textbf{$\tau^{self}$} &
\textbf{$\tau_{avg}^{token}$} \\

\midrule

\ding{172}
& 0.63 {\scriptsize (0.008)}
& 0.54 {\scriptsize (0.012)}
& 0.34 {\scriptsize (0.006)}
& 0.36 {\scriptsize (0.006)}
& 0.38
& 0.42 \\

\ding{173}
& 0.62 {\scriptsize (0.008)}
& 0.54 {\scriptsize (0.012)}
& 0.34 {\scriptsize (0.006)}
& 0.35 {\scriptsize (0.005)}
& 0.39
& 0.41 \\

\ding{174}
& 0.62 {\scriptsize (0.008)}
& 0.52 {\scriptsize (0.012)}
& 0.35 {\scriptsize (0.006)}
& 0.34 {\scriptsize (0.005)}
& 0.38
& 0.41 \\

\ding{175}
& 0.62 {\scriptsize (0.008)}
& 0.55 {\scriptsize (0.011)}
& 0.34 {\scriptsize (0.006)}
& 0.35 {\scriptsize (0.006)}
& 0.37
& 0.41 \\

\bottomrule
\end{tabular}
}

\caption{Comparison of self-reported and token-probability-based confidence metrics across paraphrased abstention-conditioned prompt templates for GPT-4o mini on PopQA. Values in parentheses denote 95\% confidence-interval half-widths.}
\label{tab:popQA_paraphrase}
\end{table}

\paragraph{Robustness to Prompt Paraphrasing}
Across templates, FAR, Pearson's $r$, and Brier scores for $\tau^{self}$ have overlapping 95\% confidence intervals, and its $\widehat{\text{ECE}}$ varies only slightly. 
This suggests that $\tau^{self}$ is not strongly driven by superficial wording differences.

\paragraph{Robustness to Verbal Confidence Scale Variants}
\label{subsec:experiment_1_scale}
To evaluate the robustness of $\tau^{self}$ to confidence-reporting conventions, 
we keep the setup fixed and vary only the confidence format (Table~\ref{tab:popQAconfidence_scaling}).
In particular, we include the $[0,20]$ scale, which has recently been reported to improve metacognitive efficiency over the standard $[0,100]$ format \citep{dai2026rescaling}.
Across all confidence formats, all metric values remain consistent, indicating that $\tau^{self}$ is robust to confidence reporting conventions. 
We adopt the $[0,1]$ format in all subsequent experiments.

\begin{table}[h]
\centering
{\fontsize{9pt}{9pt}\selectfont
\setlength{\tabcolsep}{3pt}
\setlength{\heavyrulewidth}{1.2pt}

\begin{tabular}{ccccccc}
\toprule

\multirow{2}{*}{\shortstack{\textbf{Confidence}\\\textbf{Format}}} &
\multirow{2}{*}{\textbf{FAR}} &
\multirow{2}{*}{\textbf{Pearson's $r$}} &
\multicolumn{2}{c}{\textbf{Brier Score}} &
\multicolumn{2}{c}{$\widehat{\textbf{ECE}}$} \\

\cmidrule(lr){4-5}
\cmidrule(lr){6-7}

& & &
\textbf{$\tau^{self}$} &
\textbf{$\tau_{avg}^{token}$} &
\textbf{$\tau^{self}$} &
\textbf{$\tau_{avg}^{token}$} \\

\midrule

$[0,1]$
& 0.63
& 0.54
& 0.34
& 0.36
& 0.38
& 0.42 \\

$[0,20]$
& 0.62
& 0.59
& 0.35
& 0.35
& 0.41
& 0.42 \\

$[0,100]$
& 0.62
& 0.55
& 0.36
& 0.35
& 0.41
& 0.42 \\

$[0\%,100\%]$
& 0.62
& 0.54
& 0.36
& 0.35
& 0.41
& 0.42 \\

\bottomrule
\end{tabular}
}

\caption{
Comparison of self-reported confidence formats under prompt \ding{172}
for GPT-4o mini on PopQA.
}
\label{tab:popQAconfidence_scaling}
\end{table}

\subsection{Validity of Self-Reported Verbal Confidence in Multiple Choice Question Answering}
\label{subsec:verbal_confidence_mcqa}
We next examine the reliability and robustness of self-reported verbal confidence $\tau^{self}$ in multiple-choice question answering (MCQA). Unlike free-response generation---where uncertainty is distributed over an output sequence---uncertainty in MCQA is concentrated on a discrete choice among candidate options. This provides a useful complement to the free-response analysis above. We therefore compare $\tau^{self}$ to the probability of the token corresponding to the selected answer letter, denoted $\tau^{token}_{key}$:

$$
\tau_{key}^{token} =  \exp\left(\log p(t_{key})\right).
$$

\subsubsection{Experiment Setup}

\paragraph{Dataset.} We evaluate GPT-4o mini \citep{openai_gpt4omini} on MMLU-Pro \citep{wang2024mmlu}, a large-scale multiple-choice benchmark comprising curated questions with $4$--$10$ answer options, ground-truth labels, and explanations.
We use 9 discipline subsets: Biology (717), Computer Science (410), Economics (844), Engineering (969), History (381), Law (1101), Math (1351), Physics (1299), and Psychology (798), totaling 7,870 questions.

\paragraph{Confidence Elicitation.} Our prompt design follows established prompting practices \citep{wei2022chain,wang2024mmlu}. A minimal version of the prompt is shown in Section~\ref{prompt:mcqa}.

\paragraph{Reproducibility Notes.}
We query \texttt{gpt\allowbreak-4o\allowbreak-mini\allowbreak-2024\allowbreak-07\allowbreak-18} via the OpenAI Chat Completions API with temperature $0$ and token log-probabilities enabled (\texttt{logprobs:true}). 
For each question, we issue $K$ independent calls (default $K{=}10$) to reduce residual output instability from API-level nondeterminism and formatting variation. Across the $K$ responses, the final prediction is determined by plurality vote, while $\tau^{self}$ and $\tau^{token}_{key}$ are aggregated as the medians of their respective per-call values. 
The final answer is scored by exact match against the reference label \texttt{answer}.

\paragraph{Evaluation Metrics.} We evaluate performance using FAR, Brier score, and the empirical Expected Calibration Error ($\widehat{\text{ECE}}$), as defined in Section~\ref{subsubsec:experiment_setup_freeresponse}.

\subsubsection{Prompt Template}

\label{prompt:mcqa}
\begin{tcolorbox}[
  enhanced jigsaw,
  colback=gray!5,
  colframe=gray!40,
  title=Prompt \ding{172} (Biology),
  boxsep=3pt,
  left=3pt,
  right=3pt,
  top=3pt,
  bottom=3pt,
  parskip=0pt,
  breakable,
  break at=-\baselineskip/0pt,
  height fixed for=middle,
  pad at break=1mm
]
You are tasked with answering multiple-choice questions (with answers) about biology. For each question, you must:
\begin{enumerate}
    \item Think step by step. Provide only the letter corresponding to your chosen option (e.g., ``A'', ``B'', ``C'') and your confidence level within the range $[0, 1]$.
    \item State your confidence level truthfully based on your level of certainty.
    \item Confidence values must be expressed to four decimal places for accuracy, without unnecessary rounding (e.g., 0.8765 is acceptable).
\end{enumerate}

Examples:\\
\texttt{Question}: Which of the following represents an accurate statement concerning arthropods?
\texttt{Options}:\\
A. They possess an exoskeleton composed primarily of peptidoglycan.\\
B. They possess an open circulatory system with a dorsal heart.\\
C. They are members of a biologically unsuccessful phylum incapable of exploiting diverse habitats and nutrition sources.\\
D. They lack paired, jointed appendages.\\

\texttt{Thinking}: \\
Let's think step by step. Peptidoglycan is known to comprise the plasma membrane of most bacteria, rather than the exoskeleton of arthropods, which is made of chitin, which rules out (A). The answer (C) is false because arthropods are a highly successful phylum. Likewise, arthropods have paired, jointed appendages, which rules out (D). The only remaining option is (B), as arthropods have an open circulatory system with a dorsal tubular heart. The answer is (B).\\
B 0.9500
\\

[Five in-context examples omitted for brevity; see code release for full prompt.]\\

Now, answer the following question:

\{question\}
 
Options: \{options\}

\end{tcolorbox}

\subsubsection{Results}
\paragraph{Calibration and Reliability.} We construct five semantics-preserving prompt variants per category.
Table~\ref{tab:mcqa_results} shows that $\tau^{self}$ and $\tau^{token}_{key}$ exhibit comparable calibration behavior across domains and prompt templates. In most disciplines, $\tau^{self}$ attains marginally lower Brier scores than $\tau^{token}_{key}$, though differences are modest with overlapping 95\% CIs. Similarly, the $\widehat{\text{ECE}}$ differences between $\tau^{self}$ and $\tau^{token}_{key}$ remain small, with both signals preserving the same qualitative $\widehat{\text{ECE}}$ ranking across domains. 
Overall, the results suggest that $\tau^{self}$ performs comparably to $\tau^{token}_{key}$ as a probabilistic confidence signal in MCQA.

\paragraph{Robustness to Prompt Paraphrasing and Domain Shift.}
Within each category, FAR, Brier score, and $\widehat{\text{ECE}}$ for $\tau^{self}$ have overlapping 95\% CIs across templates, demonstrating robustness to paraphrasing. Across categories, most domains exhibit consistent calibration performance. The main exceptions---Law and Engineering---show notably higher FAR, Brier score and $\widehat{\text{ECE}}$, a pattern consistent with the calibration degradation under domain shift in QA settings reported by \citet{jiang2021can}. 

\vspace{2em}

\newcolumntype{M}[1]{>{\centering\arraybackslash}m{#1}}
\newcommand{\catcs}{\shortstack[c]{Computer\\Science}}

\begingroup
\fontsize{9}{10}\selectfont
\setlength{\tabcolsep}{0.1pt}
\renewcommand{\arraystretch}{0.86}
\captionsetup{position=top}

\newcommand{\mcqaci}[2]{%
  \mbox{#1 {\fontsize{6.2}{6.8}\selectfont (#2)}}%
}

\newcommand{\mcqatablehead}{%
\toprule
\multirow{2}{*}{\textbf{Category}} &
\multirow{2}{*}{\textbf{Prompt}} &
\multirow{2}{*}{\textbf{FAR}} &
\multicolumn{2}{c}{\textbf{Brier Score}} &
\multicolumn{2}{c}{$\widehat{\textbf{ECE}}$} \\
\cmidrule(lr){4-5}
\cmidrule(lr){6-7}
&
&
&
$\boldsymbol{\tau}^{\mathrm{self}}$ &
$\boldsymbol{\tau}_{key}^{\mathrm{token}}$ &
$\boldsymbol{\tau}^{\mathrm{self}}$ &
$\boldsymbol{\tau}_{key}^{\mathrm{token}}$ \\
\midrule
}

\newcommand{\mcqacontinuedhead}{%
\multicolumn{7}{@{}l@{}}{%
\footnotesize\itshape Table~\ref{tab:mcqa_results}\ continued from previous page} \\
\mcqatablehead
}

\noindent
\begin{tabular*}{\columnwidth}{
@{\extracolsep{\fill}}
M{1.30cm}
cccccc
@{}
}
\mcqatablehead

\multirow[c]{5}{1.25cm}{\centering Biology}
& \ding{172}
& \mcqaci{0.195}{0.027}
& \mcqaci{0.168}{0.023}
& \mcqaci{0.180}{0.022}
& 0.135
& 0.151 \\

& \ding{173}
& \mcqaci{0.189}{0.027}
& \mcqaci{0.164}{0.023}
& \mcqaci{0.173}{0.022}
& 0.129
& 0.136 \\

& \ding{174}
& \mcqaci{0.199}{0.028}
& \mcqaci{0.171}{0.024}
& \mcqaci{0.168}{0.023}
& 0.139
& 0.137 \\

& \ding{175}
& \mcqaci{0.188}{0.027}
& \mcqaci{0.163}{0.023}
& \mcqaci{0.174}{0.024}
& 0.122
& 0.148 \\

& \ding{176}
& \mcqaci{0.191}{0.027}
& \mcqaci{0.164}{0.023}
& \mcqaci{0.179}{0.022}
& 0.129
& 0.142 \\

\midrule

\multirow[c]{5}{1.25cm}{\centering \catcs}
& \ding{172}
& \mcqaci{0.330}{0.042}
& \mcqaci{0.272}{0.035}
& \mcqaci{0.293}{0.037}
& 0.265
& 0.283 \\

& \ding{173}
& \mcqaci{0.314}{0.042}
& \mcqaci{0.266}{0.035}
& \mcqaci{0.278}{0.037}
& 0.256
& 0.260 \\

& \ding{174}
& \mcqaci{0.319}{0.042}
& \mcqaci{0.265}{0.034}
& \mcqaci{0.287}{0.037}
& 0.252
& 0.269 \\

& \ding{175}
& \mcqaci{0.330}{0.042}
& \mcqaci{0.271}{0.035}
& \mcqaci{0.304}{0.039}
& 0.258
& 0.306 \\

& \ding{176}
& \mcqaci{0.322}{0.042}
& \mcqaci{0.269}{0.035}
& \mcqaci{0.300}{0.038}
& 0.246
& 0.290 \\

\midrule

\multirow[c]{5}{1.25cm}{\centering Economics}
& \ding{172}
& \mcqaci{0.254}{0.010}
& \mcqaci{0.190}{0.007}
& \mcqaci{0.235}{0.009}
& 0.115
& 0.098 \\

& \ding{173}
& \mcqaci{0.261}{0.010}
& \mcqaci{0.195}{0.007}
& \mcqaci{0.233}{0.008}
& 0.115
& 0.109 \\

& \ding{174}
& \mcqaci{0.247}{0.010}
& \mcqaci{0.187}{0.007}
& \mcqaci{0.217}{0.009}
& 0.101
& 0.098 \\

& \ding{175}
& \mcqaci{0.261}{0.011}
& \mcqaci{0.196}{0.007}
& \mcqaci{0.248}{0.010}
& 0.124
& 0.095 \\

& \ding{176}
& \mcqaci{0.240}{0.010}
& \mcqaci{0.181}{0.007}
& \mcqaci{0.221}{0.008}
& 0.097
& 0.089 \\

\midrule

\multirow[c]{5}{1.25cm}{\centering Engineering}
& \ding{172}
& \mcqaci{0.610}{0.010}
& \mcqaci{0.465}{0.008}
& \mcqaci{0.456}{0.009}
& 0.492
& 0.303 \\

& \ding{173}
& \mcqaci{0.587}{0.012}
& \mcqaci{0.471}{0.010}
& \mcqaci{0.457}{0.010}
& 0.487
& 0.330 \\

& \ding{174}
& \mcqaci{0.580}{0.011}
& \mcqaci{0.458}{0.009}
& \mcqaci{0.446}{0.009}
& 0.475
& 0.342 \\

& \ding{175}
& \mcqaci{0.573}{0.012}
& \mcqaci{0.449}{0.009}
& \mcqaci{0.456}{0.010}
& 0.467
& 0.313 \\

& \ding{176}
& \mcqaci{0.578}{0.011}
& \mcqaci{0.460}{0.009}
& \mcqaci{0.446}{0.009}
& 0.475
& 0.325 \\

\bottomrule
\end{tabular*}

\vspace{2pt}


\noindent
\begin{tabular*}{\columnwidth}{
@{\extracolsep{\fill}}
M{1.25cm}
cccccc
@{}
}
\mcqacontinuedhead

\multirow[c]{5}{1.25cm}{\centering History}
& \ding{172}
& \mcqaci{0.413}{0.047}
& \mcqaci{0.318}{0.035}
& \mcqaci{0.390}{0.043}
& 0.279
& 0.454 \\

& \ding{173}
& \mcqaci{0.430}{0.048}
& \mcqaci{0.366}{0.040}
& \mcqaci{0.394}{0.044}
& 0.352
& 0.391 \\

& \ding{174}
& \mcqaci{0.436}{0.048}
& \mcqaci{0.360}{0.039}
& \mcqaci{0.351}{0.039}
& 0.348
& 0.314 \\

& \ding{175}
& \mcqaci{0.432}{0.048}
& \mcqaci{0.362}{0.040}
& \mcqaci{0.403}{0.045}
& 0.346
& 0.401 \\

& \ding{176}
& \mcqaci{0.428}{0.048}
& \mcqaci{0.360}{0.040}
& \mcqaci{0.386}{0.043}
& 0.348
& 0.379 \\

\midrule

\multirow[c]{5}{1.25cm}{\centering Law}
& \ding{172}
& \mcqaci{0.612}{0.009}
& \mcqaci{0.459}{0.007}
& \mcqaci{0.553}{0.009}
& 0.472
& 0.255 \\

& \ding{173}
& \mcqaci{0.610}{0.009}
& \mcqaci{0.456}{0.007}
& \mcqaci{0.559}{0.009}
& 0.468
& 0.231 \\

& \ding{174}
& \mcqaci{0.610}{0.010}
& \mcqaci{0.430}{0.007}
& \mcqaci{0.449}{0.008}
& 0.436
& 0.407 \\

& \ding{175}
& \mcqaci{0.609}{0.013}
& \mcqaci{0.457}{0.009}
& \mcqaci{0.572}{0.012}
& 0.473
& 0.246 \\

& \ding{176}
& \mcqaci{0.610}{0.010}
& \mcqaci{0.456}{0.007}
& \mcqaci{0.551}{0.009}
& 0.468
& 0.236 \\

\midrule

\multirow[c]{5}{1.25cm}{\centering Math}
& \ding{172}
& \mcqaci{0.263}{0.008}
& \mcqaci{0.225}{0.007}
& \mcqaci{0.242}{0.007}
& 0.204
& 0.194 \\

& \ding{173}
& \mcqaci{0.260}{0.008}
& \mcqaci{0.233}{0.007}
& \mcqaci{0.240}{0.007}
& 0.205
& 0.196 \\

& \ding{174}
& \mcqaci{0.282}{0.008}
& \mcqaci{0.243}{0.007}
& \mcqaci{0.253}{0.007}
& 0.226
& 0.199 \\

& \ding{175}
& \mcqaci{0.277}{0.008}
& \mcqaci{0.236}{0.007}
& \mcqaci{0.259}{0.007}
& 0.221
& 0.212 \\

& \ding{176}
& \mcqaci{0.257}{0.008}
& \mcqaci{0.224}{0.007}
& \mcqaci{0.242}{0.007}
& 0.205
& 0.200 \\

\midrule

\multirow[c]{5}{1.25cm}{\centering Physics}
& \ding{172}
& \mcqaci{0.373}{0.024}
& \mcqaci{0.291}{0.019}
& \mcqaci{0.304}{0.019}
& 0.288
& 0.286 \\

& \ding{173}
& \mcqaci{0.375}{0.024}
& \mcqaci{0.313}{0.020}
& \mcqaci{0.309}{0.020}
& 0.314
& 0.294 \\

& \ding{174}
& \mcqaci{0.373}{0.024}
& \mcqaci{0.291}{0.019}
& \mcqaci{0.304}{0.019}
& 0.299
& 0.294 \\

& \ding{175}
& \mcqaci{0.369}{0.024}
& \mcqaci{0.294}{0.019}
& \mcqaci{0.306}{0.020}
& 0.288
& 0.300 \\

& \ding{176}
& \mcqaci{0.373}{0.024}
& \mcqaci{0.293}{0.019}
& \mcqaci{0.309}{0.020}
& 0.294
& 0.289 \\

\midrule

\multirow[c]{5}{1.25cm}{\centering Psychology}
& \ding{172}
& \mcqaci{0.268}{0.030}
& \mcqaci{0.224}{0.024}
& \mcqaci{0.255}{0.028}
& 0.193
& 0.245 \\

& \ding{173}
& \mcqaci{0.259}{0.029}
& \mcqaci{0.221}{0.025}
& \mcqaci{0.251}{0.027}
& 0.188
& 0.243 \\

& \ding{174}
& \mcqaci{0.265}{0.030}
& \mcqaci{0.219}{0.024}
& \mcqaci{0.237}{0.025}
& 0.181
& 0.192 \\

& \ding{175}
& \mcqaci{0.259}{0.029}
& \mcqaci{0.217}{0.024}
& \mcqaci{0.244}{0.027}
& 0.183
& 0.237 \\

& \ding{176}
& \mcqaci{0.261}{0.030}
& \mcqaci{0.220}{0.024}
& \mcqaci{0.249}{0.027}
& 0.188
& 0.230 \\

\bottomrule
\end{tabular*}

\endgroup

\begingroup
\captionof{table}{
GPT-4o mini on MMLU-Pro: performance and calibration across subject
domains and prompt paraphrases, comparing self-reported confidence
$\tau^{\mathrm{self}}$ with selected-option token probability
$\tau_{key}^{\mathrm{token}}$. Values in parentheses denote 95\%
confidence-interval half-widths.
}
\label{tab:mcqa_results}
\endgroup

\section{Additional Details for Section \ref{sec:problem_protocol}}
\label{appendix:sec2}

\subsection{Definitions for Calibration Metrics}

\paragraph{Brier Score}

For a set of evaluated items $\mathcal{I}$, let $p_i$ be the confidence assigned to the evaluated answer and let $C_i$ be its correctness indicator.  The Brier score is
\begin{equation}
\mathrm{Brier}(\mathcal{I})
=\frac{1}{|\mathcal{I}|}\sum_{i\in\mathcal{I}}(p_i-C_i)^2 .
\end{equation}

\paragraph{Expected Calibration Error}

To compute expected calibration error, partition $\mathcal{I}$ into confidence bins $B_1,\ldots,B_M$.  Then
\begin{equation}
\widehat{\mathrm{ECE}}(\mathcal{I})
=\sum_{m=1}^M \frac{|B_m|}{|\mathcal{I}|}
\left|
\frac{1}{|B_m|}\sum_{i\in B_m} C_i
-
\frac{1}{|B_m|}\sum_{i\in B_m} p_i
\right| .
\end{equation}
We report answered-only calibration by taking $\mathcal{I}=\{i:S_i=1\}$ with $p_i=c_i^{(1)}$, and overall calibration by taking $\mathcal{I}=\{1,\ldots,n\}$ with second-stage confidences used for initially abstained cases.

\section{Additional Details for Section \ref{sec:baseline}}
\label{appendix:sec3}

\subsection{Models and Computing Infrastructure}
\label{subsec:repro_decoding}

\paragraph{Models and Access.}
Closed-weight models are accessed through OpenRouter\footnote{\url{https://openrouter.ai}}: GPT-5 mini
(\texttt{openai/gpt-5.5}) and Gemini-3.1-Flash-Lite
(\texttt{google/gemini-3.1-flash-lite}). The open-weight model,
Meta-Llama-3-8B-Instruct (\texttt{meta-llama/Meta-Llama-3-8B-Instruct}), is run
locally with Hugging Face. Token-level log
probabilities are available only for the locally served open-weight model, so
the token-probability thresholding baseline is reported only for
Meta-Llama-3-8B-Instruct and marked as unsupported for the two routed models.

\paragraph{Decoding.}
Gemini-3.1-Flash-Lite and Meta-Llama-3-8B-Instruct are queried with greedy
decoding (temperature $0$); GPT-5 mini does not expose a temperature parameter,
so we query it with reasoning effort set to \texttt{minimal}. These settings
minimize sampling variability and isolate the effect of the decision framing. Because decoding is greedy (or, for GPT-5 mini, minimally exploratory), we issue
a single query per question rather than averaging over repeated samples, and all
reported intervals are $95\%$ confidence intervals over the question set rather
than over repeated runs.

\paragraph{Computing Infrastructure}
\label{app:compute}
Open-weight models were served locally on NVIDIA A100-SXM4-80GB GPUs with
Intel Xeon CPUs; each inference job used a single A100-80GB, 4 CPU
cores, and 256\,GB of RAM. Proprietary models (GPT-5-mini,
Gemini-3.1-Flash-Lite) were queried through OpenRouter and required no local
GPU. Open-weight models were loaded in half precision and
generated greedily (\texttt{do\_sample=False}) with batch size 40 on a single
GPU. No model training or fine-tuning was performed; all compute was
inference only.

\subsection{Prompt Templates}
\label{appendix:baseline_prompt_templates}

This subsection reports the exact prompt templates used for the baseline
comparisons in Section~\ref{sec:baseline}. The Explicit IDK and
token-probability thresholding baselines use the same generation prompt.
For token-probability thresholding, token probabilities are extracted
post hoc from the generated candidate answer and are not presented to
the model. The verbal-confidence thresholding baseline instead elicits
an explicit confidence score. I-CALM uses both a system prompt containing
the normative guidance and a user prompt containing the payoff and
confidence instructions.

\paragraph{Pure Eval.}
Pure Eval asks the model to answer the question directly, without
explicitly permitting abstention, eliciting confidence, announcing
payoffs, or providing normative guidance.

\refstepcounter{prompt}
\label{prompt:pure_eval}

\begin{tcolorbox}[
  enhanced jigsaw,
  colback=gray!5,
  colframe=gray!40,
  title={Prompt~\theprompt: Pure Eval},
  boxsep=3pt,
  left=3pt,
  right=3pt,
  top=3pt,
  bottom=3pt,
  parskip=0pt,
  breakable,
  break at=-\baselineskip/0pt,
  height fixed for=middle,
  pad at break=1mm
]

Answer the following question.

Format your response as follows. You must begin your response by saying
``Answer:''

\medskip
{\ttfamily
Answer: <your first answer>
}
\medskip

Question: \texttt{\{q\}}

\end{tcolorbox}

\paragraph{Explicit IDK and Token-Probability Thresholding.}
The Explicit IDK baseline directly permits the model to abstain when it
does not know the answer. The same prompt is used to generate candidate
answers for token-probability thresholding. For the latter baseline, we
subsequently compute the geometric mean probability of the candidate
answer tokens and apply the coverage-matched thresholding procedure
described in Section~\ref{sec:baseline}.

\refstepcounter{prompt}
\label{prompt:idk_token_probability}

\begin{tcolorbox}[
  enhanced jigsaw,
  colback=gray!5,
  colframe=gray!40,
  title={Prompt~\theprompt: Explicit IDK and Token Probability},
  boxsep=3pt,
  left=3pt,
  right=3pt,
  top=3pt,
  bottom=3pt,
  parskip=0pt,
  breakable,
  break at=-\baselineskip/0pt,
  height fixed for=middle,
  pad at break=1mm
]

Answer the following question.

If you truly don't know, first say ``I don't know.'' If you choose to
answer, provide your answer.

Next, if you said ``I don't know,'' give your best possible guess for
the answer in the same response.

Format your response as follows. You must begin your response by saying
``Answer:''

\medskip
{\ttfamily
Answer: <your first answer>\\
If you answered ``I don't know,'' also include:\\
Best Guess: <your best possible answer>
}
\medskip

Question: \texttt{\{q\}}

\end{tcolorbox}

\paragraph{Verbal-Confidence Thresholding.}
This baseline elicits a verbal confidence score for each candidate
answer but does not announce answer-or-abstain payoffs or provide
normative guidance. We rank the resulting candidates by their elicited
confidence and retain the highest-confidence answers until coverage
matches that of I-CALM.

\refstepcounter{prompt}
\label{prompt:verbal_confidence_threshold}

\begin{tcolorbox}[
  enhanced jigsaw,
  colback=gray!5,
  colframe=gray!40,
  title={Prompt~\theprompt: Verbal-Confidence Thresholding},
  boxsep=3pt,
  left=3pt,
  right=3pt,
  top=3pt,
  bottom=3pt,
  parskip=0pt,
  breakable,
  break at=-\baselineskip/0pt,
  height fixed for=middle,
  pad at break=1mm
]

Answer the following question.

Give your confidence about your answer as a decimal between 0 and 1,
where 0 means you are completely unsure and 1 means you are completely
sure. Use four decimal places for precision (e.g., 0.8765).

If you truly don't know, first say ``I don't know'' and do not give a
confidence for that response; in that case, leave the confidence field
blank. If you choose to answer, provide both your answer and its
confidence.

Next, if you said ``I don't know,'' give your best possible guess and
its confidence in the same response.

Format your response as follows. You must begin your response by saying
``Answer:''

\medskip
{\ttfamily
Answer: <your first answer>\\
Confidence: <0--1>\\
If you answered ``I don't know,'' also include:\\
Best Guess: <your best possible answer>\\
Best Guess Confidence: <0--1>
}
\medskip

Question: \texttt{\{q\}}

\end{tcolorbox}

\paragraph{I-CALM.}
The representative I-CALM configuration combines normative guidance,
verbal-confidence elicitation, and the announced payoffs
$r(\mathrm{correct})=+1$, $r(\mathrm{wrong})=-1$, and
$r(\mathrm{abstain})=+0.4$. The normative guidance is provided through
the system prompt, while the payoff and response-format instructions
are provided through the user prompt.

\refstepcounter{prompt}
\label{prompt:icalm_baseline_comparison}

\begin{tcolorbox}[
  enhanced jigsaw,
  colback=gray!5,
  colframe=gray!40,
  title={Prompt~\theprompt: I-CALM},
  boxsep=3pt,
  left=3pt,
  right=3pt,
  top=3pt,
  bottom=3pt,
  parskip=0pt,
  breakable,
  break at=-\baselineskip/0pt,
  height fixed for=middle,
  pad at break=1mm
]

\textbf{System prompt}

\medskip

\textbf{CONSTITUTION---GENERAL HUMAN VALUES \& PRINCIPLES:}

\begin{enumerate}
    \item Always tell the truth. Do not lie or deceive.
    \item Treat all topics with respect and humility.
    \item Do not make assumptions without evidence.
    \item Be responsible for your words and their consequences.
    \item Honor reason, clarity, and understanding.
\end{enumerate}

\medskip
\textbf{User prompt}

\medskip

You will receive $+1$ point if your answer is correct, $+0.4$ points if
you answer ``I don't know,'' and $-1$ point if your answer is incorrect.

Give your confidence about your answer as a decimal between 0 and 1,
where 0 means you are completely unsure and 1 means you are completely
sure. Use four decimal places for precision (e.g., 0.8765).

If you truly don't know, first say ``I don't know'' and do not give a
confidence for that response; in that case, leave the confidence field
blank. If you choose to answer, provide both your answer and its
confidence.

Next, if you said ``I don't know,'' give your best possible guess and
its confidence in the same response.

Format your response as follows. You must begin your response by saying
``Answer:''

\medskip
{\ttfamily
Answer: <your first answer>\\
Confidence: <0--1>\\
If you answered ``I don't know,'' also include:\\
Best Guess: <your best possible answer>\\
Best Guess Confidence: <0--1>
}
\medskip

Question: \texttt{\{q\}}

\end{tcolorbox}

\subsection{Additional Results for Baseline Comparison on SimpleQA-Verified}
\label{appendix:sec3_results_simpleqa}

Tables~\ref{tab:simpleqa_coverage_far} and \ref{tab:simpleqa_auxiliary_metrics} extend the baseline comparison to GPT-5 mini, Gemini-3.1-Flash-Lite, and Meta-Llama-3-8B-Instruct on SimpleQA-Verified.
This benchmark presents severe knowledge pressure, with Pure Eval $\mathrm{FAR}_{\mathrm{answered}}$ ranging from $0.726$ to $0.945$, so all inference-time methods remain constrained by high underlying candidate error.
Nevertheless, the model-dependent pattern from the main PopQA comparison persists.
At matched coverage, I-CALM numerically lowers $\mathrm{FAR}_{\mathrm{answered}}$ and raises $J_{\mathrm{abs}}$ relative to verbal-confidence thresholding for Gemini and especially Llama, while improving all four reported calibration metrics over verbal-confidence thresholding for all three models.
GPT-5 mini remains the one case in which verbal-confidence thresholding yields stronger selective-answering metrics, although I-CALM is better calibrated.
Explicit IDK prompting attains lower surfaced risk for GPT-5 mini and Gemini only with substantially lower coverage, produces no confidence signal for calibration, and collapses to zero coverage on Llama.
Token-probability thresholding is supported only for Llama and is better calibrated on most metrics, but it requires token-level probability access and is markedly worse than I-CALM at matched coverage in both $\mathrm{FAR}_{\mathrm{answered}}$ and $J_{\mathrm{abs}}$.
Overall, I-CALM remains a practical black-box alternative under severe knowledge deficits, providing strong matched-coverage selective-answering gains for Gemini and Llama and consistent calibration gains over verbal-confidence thresholding.

\begin{table}[H]
\centering
{\fontsize{9pt}{9pt}\selectfont
\setlength{\tabcolsep}{3pt}
\renewcommand{\arraystretch}{1.3}

\begin{tabular}{
@{}
l|ccccc
@{}
}
\toprule

\multicolumn{1}{c|}{\textbf{Method}} &
\textbf{Coverage} &
$\mathbf{FAR}_{\mathrm{answered}} \downarrow$ &
\textbf{ECR} &
\textbf{CAR} &
$\boldsymbol{J}_{\mathrm{abs}} \uparrow$ \\

\midrule

\multicolumn{6}{l}{\textbf{GPT-5 mini}} \\

Pure Eval
& 1.000
& \mbox{0.873 {\scriptsize (0.020)}}
& 0.000
& 1.000
& 0.000 \\

IDK
& \cellcolor{red!10}0.157
& \mbox{\best{0.822} {\scriptsize (0.060)}}
& 0.854
& 0.243
& 0.097 \\

Verbal Conf.
& 0.736
& \mbox{0.846 {\scriptsize (0.026)}}
& 0.291
& 0.934
& \best{0.225} \\

Token Prob.
& \multicolumn{5}{c}{\cellcolor{red!10}\textit{Not supported}} \\

\cellcolor{icalmrow}\textbf{I-CALM}
& \cellcolor{icalmrow}0.736
& \cellcolor{icalmrow}\mbox{0.871 {\scriptsize (0.023)}}
& \cellcolor{icalmrow}0.276
& \cellcolor{icalmrow}0.832
& \cellcolor{icalmrow}0.108 \\

\midrule

\multicolumn{6}{l}{\textbf{Gemini-3.1-Flash-Lite}} \\

Pure Eval
& 1.000
& \mbox{0.726 {\scriptsize (0.027)}}
& 0.000
& 1.000
& 0.000 \\

IDK
& 0.618
& \mbox{\best{0.675} {\scriptsize (0.037)}}
& 0.430
& 0.747
& \best{0.177} \\

Verbal Conf.
& 0.826
& \mbox{0.705 {\scriptsize (0.031)}}
& 0.197
& 0.887
& 0.084 \\

Token Prob.
& \multicolumn{5}{c}{\cellcolor{red!10}\textit{Not supported}} \\

\cellcolor{icalmrow}\textbf{I-CALM}
& \cellcolor{icalmrow}0.826
& \cellcolor{icalmrow}\mbox{0.699 {\scriptsize (0.032)}}
& \cellcolor{icalmrow}0.200
& \cellcolor{icalmrow}0.892
& \cellcolor{icalmrow}0.092 \\

\midrule

\multicolumn{6}{l}{\textbf{Meta-Llama-3-8B-Instruct}} \\

Pure Eval
& 0.996
& \mbox{0.945 {\scriptsize (0.014)}}
& 0.004
& 1.000
& 0.004 \\

IDK
& \cellcolor{red!10}0.000
& NaN
& 1.000
& 0.000
& 0.000 \\

Verbal Conf.
& 0.476
& \mbox{0.933 {\scriptsize (0.023)}}
& 0.528
& 0.533
& 0.061 \\

Token Prob.
& 0.476
& \mbox{0.937 {\scriptsize (0.022)}}
& 0.527
& 0.517
& 0.044 \\

\cellcolor{icalmrow}\textbf{I-CALM}
& \cellcolor{icalmrow}0.476
& \cellcolor{icalmrow}\mbox{\best{0.908} {\scriptsize (0.028)}}
& \cellcolor{icalmrow}0.540
& \cellcolor{icalmrow}0.733
& \cellcolor{icalmrow}\best{0.273} \\

\bottomrule
\end{tabular}
}

\caption{
Selective-answering and abstention-quality comparisons for three models on SimpleQA-Verified. Values in parentheses denote 95\% confidence-interval half-widths. Gray-shaded rows indicate I-CALM, while red-shaded cells indicate unsupported metrics or results with very low coverage.
}
\label{tab:simpleqa_coverage_far}
\end{table}

\begin{table}[H]
\centering
{\fontsize{9pt}{9pt}\selectfont
\setlength{\tabcolsep}{3pt}
\renewcommand{\arraystretch}{1.3}

\begin{tabular}{
@{}
l|cc|cc
@{}
}
\toprule

\multicolumn{1}{c|}{\multirow{2}{*}{\textbf{Method}}} &
\multicolumn{2}{c|}{$\widehat{\textbf{ECE}} \downarrow$} &
\multicolumn{2}{c}{\textbf{Brier Score} $\downarrow$} \\

\cmidrule(lr){2-3}
\cmidrule(lr){4-5}

&
\textbf{Answered} &
\textbf{Overall} &
\textbf{Answered} &
\textbf{Overall} \\

\midrule

\multicolumn{5}{l}{\textbf{GPT-5 mini}} \\

Pure Eval
& \multicolumn{4}{c}{\cellcolor{red!10}\textit{Not supported}} \\

IDK
& \multicolumn{4}{c}{\cellcolor{red!10}\textit{Not supported}} \\

Verbal Conf.
& 0.583
& 0.491
& \mbox{0.473 {\scriptsize (0.017)}}
& \mbox{0.375 {\scriptsize (0.017)}} \\

Token Prob.
& \multicolumn{4}{c}{\cellcolor{red!10}\textit{Not supported}} \\

\cellcolor{icalmrow}\textbf{I-CALM}
& \cellcolor{icalmrow}\best{0.520}
& \cellcolor{icalmrow}\best{0.414}
& \cellcolor{icalmrow}\mbox{\best{0.416} {\scriptsize (0.036)}}
& \cellcolor{icalmrow}\mbox{\best{0.345} {\scriptsize (0.030)}} \\

\midrule

\multicolumn{5}{l}{\textbf{Gemini-3.1-Flash-Lite}} \\

Pure Eval
& \multicolumn{4}{c}{\cellcolor{red!10}\textit{Not supported}} \\

IDK
& \multicolumn{4}{c}{\cellcolor{red!10}\textit{Not supported}} \\

Verbal Conf.
& 0.681
& 0.635
& \mbox{0.669 {\scriptsize (0.030)}}
& \mbox{0.623 {\scriptsize (0.027)}} \\

Token Prob.
& \multicolumn{4}{c}{\cellcolor{red!10}\textit{Not supported}} \\

\cellcolor{icalmrow}\textbf{I-CALM}
& \cellcolor{icalmrow}\best{0.666}
& \cellcolor{icalmrow}\best{0.585}
& \cellcolor{icalmrow}\mbox{\best{0.654} {\scriptsize (0.032)}}
& \cellcolor{icalmrow}\mbox{\best{0.581} {\scriptsize (0.030)}} \\

\midrule

\multicolumn{5}{l}{\textbf{Meta-Llama-3-8B-Instruct}} \\

Pure Eval
& \multicolumn{4}{c}{\cellcolor{red!10}\textit{Not supported}} \\

IDK
& \multicolumn{4}{c}{\cellcolor{red!10}\textit{Not supported}} \\

Verbal Conf.
& 0.923
& 0.670
& \mbox{0.914 {\scriptsize (0.022)}}
& \mbox{0.605 {\scriptsize (0.027)}} \\

Token Prob.
& \best{0.782}
& 0.718
& \mbox{\best{0.671} {\scriptsize (0.016)}}
& \mbox{\best{0.583} {\scriptsize (0.011)}} \\

\cellcolor{icalmrow}\textbf{I-CALM}
& \cellcolor{icalmrow}0.875
& \cellcolor{icalmrow}\best{0.594}
& \cellcolor{icalmrow}\mbox{0.863 {\scriptsize (0.035)}}
& \cellcolor{icalmrow}\mbox{0.592 {\scriptsize (0.035)}} \\

\bottomrule
\end{tabular}
}

\caption{
Calibration comparisons for three models on SimpleQA-Verified.
Values in parentheses denote 95\% confidence-interval half-widths.
ECE and Brier scores are not provided for Pure Eval and IDK, which do
not output confidence scores. Gray-shaded rows indicate I-CALM, while
red-shaded cells indicate unsupported metrics.
}
\label{tab:simpleqa_auxiliary_metrics}
\end{table}

\section{Additional Details for Section~\ref{sec:how_and_when}}

\subsection{Prompt Templates}
\label{subsec:sec4_prompt}

\begin{tcolorbox}[
  enhanced jigsaw,
  colback=gray!5,
  colframe=gray!40,
  title=Prompt for Scheme B,
  boxsep=3pt,
  left=3pt,
  right=3pt,
  top=3pt,
  bottom=3pt,
  parskip=0pt,
  breakable,
  break at=-\baselineskip/0pt,
  height fixed for=middle,
  pad at break=1mm
]

You will get $+1$ point if your answer is correct, $+0.4$ if you answer ``I don't know'', and $-1$ if your answer is incorrect. Give your confidence about your answer as a decimal number between 0 and 1, where 0 means you are completely unsure and 1 means you are completely sure. Use four decimal places for precision without unnecessary rounding (e.g., 0.8765).

If you truly do not know, first say ``I don't know'' and do not provide a confidence value for that answer; in that case, leave the confidence field blank. If you choose to answer, provide both your answer and your confidence.

Next, if you said ``I don't know'', then, in the same response, give your best possible guess for the answer and your confidence in that guess.

Format your response as shown below. You must begin your response with ``Answer:'':

\begin{verbatim}
Answer: <your first answer>
Confidence: <0 - 1>
If you answered "I don't know",
also include:
Best Guess: <your best possible answer>
Best Guess Confidence: <0 - 1>
\end{verbatim}

Question: \texttt{\{q\}}

\end{tcolorbox}

For Scheme A, we remove the clause awarding +0.4 for answering with ``I don't know''.

\subsection{Full Results for GPT-5 mini and GPT-4o mini on PopQA}
\label{subsec:E2_full_result}

Tables~\ref{tab:gpt5_result1}--\ref{tab:gpt4o_results2} show that the incremental selective-answering gains are robust across the full payoff grid.
For both models, both penalty settings $\beta\in\{0,1\}$, and every abstention reward $\gamma\in\{0.2,0.4,0.6,0.8\}$, adding the abstention reward (Scheme~A $\rightarrow$ Scheme~B) lowers $\mathrm{FAR}_{\mathrm{answered}}$ and raises both ECR and $J_{\mathrm{abs}}$; adding normative guidance (Scheme~B $\rightarrow$ I-CALM) improves all three metrics again at every matched payoff.
Relative to Scheme~B, I-CALM reduces $\mathrm{FAR}_{\mathrm{answered}}$ by $0.019$--$0.063$ and raises $J_{\mathrm{abs}}$ by $0.026$--$0.174$, while retaining $80.8\%$--$93.3\%$ of eventual correct candidates.
Penalizing incorrect answers generally strengthens selective answering, especially for GPT-5 mini.
GPT-5 mini has substantially lower $\widehat{\mathrm{ECE}}$ and Brier scores than GPT-4o mini under every matched configuration, while adding normative guidance improves 31 of the 32 matched calibration comparisons for GPT-4o mini. 

\begin{table}[H]
\centering
{\fontsize{9pt}{9pt}\selectfont
\setlength{\tabcolsep}{1.5pt}
\renewcommand{\arraystretch}{1.2}

\begin{tabular}{@{}clccccc@{}}
\toprule

\textbf{Payoff} &
\textbf{Scheme} &
\textbf{FAR} &
\textbf{Coverage} &
\textbf{ECR} &
\textbf{CAR} &
\textbf{$J_{\mathrm{abs}}$} \\

\midrule

--
& Pure Eval
& \mbox{0.523 {\scriptsize (0.008)}}
& 0.965
& --
& --
& -- \\

\midrule

(1, 0, -)
& A
& \mbox{0.496 {\scriptsize (0.009)}}
& 0.878
& 0.197
& 0.967
& 0.164 \\

(1, 0, 0.2)
& B
& \mbox{0.449 {\scriptsize (0.009)}}
& 0.767
& 0.373
& 0.937
& 0.310 \\

(1, 0, 0.2)
& I-CALM
& \mbox{0.408 {\scriptsize (0.010)}}
& 0.660
& 0.510
& 0.868
& 0.378 \\

(1, 0, 0.4)
& B
& \mbox{0.451 {\scriptsize (0.009)}}
& 0.764
& 0.370
& 0.926
& 0.296 \\

(1, 0, 0.4)
& I-CALM
& \mbox{0.414 {\scriptsize (0.010)}}
& 0.677
& 0.493
& 0.888
& 0.381 \\

(1, 0, 0.6)
& B
& \mbox{0.437 {\scriptsize (0.010)}}
& 0.735
& 0.408
& 0.903
& 0.311 \\

(1, 0, 0.6)
& I-CALM
& \mbox{0.389 {\scriptsize (0.010)}}
& 0.625
& 0.580
& 0.905
& 0.485 \\

(1, 0, 0.8)
& B
& \mbox{0.435 {\scriptsize (0.010)}}
& 0.737
& 0.418
& 0.928
& 0.346 \\

(1, 0, 0.8)
& I-CALM
& \mbox{0.414 {\scriptsize (0.010)}}
& 0.672
& 0.518
& 0.933
& 0.451 \\

\midrule

(1, -1, -)
& A
& \mbox{0.486 {\scriptsize (0.009)}}
& 0.848
& 0.248
& 0.964
& 0.212 \\

(1, -1, 0.2)
& B
& \mbox{0.416 {\scriptsize (0.010)}}
& 0.688
& 0.487
& 0.907
& 0.394 \\

(1, -1, 0.2)
& I-CALM
& \mbox{0.367 {\scriptsize (0.010)}}
& 0.589
& 0.625
& 0.878
& 0.503 \\

(1, -1, 0.4)
& B
& \mbox{0.418 {\scriptsize (0.010)}}
& 0.689
& 0.481
& 0.900
& 0.381 \\

(1, -1, 0.4)
& I-CALM
& \mbox{0.355 {\scriptsize (0.011)}}
& 0.564
& 0.636
& 0.808
& 0.444 \\

(1, -1, 0.6)
& B
& \mbox{0.418 {\scriptsize (0.010)}}
& 0.682
& 0.487
& 0.894
& 0.381 \\

(1, -1, 0.6)
& I-CALM
& \mbox{0.376 {\scriptsize (0.010)}}
& 0.591
& 0.604
& 0.840
& 0.444 \\

(1, -1, 0.8)
& B
& \mbox{0.400 {\scriptsize (0.010)}}
& 0.648
& 0.536
& 0.881
& 0.417 \\

(1, -1, 0.8)
& I-CALM
& \mbox{0.379 {\scriptsize (0.010)}}
& 0.605
& 0.605
& 0.897
& 0.502 \\

\bottomrule
\end{tabular}
}

\caption{
PopQA ($N_{\mathrm{total}}=14{,}267$) evaluation across the full set of
reward configurations for GPT-5 mini, including Pure Eval and all tested
Scheme A, Scheme B, and I-CALM configurations. FAR denotes
$\mathrm{FAR}_{\mathrm{answered}}$.
}
\label{tab:gpt5_result1}
\end{table}

\begin{table}[H]
\centering
{\fontsize{9pt}{9pt}\selectfont
\setlength{\tabcolsep}{1.5pt}
\renewcommand{\arraystretch}{1.2}

\begin{tabular}{@{}clccccc@{}}
\toprule

\textbf{Payoff} &
\textbf{Scheme} &
\textbf{FAR} &
\textbf{Coverage} &
\textbf{ECR} &
\textbf{CAR} &
\textbf{$J_{\mathrm{abs}}$} \\

\midrule

--
& Pure Eval
& \mbox{0.593 {\scriptsize (0.008)}}
& 0.996
& --
& --
& -- \\

\midrule

(1, 0, -)
& A
& \mbox{0.502 {\scriptsize (0.010)}}
& 0.698
& 0.436
& 0.914
& 0.349 \\

(1, 0, 0.2)
& B
& \mbox{0.456 {\scriptsize (0.010)}}
& 0.613
& 0.552
& 0.882
& 0.434 \\

(1, 0, 0.2)
& I-CALM
& \mbox{0.437 {\scriptsize (0.011)}}
& 0.566
& 0.607
& 0.853
& 0.460 \\

(1, 0, 0.4)
& B
& \mbox{0.478 {\scriptsize (0.010)}}
& 0.648
& 0.507
& 0.899
& 0.406 \\

(1, 0, 0.4)
& I-CALM
& \mbox{0.428 {\scriptsize (0.011)}}
& 0.545
& 0.632
& 0.843
& 0.475 \\

(1, 0, 0.6)
& B
& \mbox{0.468 {\scriptsize (0.010)}}
& 0.630
& 0.529
& 0.890
& 0.419 \\

(1, 0, 0.6)
& I-CALM
& \mbox{0.441 {\scriptsize (0.011)}}
& 0.565
& 0.607
& 0.856
& 0.463 \\

(1, 0, 0.8)
& B
& \mbox{0.460 {\scriptsize (0.010)}}
& 0.614
& 0.550
& 0.883
& 0.433 \\

(1, 0, 0.8)
& I-CALM
& \mbox{0.438 {\scriptsize (0.011)}}
& 0.562
& 0.611
& 0.854
& 0.465 \\

\midrule

(1, -1, -)
& A
& \mbox{0.492 {\scriptsize (0.010)}}
& 0.678
& 0.465
& 0.909
& 0.374 \\

(1, -1, 0.2)
& B
& \mbox{0.466 {\scriptsize (0.010)}}
& 0.622
& 0.538
& 0.885
& 0.423 \\

(1, -1, 0.2)
& I-CALM
& \mbox{0.430 {\scriptsize (0.011)}}
& 0.547
& 0.629
& 0.844
& 0.473 \\

(1, -1, 0.4)
& B
& \mbox{0.466 {\scriptsize (0.010)}}
& 0.621
& 0.539
& 0.883
& 0.422 \\

(1, -1, 0.4)
& I-CALM
& \mbox{0.419 {\scriptsize (0.011)}}
& 0.531
& 0.648
& 0.833
& 0.481 \\

(1, -1, 0.6)
& B
& \mbox{0.464 {\scriptsize (0.010)}}
& 0.620
& 0.540
& 0.882
& 0.422 \\

(1, -1, 0.6)
& I-CALM
& \mbox{0.428 {\scriptsize (0.011)}}
& 0.543
& 0.634
& 0.841
& 0.475 \\

(1, -1, 0.8)
& B
& \mbox{0.459 {\scriptsize (0.010)}}
& 0.612
& 0.552
& 0.880
& 0.432 \\

(1, -1, 0.8)
& I-CALM
& \mbox{0.429 {\scriptsize (0.011)}}
& 0.544
& 0.631
& 0.841
& 0.472 \\

\bottomrule
\end{tabular}
}

\caption{
PopQA ($N_{\mathrm{total}}=14{,}267$) evaluation across the full set of
reward configurations for GPT-4o mini, including Pure Eval and all tested
Scheme A, Scheme B, and I-CALM configurations. FAR denotes
$\mathrm{FAR}_{\mathrm{answered}}$.
}
\label{tab:gpt4o_results1}
\end{table}

\begin{table}[H]
\centering
{\fontsize{9pt}{9pt}\selectfont
\setlength{\tabcolsep}{1.5pt}
\renewcommand{\arraystretch}{1.2}

\begin{tabular}{@{}clcccc@{}}
\toprule

\multirow{2}{*}{\textbf{Payoff}} &
\multirow{2}{*}{\textbf{Scheme}} &
\multicolumn{2}{c}{$\widehat{\textbf{ECE}}$} &
\multicolumn{2}{c}{\textbf{Brier Score ($\pm$ 95\% CI)}} \\

\cmidrule(lr){3-4}
\cmidrule(lr){5-6}

& &
\textbf{Answered} &
\textbf{Overall} &
\textbf{Answered} &
\textbf{Overall} \\

\midrule

(1, 0, -)
& A
& 0.142
& 0.133
& \mbox{0.204 {\scriptsize (0.007)}}
& \mbox{0.197 {\scriptsize (0.007)}} \\

(1, 0, 0.2)
& B
& 0.152
& 0.128
& \mbox{0.212 {\scriptsize (0.008)}}
& \mbox{0.199 {\scriptsize (0.007)}} \\

(1, 0, 0.2)
& I-CALM
& 0.180
& 0.143
& \mbox{0.226 {\scriptsize (0.009)}}
& \mbox{0.202 {\scriptsize (0.007)}} \\

(1, 0, 0.4)
& B
& 0.171
& 0.144
& \mbox{0.221 {\scriptsize (0.008)}}
& \mbox{0.206 {\scriptsize (0.007)}} \\

(1, 0, 0.4)
& I-CALM
& 0.189
& 0.156
& \mbox{0.231 {\scriptsize (0.009)}}
& \mbox{0.205 {\scriptsize (0.007)}} \\

(1, 0, 0.6)
& B
& 0.174
& 0.148
& \mbox{0.222 {\scriptsize (0.008)}}
& \mbox{0.202 {\scriptsize (0.007)}} \\

(1, 0, 0.6)
& I-CALM
& 0.169
& 0.120
& \mbox{0.218 {\scriptsize (0.009)}}
& \mbox{0.203 {\scriptsize (0.007)}} \\

(1, 0, 0.8)
& B
& 0.166
& 0.137
& \mbox{0.220 {\scriptsize (0.008)}}
& \mbox{0.207 {\scriptsize (0.007)}} \\

(1, 0, 0.8)
& I-CALM
& 0.195
& 0.146
& \mbox{0.234 {\scriptsize (0.009)}}
& \mbox{0.218 {\scriptsize (0.008)}} \\

\midrule

(1, -1, -)
& A
& 0.144
& 0.131
& \mbox{0.200 {\scriptsize (0.007)}}
& \mbox{0.193 {\scriptsize (0.007)}} \\

(1, -1, 0.2)
& B
& 0.154
& 0.122
& \mbox{0.210 {\scriptsize (0.008)}}
& \mbox{0.195 {\scriptsize (0.007)}} \\

(1, -1, 0.2)
& I-CALM
& 0.144
& 0.096
& \mbox{0.207 {\scriptsize (0.009)}}
& \mbox{0.193 {\scriptsize (0.007)}} \\

(1, -1, 0.4)
& B
& 0.149
& 0.121
& \mbox{0.210 {\scriptsize (0.008)}}
& \mbox{0.193 {\scriptsize (0.007)}} \\

(1, -1, 0.4)
& I-CALM
& 0.147
& 0.104
& \mbox{0.206 {\scriptsize (0.009)}}
& \mbox{0.186 {\scriptsize (0.007)}} \\

(1, -1, 0.6)
& B
& 0.156
& 0.131
& \mbox{0.216 {\scriptsize (0.008)}}
& \mbox{0.197 {\scriptsize (0.007)}} \\

(1, -1, 0.6)
& I-CALM
& 0.160
& 0.122
& \mbox{0.214 {\scriptsize (0.009)}}
& \mbox{0.192 {\scriptsize (0.007)}} \\

(1, -1, 0.8)
& B
& 0.153
& 0.120
& \mbox{0.214 {\scriptsize (0.008)}}
& \mbox{0.199 {\scriptsize (0.007)}} \\

(1, -1, 0.8)
& I-CALM
& 0.172
& 0.131
& \mbox{0.220 {\scriptsize (0.009)}}
& \mbox{0.202 {\scriptsize (0.007)}} \\

\bottomrule
\end{tabular}
}

\caption{
PopQA evaluation across the full set of reward configurations:
calibration metrics for GPT-5 mini across all tested Scheme A,
Scheme B, and I-CALM configurations.
}
\label{tab:gpt5_result2}
\end{table}

\begin{table}[H]
\centering
{\fontsize{9pt}{9pt}\selectfont
\setlength{\tabcolsep}{1.5pt}
\renewcommand{\arraystretch}{1.2}

\begin{tabular}{@{}clcccc@{}}
\toprule

\multirow{2}{*}{\textbf{Payoff}} &
\multirow{2}{*}{\textbf{Scheme}} &
\multicolumn{2}{c}{$\widehat{\textbf{ECE}}$} &
\multicolumn{2}{c}{\textbf{Brier Score ($\pm$ 95\% CI)}} \\

\cmidrule(lr){3-4}
\cmidrule(lr){5-6}

& &
\textbf{Answered} &
\textbf{Overall} &
\textbf{Answered} &
\textbf{Overall} \\

\midrule

(1, 0, -)
& A
& 0.429
& 0.386
& \mbox{0.422 {\scriptsize (0.010)}}
& \mbox{0.350 {\scriptsize (0.008)}} \\

(1, 0, 0.2)
& B
& 0.386
& 0.341
& \mbox{0.382 {\scriptsize (0.010)}}
& \mbox{0.304 {\scriptsize (0.008)}} \\

(1, 0, 0.2)
& I-CALM
& 0.373
& 0.321
& \mbox{0.371 {\scriptsize (0.011)}}
& \mbox{0.287 {\scriptsize (0.007)}} \\

(1, 0, 0.4)
& B
& 0.405
& 0.353
& \mbox{0.399 {\scriptsize (0.010)}}
& \mbox{0.317 {\scriptsize (0.008)}} \\

(1, 0, 0.4)
& I-CALM
& 0.359
& 0.330
& \mbox{0.360 {\scriptsize (0.011)}}
& \mbox{0.287 {\scriptsize (0.007)}} \\

(1, 0, 0.6)
& B
& 0.398
& 0.355
& \mbox{0.392 {\scriptsize (0.010)}}
& \mbox{0.314 {\scriptsize (0.008)}} \\

(1, 0, 0.6)
& I-CALM
& 0.379
& 0.357
& \mbox{0.375 {\scriptsize (0.011)}}
& \mbox{0.309 {\scriptsize (0.008)}} \\

(1, 0, 0.8)
& B
& 0.399
& 0.371
& \mbox{0.395 {\scriptsize (0.010)}}
& \mbox{0.324 {\scriptsize (0.008)}} \\

(1, 0, 0.8)
& I-CALM
& 0.383
& 0.370
& \mbox{0.380 {\scriptsize (0.011)}}
& \mbox{0.318 {\scriptsize (0.008)}} \\

\midrule

(1, -1, -)
& A
& 0.425
& 0.380
& \mbox{0.416 {\scriptsize (0.010)}}
& \mbox{0.341 {\scriptsize (0.008)}} \\

(1, -1, 0.2)
& B
& 0.399
& 0.341
& \mbox{0.394 {\scriptsize (0.010)}}
& \mbox{0.307 {\scriptsize (0.008)}} \\

(1, -1, 0.2)
& I-CALM
& 0.377
& 0.332
& \mbox{0.375 {\scriptsize (0.011)}}
& \mbox{0.291 {\scriptsize (0.008)}} \\

(1, -1, 0.4)
& B
& 0.396
& 0.353
& \mbox{0.391 {\scriptsize (0.010)}}
& \mbox{0.313 {\scriptsize (0.008)}} \\

(1, -1, 0.4)
& I-CALM
& 0.359
& 0.330
& \mbox{0.360 {\scriptsize (0.011)}}
& \mbox{0.286 {\scriptsize (0.007)}} \\

(1, -1, 0.6)
& B
& 0.391
& 0.351
& \mbox{0.387 {\scriptsize (0.010)}}
& \mbox{0.310 {\scriptsize (0.008)}} \\

(1, -1, 0.6)
& I-CALM
& 0.369
& 0.345
& \mbox{0.367 {\scriptsize (0.011)}}
& \mbox{0.297 {\scriptsize (0.008)}} \\

(1, -1, 0.8)
& B
& 0.398
& 0.373
& \mbox{0.395 {\scriptsize (0.010)}}
& \mbox{0.327 {\scriptsize (0.008)}} \\

(1, -1, 0.8)
& I-CALM
& 0.374
& 0.351
& \mbox{0.372 {\scriptsize (0.011)}}
& \mbox{0.304 {\scriptsize (0.008)}} \\

\bottomrule
\end{tabular}
}

\caption{
PopQA evaluation across the full set of reward configurations:
calibration metrics for GPT-4o mini across all tested Scheme A,
Scheme B, and I-CALM configurations.
}
\label{tab:gpt4o_results2}
\end{table}

Figure~\ref{fig:confidence_distribution_gpt5_gpt4} shows that the same selective-routing mechanism operates despite markedly different verbal-confidence distributions.
GPT-5 mini initially spreads surfaced answers across much of the confidence range; as the abstention reward and normative guidance are added, low- and medium-confidence incorrect mass increasingly moves into the best-guess channel, where it concentrates near $0.3$.
GPT-4o mini instead compresses most surfaced answers near $0.9$--$1.0$, including many incorrect ones, while its withheld best guesses concentrate primarily near $0.3$ and $0.5$.
Thus, I-CALM's improvements do not depend on a shared raw confidence scale: across both models, the prompt components increasingly separate error-prone candidates from the answers presented to the user, consistent with the uniform gains in $J_{\mathrm{abs}}$.

\begin{figure}[H]
    \centering

    \begin{subfigure}{\columnwidth}
        \centering
        \includegraphics[width=0.7\columnwidth]{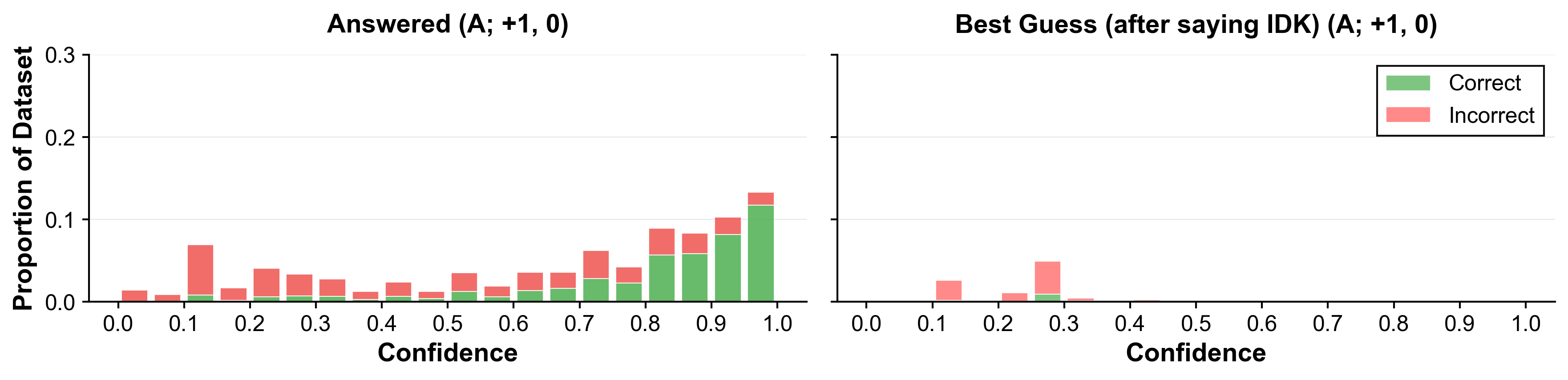}
        \includegraphics[width=0.7\columnwidth]{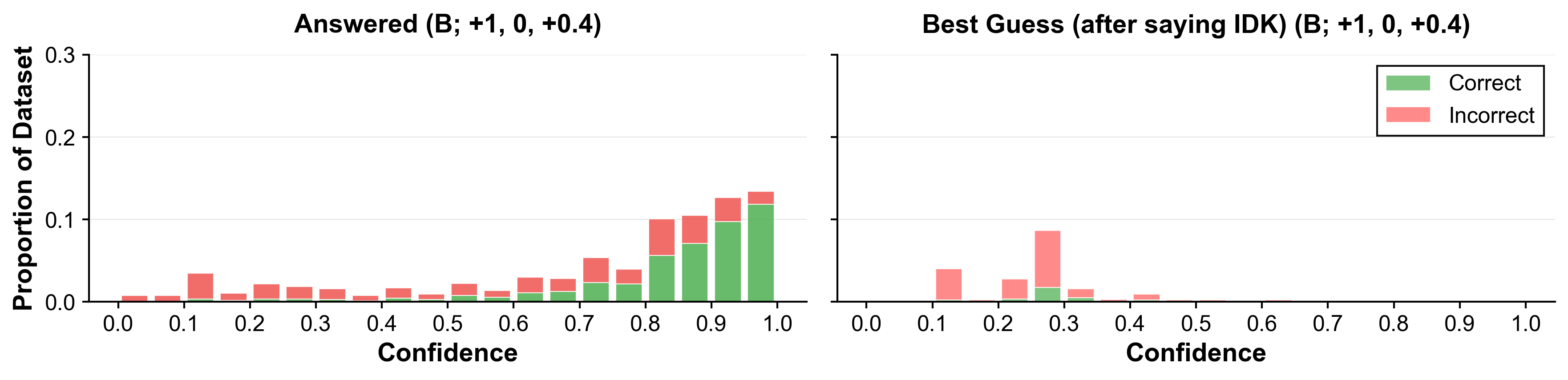}
        \includegraphics[width=0.7\columnwidth]{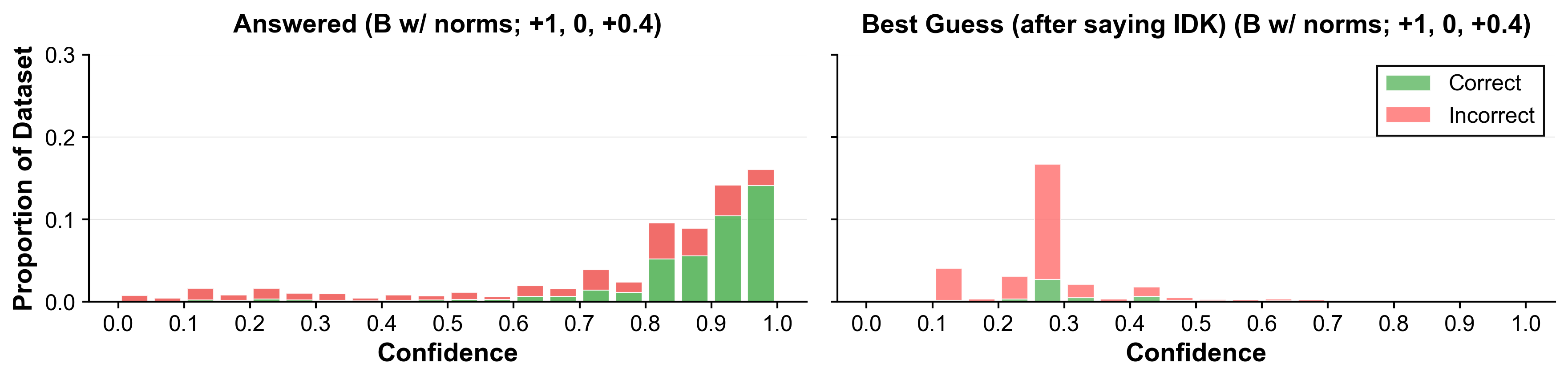}

        \caption{GPT-5 mini: incorrect answer penalty = 0}
        \label{fig:confidence_distribution_gpt5_1}
    \end{subfigure}

    
\end{figure}

\begin{figure}[H]
    \ContinuedFloat
    \centering

    \begin{subfigure}{\columnwidth}
        \centering
        \includegraphics[width=0.7\columnwidth]{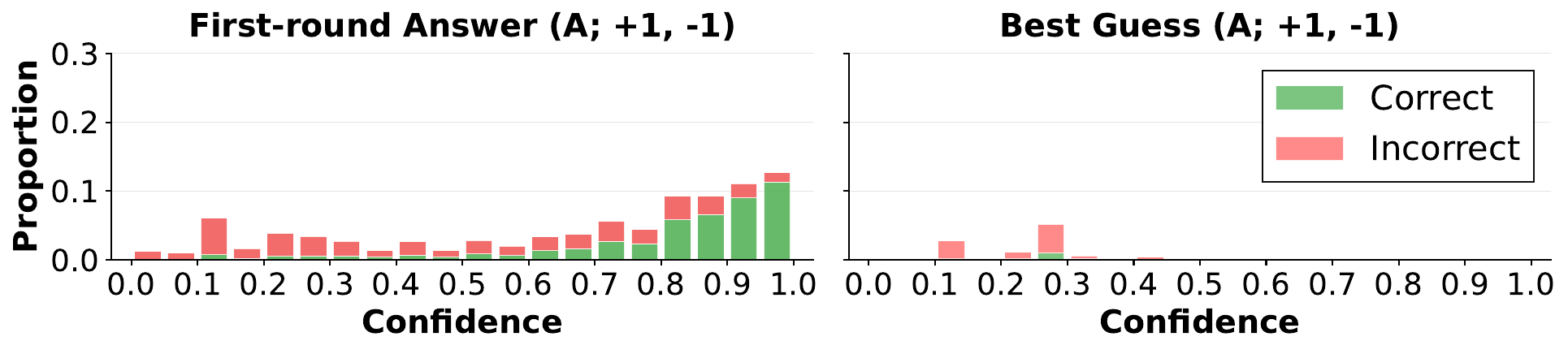}
        
        \includegraphics[width=0.7\columnwidth]{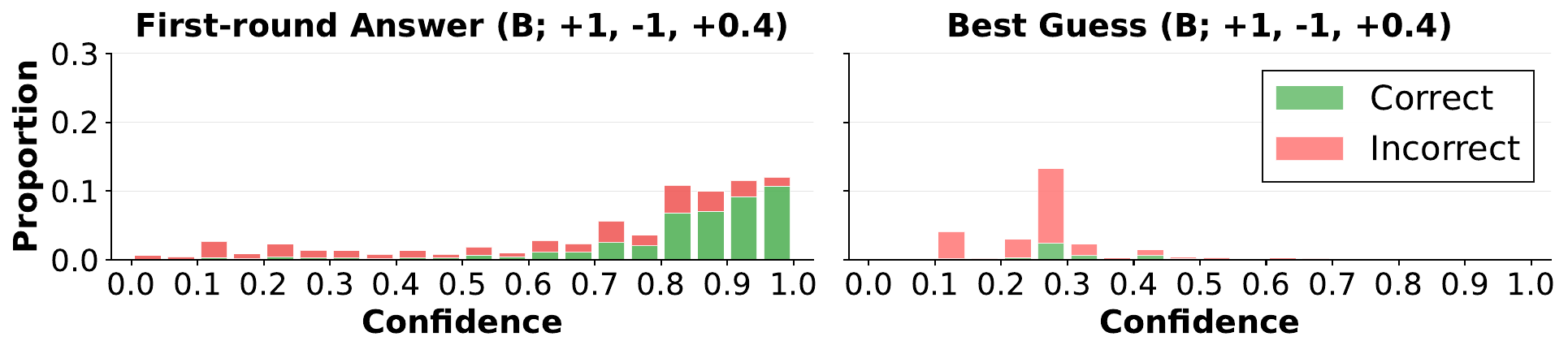}

        \includegraphics[width=0.7\columnwidth]{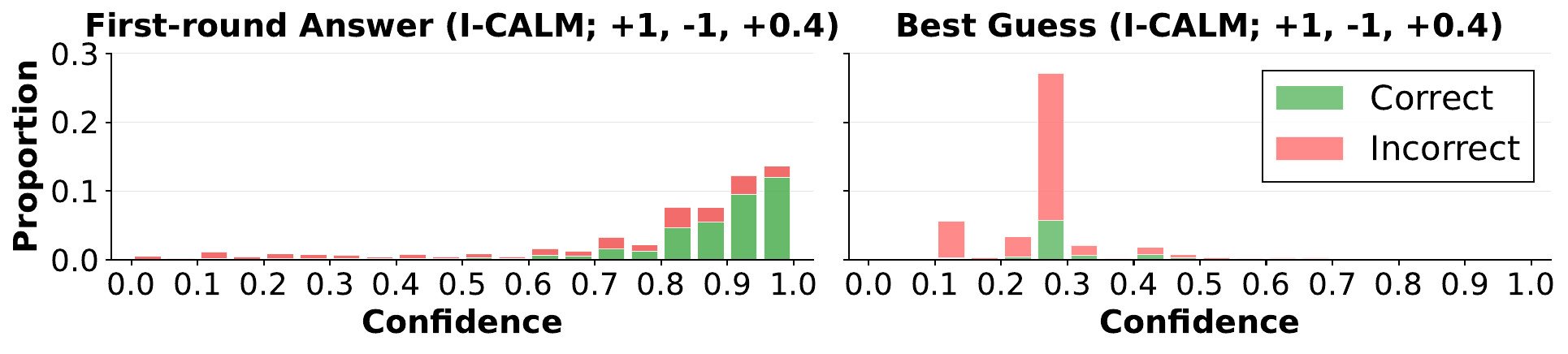}

        \caption{GPT-5 mini: incorrect answer penalty = -1}
        \label{fig:confidence_distribution_gpt5_2}
    \end{subfigure}

\end{figure}

\begin{figure}[H]
    \ContinuedFloat
    \centering

    \begin{subfigure}{\columnwidth}
        \centering
        \includegraphics[width=0.7\columnwidth]{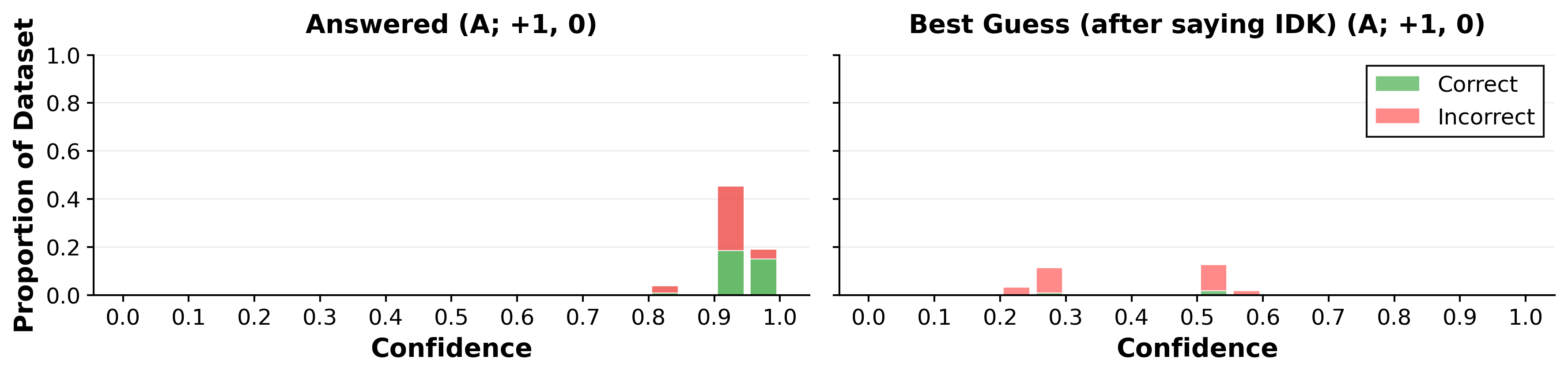}
        
        \includegraphics[width=0.7\columnwidth]{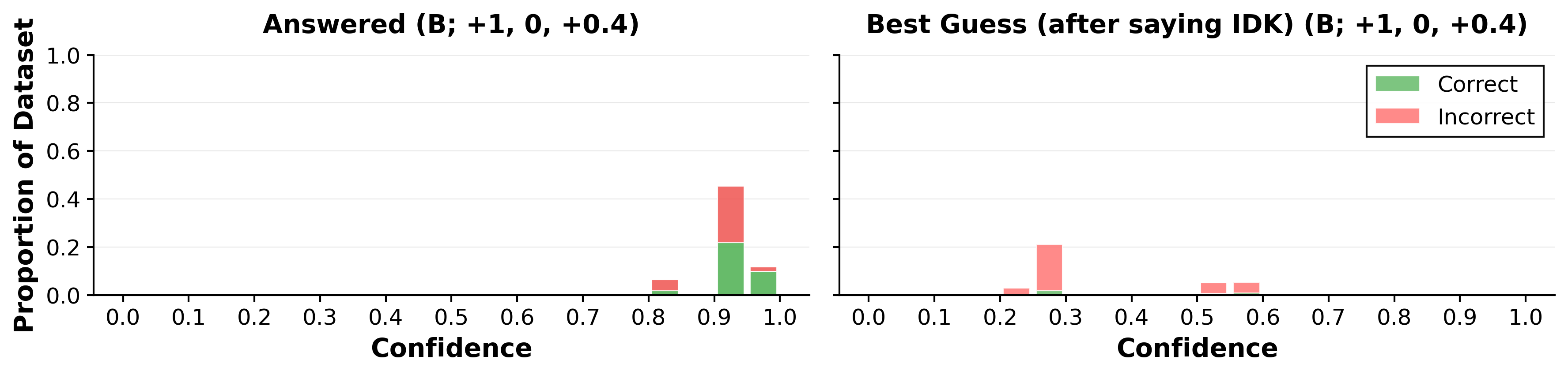}
        
        \includegraphics[width=0.7\columnwidth]{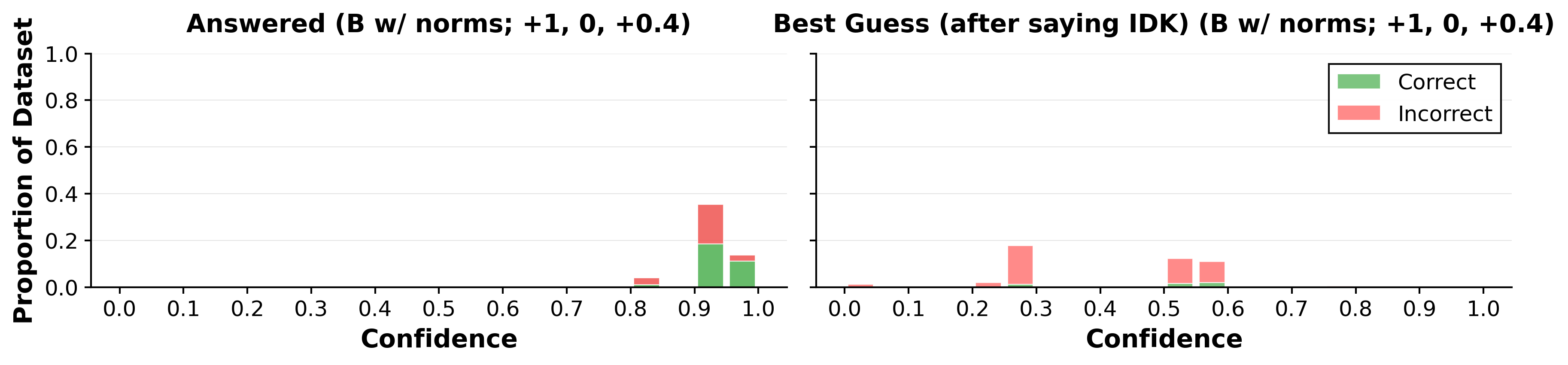}
        
        \caption{GPT-4o mini: incorrect answer penalty = 0}
        \label{fig:confidence_distribution_gpt4_1}
    \end{subfigure}

\end{figure}

\begin{figure}[H]
    \ContinuedFloat
    \centering

    \begin{subfigure}{\columnwidth}
        \centering
        \includegraphics[width=0.7\columnwidth]{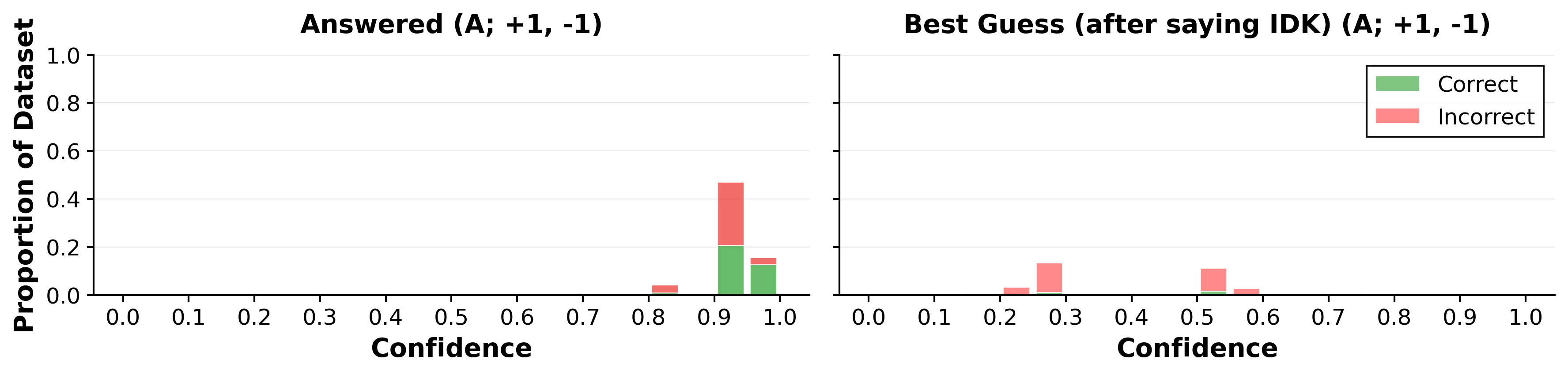}
        
        \includegraphics[width=0.7\columnwidth]{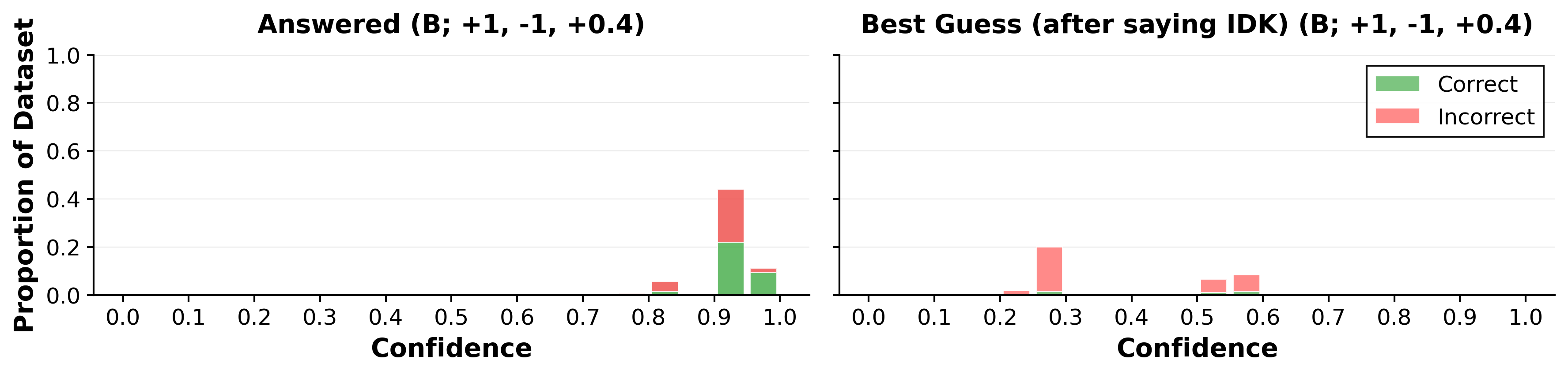}

        \includegraphics[width=0.7\columnwidth]{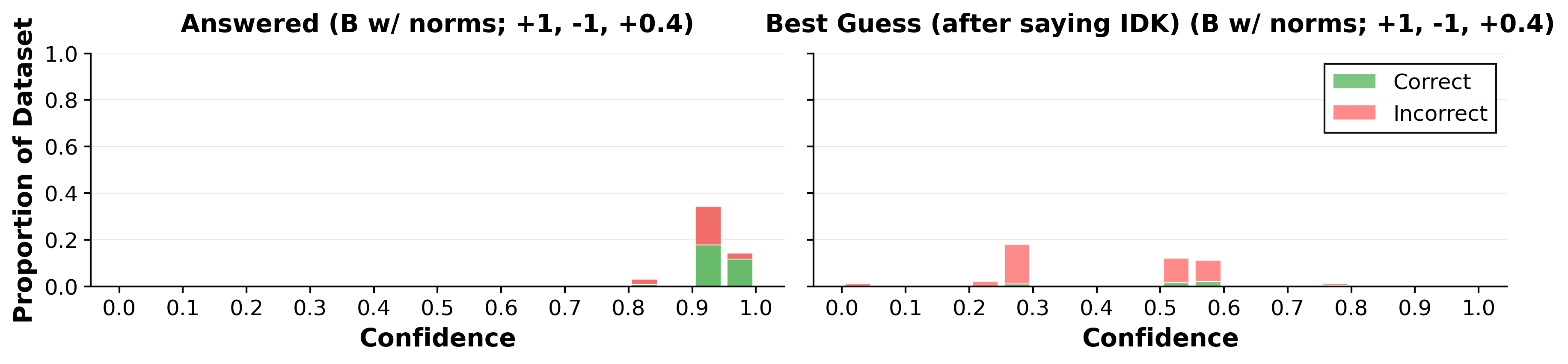}

        \caption{GPT-4o mini: incorrect answer penalty = -1}
        \label{fig:confidence_distribution_gpt4_2}
    \end{subfigure}

    \caption{Confidence distributions on PopQA under matched reward settings. For each model, rows correspond to Scheme A, Scheme B, and I-CALM (Scheme B with norms). Figures (a) and (c) use zero penalty for incorrect answers for GPT-5 mini and GPT-4o mini, respectively, while (b) and (d) use a $-1$ penalty. Within each small panel, the left histogram shows first-round answered cases and the right histogram shows second-round best guesses after an initial ``I don't know.''}
    \label{fig:confidence_distribution_gpt5_gpt4}
\end{figure}

\subsection{Robustness Check on Reward Scaling}
\label{subsec:_robust_reward_scaling}

We evaluate reward scaling by multiplying our reward scheme by factors of 10, 100, and 1000. Overall, scaling the magnitude of rewards has relatively little impact on performance, as shown in Table~\ref{tab:scaled_reward_results1} and Table~\ref{tab:scaled_reward_results2}. 
This indicates that model behavior is primarily driven by the relative reward structure rather than its absolute magnitude, and performance is reasonably robust under reward scaling.

\begin{table}[H]
\centering
{\fontsize{9pt}{9pt}\selectfont
\setlength{\tabcolsep}{2pt}
\renewcommand{\arraystretch}{1}

\begin{tabular}{@{}cccccc@{}}
\toprule

\textbf{Payoff} &
\textbf{Scheme} &
\textbf{FAR} &
\textbf{Coverage} &
\textbf{ECR} &
\textbf{CAR} \\

\midrule

(1, -1, -)
& A
& \mbox{0.486 {\scriptsize (0.009)}}
& 0.848
& 0.248
& 0.964 \\

(1, -1, 0.4)
& B
& \mbox{0.418 {\scriptsize (0.010)}}
& 0.689
& 0.481
& 0.900 \\

(1, -1, 0.4)
& I-CALM
& \mbox{0.355 {\scriptsize (0.011)}}
& 0.564
& 0.636
& 0.808 \\

\midrule

(10, -10, -)
& A
& \mbox{0.479 {\scriptsize (0.009)}}
& 0.831
& 0.271
& 0.954 \\

(10, -10, 4)
& B
& \mbox{0.391 {\scriptsize (0.010)}}
& 0.640
& 0.561
& 0.907 \\

(10, -10, 4)
& I-CALM
& \mbox{0.361 {\scriptsize (0.011)}}
& 0.549
& 0.648
& 0.803 \\

\midrule

(100, -100, -)
& A
& \mbox{0.492 {\scriptsize (0.009)}}
& 0.846
& 0.244
& 0.955 \\

(100, -100, 40)
& B
& \mbox{0.412 {\scriptsize (0.010)}}
& 0.678
& 0.510
& 0.927 \\

(100, -100, 40)
& I-CALM
& \mbox{0.346 {\scriptsize (0.011)}}
& 0.509
& 0.684
& 0.750 \\

\midrule

(1000, -1000, -)
& A
& \mbox{0.491 {\scriptsize (0.009)}}
& 0.859
& 0.232
& 0.970 \\

(1000, -1000, 400)
& B
& \mbox{0.407 {\scriptsize (0.010)}}
& 0.665
& 0.531
& 0.932 \\

(1000, -1000, 400)
& I-CALM
& \mbox{0.367 {\scriptsize (0.010)}}
& 0.583
& 0.610
& 0.817 \\

\bottomrule
\end{tabular}
}

\caption{
GPT-5 mini on PopQA
($N_{\mathrm{total}}=14{,}267$): performance metrics when the
entire reward scheme is scaled by factors of $10$, $100$, and $1000$.
FAR denotes $\mathrm{FAR}_{\mathrm{answered}}$. Parentheses report 95\%
CI half-widths.
}
\label{tab:scaled_reward_results1}
\end{table}

\begin{table}[H]
\centering
{\fontsize{9pt}{9pt}\selectfont
\setlength{\tabcolsep}{1pt}
\renewcommand{\arraystretch}{1}

\begin{tabular}{@{}cccccc@{}}
\toprule

\multirow{2}{*}{\textbf{Payoff}} &
\multirow{2}{*}{\textbf{Scheme}} &
\multicolumn{2}{c}{$\widehat{\textbf{ECE}}$} &
\multicolumn{2}{c}{\textbf{Brier Score}} \\

\cmidrule(lr){3-4}
\cmidrule(lr){5-6}

& &
\textbf{Answered} &
\textbf{Overall} &
\textbf{Answered} &
\textbf{Overall} \\

\midrule

(1, -1, -)
& A
& 0.144
& 0.131
& 0.200
& 0.193 \\

(1, -1, 0.4)
& B
& 0.149
& 0.121
& 0.210
& 0.193 \\

(1, -1, 0.4)
& I-CALM
& 0.147
& 0.104
& 0.206
& 0.186 \\

\midrule

(10, -10, -)
& A
& 0.177
& 0.158
& 0.220
& 0.210 \\

(10, -10, 4)
& B
& 0.151
& 0.109
& 0.210
& 0.198 \\

(10, -10, 4)
& I-CALM
& 0.164
& 0.115
& 0.217
& 0.191 \\

\midrule

(100, -100, -)
& A
& 0.185
& 0.164
& 0.221
& 0.211 \\

(100, -100, 40)
& B
& 0.159
& 0.117
& 0.215
& 0.204 \\

(100, -100, 40)
& I-CALM
& 0.175
& 0.106
& 0.219
& 0.194 \\

\midrule

(1000, -1000, -)
& A
& 0.166
& 0.148
& 0.212
& 0.206 \\

(1000, -1000, 400)
& B
& 0.149
& 0.109
& 0.211
& 0.201 \\

(1000, -1000, 400)
& I-CALM
& 0.166
& 0.116
& 0.215
& 0.192 \\

\bottomrule
\end{tabular}
}

\caption{
GPT-5 mini on PopQA: calibration metrics when the entire reward
scheme is scaled by factors of $10$, $100$, and $1000$.
}
\label{tab:scaled_reward_results2}
\end{table}

\subsection{Ablation on Norms}
\label{subsec:norm_ablation}

We analyze the single-norm and leave-one-out variants to understand how each principle contributes to selective answering, as shown in Table~\ref{tab:gpt5mini_norm_ablation_performance}. Relative to the no-norm baseline, every single norm lowers $\mathrm{FAR}_{\mathrm{answered}}$ and increases total reward, but all raise $\mathrm{FAR}_{\mathrm{overall}}$. Norm~3 performs best among the single norms on $\mathrm{FAR}_{\mathrm{answered}}$, reward, and Brier score, but also has the highest $\mathrm{FAR}_{\mathrm{overall}}$, indicating that no single norm provides the best overall balance.

\begin{table}[H]
\centering
{\fontsize{9pt}{9pt}\selectfont
\setlength{\tabcolsep}{3pt}
\renewcommand{\arraystretch}{1.15}

\begin{tabular}{@{}ccccc@{}}
\toprule

\multirow{2}{*}{\textbf{Kept Norms}}
& \multicolumn{2}{c}{\textbf{FAR}}
& \multirow{2}{*}{\shortstack{\textbf{Total}\\\textbf{Reward}}}
& \multirow{2}{*}{\shortstack{\textbf{Brier Score}\\\textbf{(Answered)}}} \\

\cmidrule(lr){2-3}

&
\textbf{Answered}
& \textbf{Overall}
&
&
\\

\midrule

Full-norm
& 0.355 {\scriptsize (0.011)}
& 0.550 {\scriptsize (0.008)}
& 4822.6
& 0.2061 {\scriptsize (0.0090)} \\

No-norm
& 0.418 {\scriptsize (0.010)}
& 0.555 {\scriptsize (0.008)}
& 3382.6
& 0.2102 {\scriptsize (0.0082)} \\

\cmidrule(lr){1-5}

\{1\}
& 0.379 {\scriptsize (0.010)}
& 0.562 {\scriptsize (0.008)}
& 4384.2
& 0.2256 {\scriptsize (0.0090)} \\

\{2\}
& 0.379 {\scriptsize (0.010)}
& 0.565 {\scriptsize (0.008)}
& 4368.4
& 0.2190 {\scriptsize (0.0089)} \\

\{3\}
& 0.365 {\scriptsize (0.011)}
& 0.614 {\scriptsize (0.008)}
& 4659
& 0.2069 {\scriptsize (0.0091)} \\

\{4\}
& 0.395 {\scriptsize (0.010)}
& 0.575 {\scriptsize (0.008)}
& 3992
& 0.2257 {\scriptsize (0.0087)} \\

\{5\}
& 0.383 {\scriptsize (0.010)}
& 0.562 {\scriptsize (0.008)}
& 4275.6
& 0.2207 {\scriptsize (0.0088)} \\

\cmidrule(lr){1-5}

\{2,3,4,5\}
& 0.381 {\scriptsize (0.010)}
& 0.553 {\scriptsize (0.008)}
& 4281.2
& 0.2095 {\scriptsize (0.0086)} \\

\{1,3,4,5\}
& 0.372 {\scriptsize (0.010)}
& 0.564 {\scriptsize (0.008)}
& 4490.2
& 0.2179 {\scriptsize (0.0089)} \\

\{1,2,4,5\}
& 0.375 {\scriptsize (0.010)}
& 0.558 {\scriptsize (0.008)}
& 4422.2
& 0.2210 {\scriptsize (0.0089)} \\

\{1,2,3,5\}
& 0.359 {\scriptsize (0.010)}
& 0.561 {\scriptsize (0.008)}
& 4741.4
& 0.2099 {\scriptsize (0.0089)} \\

\{1,2,3,4\}
& 0.361 {\scriptsize (0.010)}
& 0.571 {\scriptsize (0.008)}
& 4719
& 0.2124 {\scriptsize (0.0091)} \\

\bottomrule
\end{tabular}
}

\caption{Norm performance metrics for GPT-5 mini under I-CALM $(+1,-1,+0.4)$. The first two rows compare the full-norm and no-norm variants; subsequent rows report ablations over retained norm subsets. Brier stands for the Brier score for the first-round answered questions. Parentheses report 95\% CI half-widths.}
\label{tab:gpt5mini_norm_ablation_performance}
\end{table}

The full norm set is numerically best on all four metrics. In the leave-one-out ablations, removing Norm~1 most worsens $\mathrm{FAR}_{\mathrm{answered}}$ and reward, removing Norm~3 most worsens the Brier score, and removing Norm~5 most increases $\mathrm{FAR}_{\mathrm{overall}}$. Removing Norms~4--5 has smaller effects on $\mathrm{FAR}_{\mathrm{answered}}$ and reward than removing Norms~1--3, but still degrades performance.

Overall, the norm effects are structured rather than a generic consequence of adding more instructions: different principles contribute most strongly to different outcomes. The complete set combines evidence-sensitive constraints with broader communicative principles, producing the best overall balance.

\subsection{Rare vs. Common Facts}
\label{subsec:rare_vs_common_facts}

This experiment evaluates how LLMs perform under different reward schemes on common versus rare facts. We sort the PopQA dataset by entity popularity using the field $o_{\text{pop}}$, which records the monthly Wikipedia pageview count of the question's target entity. The top tercile is designated as common facts and the bottom tercile as rare facts; the middle tercile is omitted from this contrast analysis. We use $o_{\text{pop}}$ as a proxy for fact commonness.

For GPT-5 mini and GPT-4o mini, Table~\ref{tab:combined_performance} reports total reward, ECR, CAR, and FAR, while Table~\ref{tab:combined_calibration} reports calibration metrics. Across all schemes and both models, common facts consistently exhibit lower FAR and better calibration (lower $\widehat{\mathrm{ECE}}$ and Brier scores) than rare facts, indicating that LLMs are substantially more reliable on high-popularity knowledge.

Across both popularity strata and models, the progression Scheme~A $\rightarrow$ Scheme~B $\rightarrow$ I-CALM monotonically lowers $\mathrm{FAR}_{\mathrm{answered}}$ and raises ECR. CAR declines as the policy becomes more selective, but the ECR gains exceed the CAR losses, so the implied $J_{\mathrm{abs}}$ rises at every step. I-CALM captures 58.2--71.5\% of eventual errors while retaining 78.1--85.3\% of eventual correct candidates. Relative to Scheme~A, it reduces $\mathrm{FAR}_{\mathrm{answered}}$ by 0.111 and 0.119 on common and rare facts for GPT-5 mini, and by 0.053 and 0.078 for GPT-4o mini, with larger reductions on rare facts for both models. $\mathrm{FAR}_{\mathrm{overall}}$ remains nearly stable or increases slightly, indicating that the gain comes from more selective surfacing rather than higher candidate accuracy. Because Scheme~A uses a different reward function, total reward is cleanly comparable only between Scheme~B and I-CALM; under identical payoffs, I-CALM earns more in all four model--stratum combinations, especially on rare facts.

A closer comparison of ECR and CAR shows that both models respond more strongly on rare facts, but in different ways. For GPT-5 mini, Scheme~A $\rightarrow$ Scheme~B increases ECR more on rare than common facts ($+0.249$ versus $+0.227$) at nearly the same CAR cost ($-0.054$ versus $-0.056$); adding normative guidance further increases rare-fact ECR more strongly ($+0.173$ versus $+0.133$), but also causes a larger CAR reduction ($-0.130$ versus $-0.086$). Thus, I-CALM makes GPT-5 mini more aggressive on rare facts, capturing more errors while also withholding more correct candidates. GPT-4o mini instead exhibits better rare-fact abstention quality under every scheme: under I-CALM, for example, rare facts attain both higher ECR ($0.715$ versus $0.582$) and higher CAR ($0.853$ versus $0.806$), yielding $J_{\mathrm{abs}}=0.568$ versus $0.388$ on common facts. This suggests that GPT-4o mini's rare-fact errors may be more readily distinguishable from its correct candidates, whereas GPT-5 mini responds to rare facts primarily by adopting a more conservative surfacing policy. Because ECR and CAR confidence intervals are not reported, these cross-stratum differences should be interpreted as descriptive patterns.

Calibration effects are model-dependent. I-CALM attains the lowest answered and overall $\widehat{\mathrm{ECE}}$ and Brier scores for GPT-4o mini in both strata. For GPT-5 mini, it is best on overall $\widehat{\mathrm{ECE}}$ and Brier score for rare facts and answered Brier score for common facts; Scheme~B is best on common-fact $\widehat{\mathrm{ECE}}$ and overall Brier score, whereas Scheme~A is best on answered-only rare-fact calibration. Because the schemes surface different subsets, answered-only calibration should be interpreted alongside their selection behavior. Thus, selective-answering gains are consistent across popularity strata, while calibration gains are not uniform. Across models, GPT-5 mini has lower $\mathrm{FAR}_{\mathrm{overall}}$ and all calibration metrics in every matched scheme and stratum; $\mathrm{FAR}_{\mathrm{answered}}$ is also lower except under Scheme~A on rare facts, where the models are effectively tied ($0.652$ versus $0.651$).

\begin{table}[H]
\centering
{\fontsize{9pt}{9pt}\selectfont
\setlength{\tabcolsep}{1.5pt}
\renewcommand{\arraystretch}{1.1}

\begin{tabular}{@{}llccccc@{}}
\toprule

\multirow{2}{*}{\textbf{Scheme}} & 
\multirow{2}{*}{\textbf{Category}} &
\multirow{2}{*}{\shortstack{\textbf{Total}\\\textbf{Reward}}} &
\multirow{2}{*}{\textbf{ECR}} &
\multirow{2}{*}{\textbf{CAR}} &
\multicolumn{2}{c}{\textbf{FAR}} \\

\cmidrule(lr){6-7}

& & & & &
\textbf{Answered} & \textbf{Overall} \\

\midrule

\multicolumn{7}{@{}l@{}}{\textbf{GPT-5 mini}} \\
\addlinespace[2pt]

\multirow{2}{*}{\shortstack[l]{Scheme A}}
& Common & 687 & 0.267 & 0.959 & \mbox{0.341 {\scriptsize (0.015)}} & \mbox{0.403 {\scriptsize (0.014)}} \\
& Rare & -1995 & 0.222 & 0.965 & \mbox{0.652 {\scriptsize (0.015)}} & \mbox{0.699 {\scriptsize (0.013)}} \\

\cmidrule(lr){1-7}

\multirow{2}{*}{\shortstack[l]{Scheme B}}
& Common & 2083.8 & 0.494 & 0.903 & \mbox{0.274 {\scriptsize (0.015)}} & \mbox{0.403 {\scriptsize (0.014)}} \\
& Rare & 130.20 & 0.471 & 0.911 & \mbox{0.592 {\scriptsize (0.017)}} & \mbox{0.714 {\scriptsize (0.013)}} \\

\cmidrule(lr){1-7}

\multirow{2}{*}{\shortstack[l]{I-CALM}}
& Common & 2330.60 & 0.627 & 0.817 & \mbox{0.230 {\scriptsize (0.015)}} & \mbox{0.395 {\scriptsize (0.014)}} \\
& Rare & 845.40 & 0.644 & 0.781 & \mbox{0.533 {\scriptsize (0.021)}} & \mbox{0.714 {\scriptsize (0.013)}} \\

\midrule

\multicolumn{7}{@{}l@{}}{\textbf{GPT-4o mini}} \\
\addlinespace[2pt]

\multirow{2}{*}{\shortstack[l]{Scheme A}}
& Common & -301 & 0.408 & 0.868 & 0.368 {\scriptsize (0.016)} & 0.460 {\scriptsize (0.014)} \\
& Rare & -2827 & 0.520 & 0.954 & 0.651 {\scriptsize (0.018)} & 0.787 {\scriptsize (0.012)} \\

\cmidrule(lr){1-7}

\multirow{2}{*}{\shortstack[l]{Scheme B}}
& Common & 1569.20 & 0.474 & 0.849 & 0.350 {\scriptsize (0.016)} & 0.465 {\scriptsize (0.014)} \\
& Rare & 324 & 0.599 & 0.916 & 0.627 {\scriptsize (0.019)} & 0.794 {\scriptsize (0.011)} \\

\cmidrule(lr){1-7}

\multirow{2}{*}{\shortstack[l]{I-CALM}}
& Common & 1811.6 & 0.582 & 0.806 & 0.315 {\scriptsize (0.017)} & 0.470 {\scriptsize (0.014)} \\
& Rare & 867 & 0.715 & 0.853 & 0.573 {\scriptsize (0.022)} & 0.801 {\scriptsize (0.011)} \\

\bottomrule
\end{tabular}
}

\caption{PopQA ($\mathrm{N}_{\mathrm{total}}=14{,}267$) common-versus-rare fact split: performance metrics for GPT-5 mini and GPT-4o mini under Scheme A $(+1,-1)$, Scheme B $(+1,-1,+0.4)$, and I-CALM $(+1,-1,+0.4)$. Parentheses report 95\% CI half-widths.}
\label{tab:combined_performance}
\end{table}

\begin{table}[H]
\centering
{\fontsize{9pt}{9pt}\selectfont
\setlength{\tabcolsep}{1.5pt}
\renewcommand{\arraystretch}{1.1}

\begin{tabular}{@{}llcccc@{}}
\toprule

\multirow{2}{*}{\textbf{Scheme}} &
\multirow{2}{*}{\textbf{Category}} &
\multicolumn{2}{c}{$\widehat{\mathbf{ECE}}$} &
\multicolumn{2}{c}{\textbf{Brier Score}} \\

\cmidrule(lr){3-4}
\cmidrule(lr){5-6}

& &
\textbf{Answered} & \textbf{Overall} &
\textbf{Answered} & \textbf{Overall} \\

\midrule

\multicolumn{6}{@{}l@{}}{\textbf{GPT-5 mini}} \\
\addlinespace[2pt]

\multirow{2}{*}{\shortstack[l]{Scheme A}}
& Common & 0.081 & 0.073 & \mbox{0.172 {\scriptsize (0.012)}} & \mbox{0.174 {\scriptsize (0.011)}} \\
& Rare & 0.222 & 0.201 & \mbox{0.229 {\scriptsize (0.013)}} & \mbox{0.214 {\scriptsize (0.012)}} \\

\cmidrule(lr){1-6}

\multirow{2}{*}{\shortstack[l]{Scheme B}}
& Common & 0.073 & 0.057 & \mbox{0.163 {\scriptsize (0.012)}} & \mbox{0.170 {\scriptsize (0.011)}} \\
& Rare & 0.242 & 0.196 & \mbox{0.264 {\scriptsize (0.016)}} & \mbox{0.214 {\scriptsize (0.012)}} \\

\cmidrule(lr){1-6}

\multirow{2}{*}{\shortstack[l]{I-CALM}}
& Common & 0.076 & 0.071 & \mbox{0.150 {\scriptsize (0.013)}} & \mbox{0.173 {\scriptsize (0.011)}} \\
& Rare & 0.251 & 0.179 & \mbox{0.277 {\scriptsize (0.019)}} & \mbox{0.195 {\scriptsize (0.012)}} \\

\midrule

\multicolumn{6}{@{}l@{}}{\textbf{GPT-4o mini}} \\
\addlinespace[2pt]

\multirow{2}{*}{\shortstack[l]{Scheme A}}
& Common & 0.306 & 0.264 & \mbox{0.312 {\scriptsize (0.015)}} & \mbox{0.290 {\scriptsize (0.013)}} \\
& Rare & 0.575 & 0.479 & \mbox{0.553 {\scriptsize (0.019)}} & \mbox{0.387 {\scriptsize (0.014)}} \\

\cmidrule(lr){1-6}

\multirow{2}{*}{\shortstack[l]{Scheme B}}
& Common & 0.284 & 0.247 & \mbox{0.292 {\scriptsize (0.016)}} & \mbox{0.274 {\scriptsize (0.013)}} \\
& Rare & 0.5516 & 0.447 & \mbox{0.5311 {\scriptsize (0.020)}} & \mbox{0.348 {\scriptsize (0.014)}} \\

\cmidrule(lr){1-6}

\multirow{2}{*}{\shortstack[l]{I-CALM}}
& Common & 0.258 & 0.230 & \mbox{0.268 {\scriptsize (0.016)}} & \mbox{0.257 {\scriptsize (0.013)}} \\
& Rare & 0.5094 & 0.418 & \mbox{0.500 {\scriptsize (0.023)}} & \mbox{0.311 {\scriptsize (0.013)}} \\

\bottomrule
\end{tabular}
}

\caption{PopQA common-versus-rare fact split: calibration metrics for GPT-5 mini and GPT-4o mini under Scheme A $(+1,-1)$, Scheme B $(+1,-1,+0.4)$, and I-CALM $(+1,-1,+0.4)$. Parentheses report 95\% CI half-widths.}
\label{tab:combined_calibration}
\end{table}

\subsection{Additional Experiments}
\label{subsec:additional_exp}

We supplement the main PopQA analysis with transfer experiments across model families and dataset difficulty, adding Meta-Llama-3-8B-Instruct, Gemini-3.1-Flash-Lite, and Qwen3-4B-Instruct-2507 on PopQA and evaluating these models together with GPT-5 mini and GPT-4o mini on TriviaQA and SimpleQA-Verified. Across the 13 model--dataset settings reported below, ECR increases at every step from Scheme~A to Scheme~B to I-CALM. Because stronger abstention can also reduce CAR, $J_{\mathrm{abs}}$ provides the more informative summary: it increases monotonically and is highest under I-CALM in 12 of 13 settings, with GPT-4o mini on SimpleQA-Verified the sole exception. Relative to Scheme~A, I-CALM also lowers $\mathrm{FAR}_{\mathrm{answered}}$ in 12 settings and ties in the remaining one. The effect is strongest and most consistent on PopQA, becomes more dependent on baseline headroom on TriviaQA, and remains visible under the severe knowledge pressure of SimpleQA-Verified. Calibration effects are more model-dependent, but the selective-answering results provide broad transfer evidence that payoff and normative framing usually make abstention progressively better targeted.

\paragraph{PopQA}

On the 14,267-question PopQA test set, the main-text pattern transfers cleanly to all three additional model families. From Scheme~A to Scheme~B to I-CALM, $J_{\mathrm{abs}}$ increases from $0.082$ to $0.238$ to $0.490$ for Meta-Llama-3-8B-Instruct, from $0.216$ to $0.224$ to $0.380$ for Gemini-3.1-Flash-Lite, and from $0.170$ to $0.220$ to $0.285$ for Qwen3-4B-Instruct-2507, showing that the gains in error capture consistently outweigh the accompanying losses in correct-answer retention. I-CALM also produces the lowest $\mathrm{FAR}_{\mathrm{answered}}$ among all configurations, reducing it relative to Pure Eval from $0.683$ to $0.515$, $0.444$ to $0.322$, and $0.780$ to $0.705$, respectively. Although coverage declines, I-CALM still retains $80.5\%$, $87.9\%$, and $72.6\%$ of eventual correct candidates while intercepting $68.5\%$, $50.1\%$, and $55.9\%$ of eventual errors. Calibration generally improves as well: I-CALM attains the lowest answered and overall $\widehat{\mathrm{ECE}}$ and Brier scores for Llama and Gemini, and for Qwen it achieves the lowest overall $\widehat{\mathrm{ECE}}$ and both Brier scores, with answered $\widehat{\mathrm{ECE}}$ changing little. The confidence distributions in Figure~\ref{fig:popqa_conf_comparison_llama_gemini_qwen} visualize the same routing effect, with more lower-confidence, error-prone candidates moved from the surfaced-answer channel to the forced-best-guess channel.

\begin{table}[H]
\centering
{\fontsize{9pt}{9pt}\selectfont
\setlength{\tabcolsep}{1.5pt}

\begin{tabular}{clccccc}
\toprule

\textbf{Model} &
\textbf{Scheme} &
\textbf{C} &
\textbf{FAR} &
\textbf{ECR} &
\textbf{CAR} &
\textbf{$J_{\mathrm{abs}}$} \\

\midrule

\multirow{4}{*}{\begin{tabular}[c]{@{}c@{}}
Meta-Llama-3-\\
8B-Instruct
\end{tabular}}
& Pure Eval & 0.963 & \mbox{0.683 {\scriptsize (0.008)}} & -- & -- & -- \\
& A         & 0.931 & \mbox{0.691 {\scriptsize (0.008)}} & 0.093 & 0.989 & 0.082 \\
& B         & 0.775 & \mbox{0.650 {\scriptsize (0.009)}} & 0.294 & 0.944 & 0.238 \\
& I-CALM    & 0.447 & \mbox{0.515 {\scriptsize (0.012)}} & 0.685 & 0.805 & 0.49 \\

\midrule

\multirow{4}{*}{\begin{tabular}[c]{@{}c@{}}
Gemini-3.1-\\
Flash-Lite
\end{tabular}}
& Pure Eval & 0.999 & \mbox{0.444 {\scriptsize (0.008)}} & -- & -- & -- \\
& A         & 0.850 & \mbox{0.388 {\scriptsize (0.009)}} & 0.268 & 0.948 & 0.216 \\
& B         & 0.839 & \mbox{0.395 {\scriptsize (0.009)}} & 0.282 & 0.942 & 0.224 \\
& I-CALM    & 0.706 & \mbox{0.322 {\scriptsize (0.009)}} & 0.501 & 0.879 & 0.380 \\

\midrule

\multirow{4}{*}{\begin{tabular}[c]{@{}c@{}}
Qwen3-4B-\\
Instruct-2507
\end{tabular}}
& Pure Eval & 0.993 & \mbox{0.780 {\scriptsize (0.007)}} & -- & -- & -- \\
& A         & 0.641 & \mbox{0.754 {\scriptsize (0.009)}} & 0.393 & 0.777 & 0.17 \\
& B         & 0.570 & \mbox{0.736 {\scriptsize (0.010)}} & 0.474 & 0.746 & 0.22 \\
& I-CALM    & 0.499 & \mbox{0.705 {\scriptsize (0.011)}} & 0.559 & 0.726 & 0.285 \\

\bottomrule
\end{tabular}
}

\caption{PopQA results ($N_{\mathrm{total}}=14{,}267$): selective answering performance metrics for Meta-Llama-3-8B-Instruct, Gemini-3.1-Flash-Lite, and Qwen3-4B-Instruct-2507 under Pure Eval, Scheme A, Scheme B, and I-CALM. C stands for coverage; FAR stands for $\mathrm{FAR}_{\mathrm{answered}}$. Parentheses report 95\% CI half-widths.}
\label{tab:combined_performance_popqa_llama}
\end{table}

\begin{table}[H]
\centering
{\fontsize{9pt}{9pt}\selectfont
\setlength{\tabcolsep}{0.5pt}

\begin{tabular}{clcccc}
\toprule

\multirow{2}{*}{\textbf{Model}} &
\multirow{2}{*}{\textbf{Scheme}} &
\multicolumn{2}{c}{$\widehat{\textbf{ECE}}$} &
\multicolumn{2}{c}{\textbf{Brier Score}} \\

\cmidrule(lr){3-4}
\cmidrule(lr){5-6}

& &
\textbf{Answered} &
\textbf{Overall} &
\textbf{Answered} &
\textbf{Overall} \\

\midrule

\multirow{3}{*}{\begin{tabular}[c]{@{}c@{}}
Meta-Llama-3-\\
8B-Instruct
\end{tabular}}
& A      & 0.666 & 0.646 & \mbox{0.663 {\scriptsize (0.008)}} & \mbox{0.637 {\scriptsize (0.008)}} \\
& B      & 0.630 & 0.563 & \mbox{0.625 {\scriptsize (0.009)}} & \mbox{0.553 {\scriptsize (0.008)}} \\
& I-CALM & 0.497 & 0.446 & \mbox{0.496 {\scriptsize (0.012)}} & \mbox{0.423 {\scriptsize (0.009)}} \\

\midrule

\multirow{3}{*}{\begin{tabular}[c]{@{}c@{}}
Gemini-3.1-\\
Flash-Lite
\end{tabular}}
& A      & 0.370 & 0.367 & \mbox{0.367 {\scriptsize (0.009)}} & \mbox{0.362 {\scriptsize (0.008)}} \\
& B      & 0.374 & 0.359 & \mbox{0.371 {\scriptsize (0.009)}} & \mbox{0.354 {\scriptsize (0.008)}} \\
& I-CALM & 0.306 & 0.331 & \mbox{0.308 {\scriptsize (0.009)}} & \mbox{0.326 {\scriptsize (0.008)}} \\

\midrule

\multirow{3}{*}{\begin{tabular}[c]{@{}c@{}}
Qwen3-4B-\\
Instruct-2507
\end{tabular}}
& A      & 0.591 & 0.468 & \mbox{0.623 {\scriptsize (0.010)}} & \mbox{0.448 {\scriptsize (0.008)}} \\
& B      & 0.596 & 0.471 & \mbox{0.605 {\scriptsize (0.011)}} & \mbox{0.421 {\scriptsize (0.008)}} \\
& I-CALM & 0.599 & 0.444 & \mbox{0.592 {\scriptsize (0.012)}} & \mbox{0.381 {\scriptsize (0.008)}} \\

\bottomrule
\end{tabular}
}

\caption{PopQA results: calibration metrics for Meta-Llama-3-8B-Instruct, Gemini-3.1-Flash-Lite, and Qwen3-4B-Instruct-2507 under Scheme A $(+1,-1)$, Scheme B $(+1,-1,+0.4)$, and I-CALM $(+1,-1,+0.4)$. Parentheses report 95\% CI half-widths.}
\label{tab:combined_calibration_popqa_llama}
\end{table}

\begin{figure}[H]
    \centering
    \begin{subfigure}{\columnwidth}
        \centering
        \includegraphics[width=0.7\columnwidth]{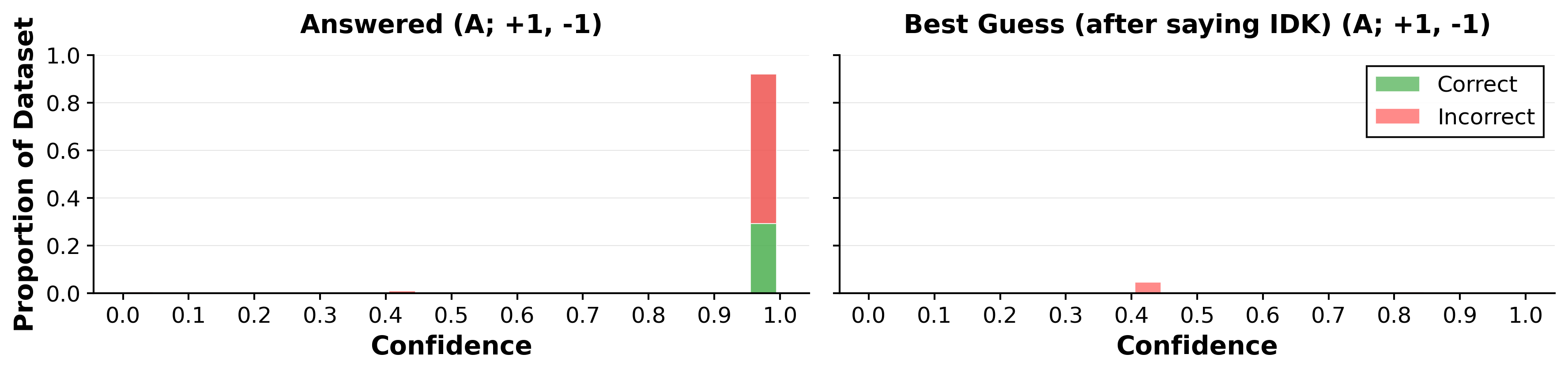}
        \includegraphics[width=0.7\columnwidth]{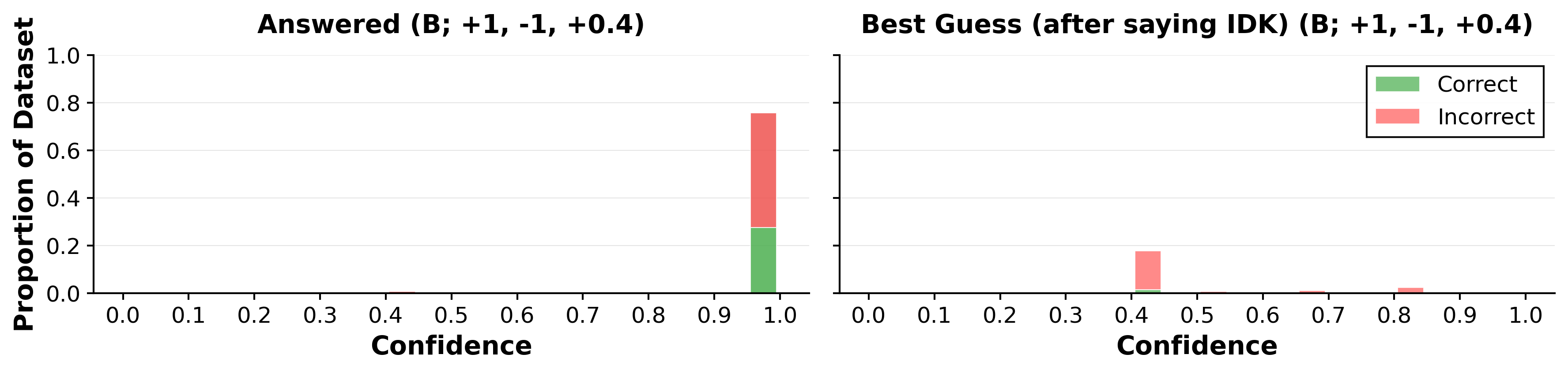}
        \includegraphics[width=0.7\columnwidth]{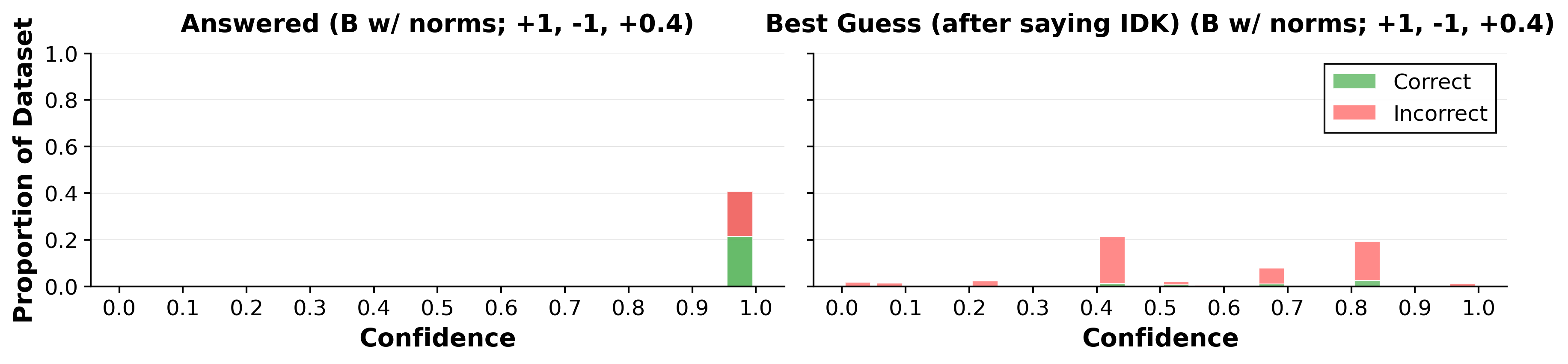}
        \caption{Meta-Llama-3-8B-Instruct}
    \end{subfigure}
\end{figure}

\begin{figure}[H]
    \ContinuedFloat
    \centering
    \begin{subfigure}{\columnwidth}
        \centering
        \includegraphics[width=0.7\columnwidth]{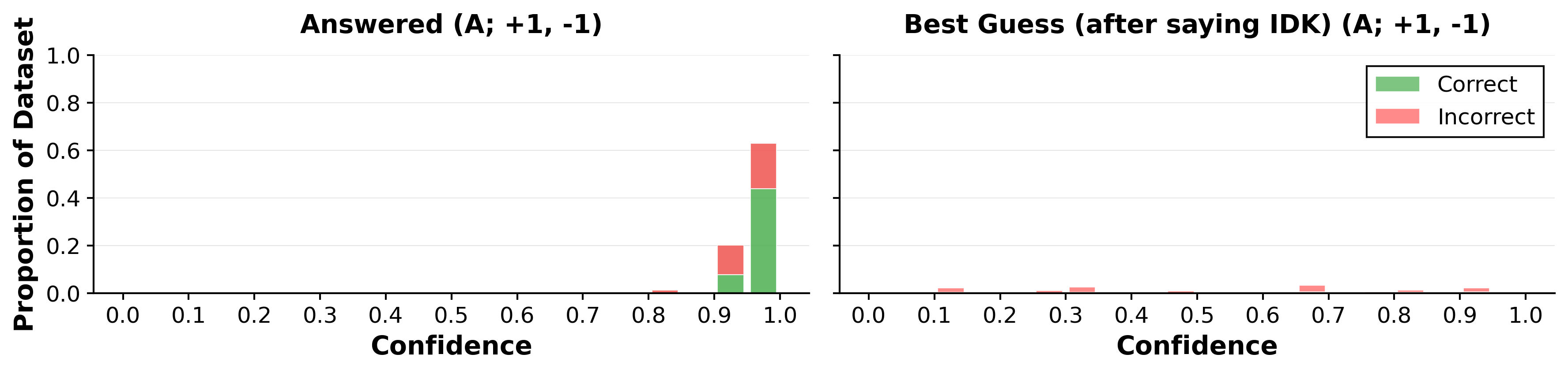}
        \includegraphics[width=0.7\columnwidth]{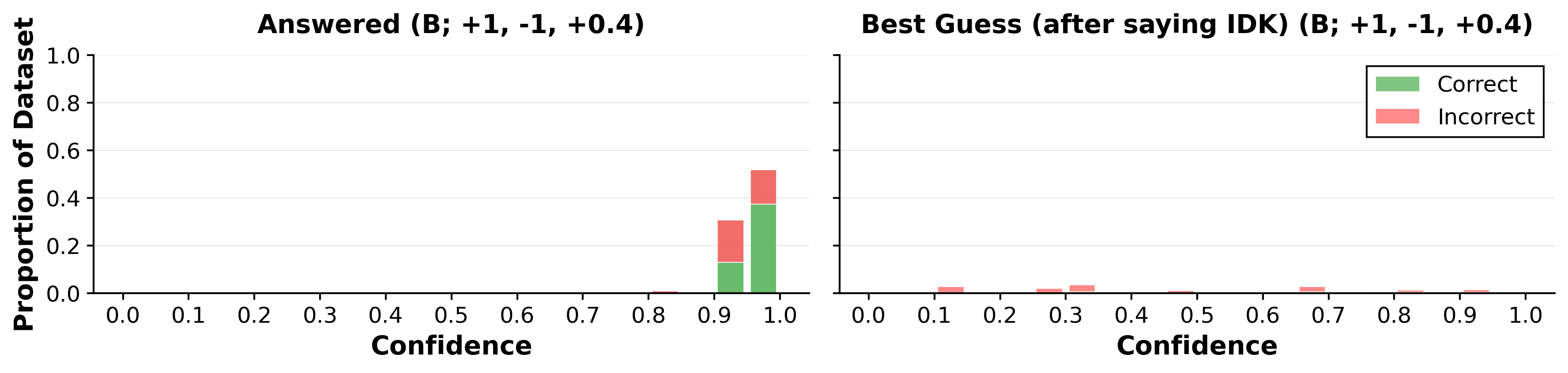}
        \includegraphics[width=0.7\columnwidth]{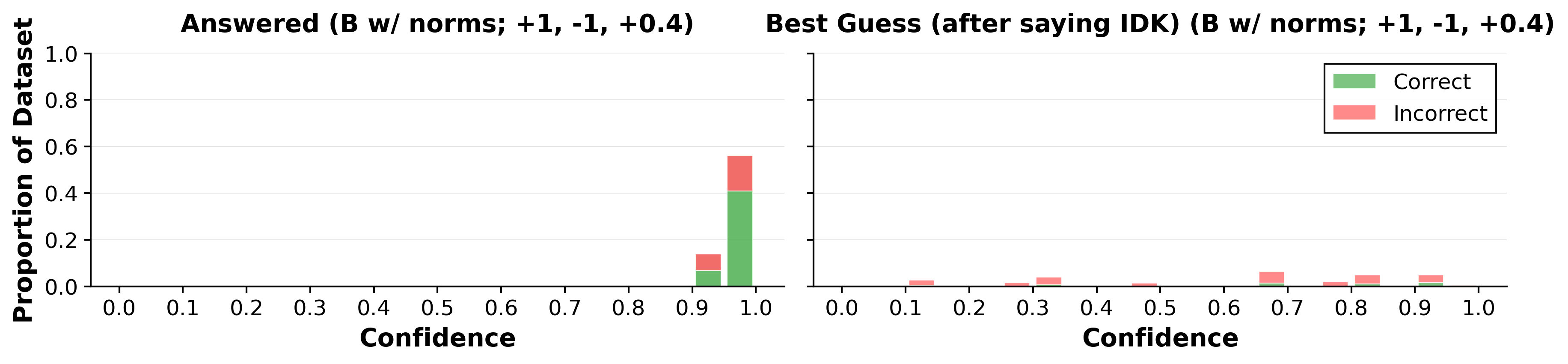}
        \caption{Gemini-3.1-Flash-Lite}
    \end{subfigure}
\end{figure}

\begin{figure}[H]
    \ContinuedFloat
    \centering
    \begin{subfigure}{\columnwidth}
        \centering
        \includegraphics[width=0.7\columnwidth]{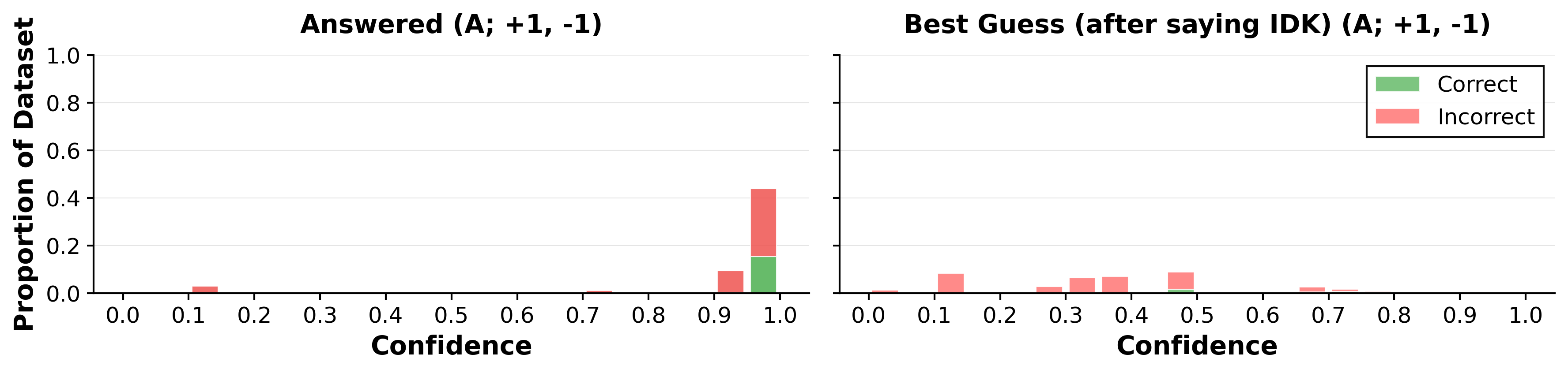}
        \includegraphics[width=0.7\columnwidth]{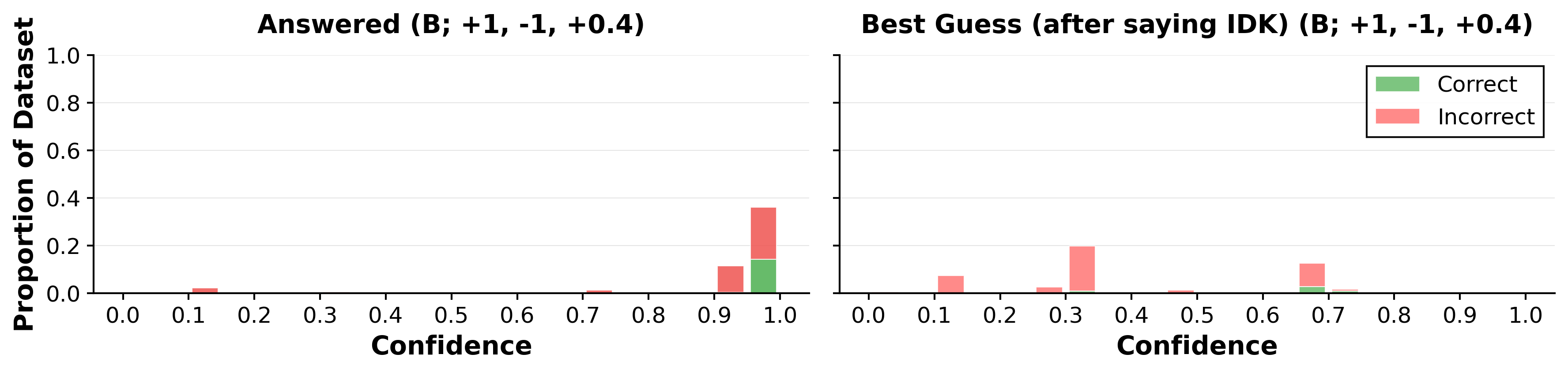}
        \includegraphics[width=0.7\columnwidth]{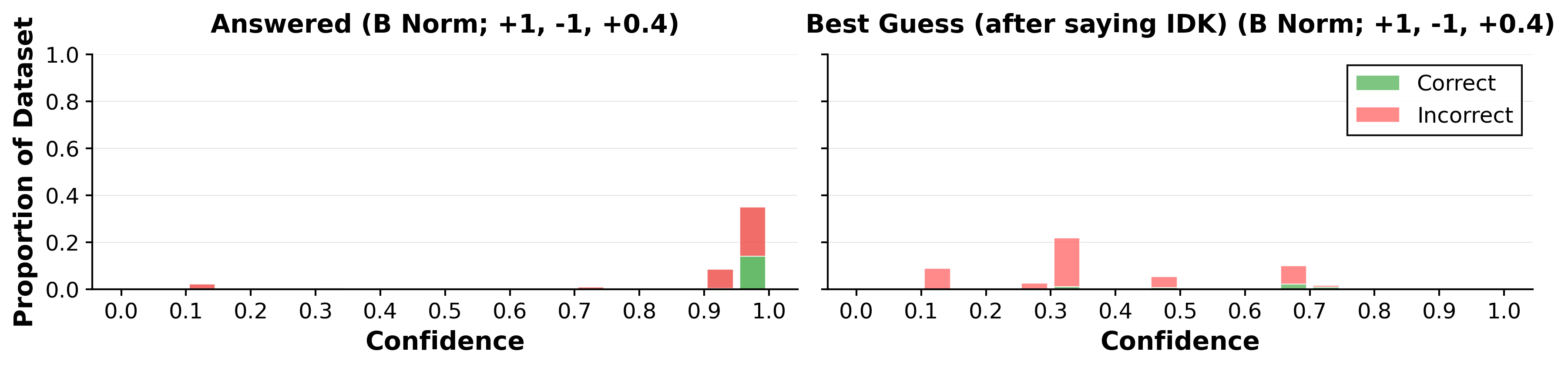}
        \caption{Qwen3-4B-Instruct-2507}
    \end{subfigure}
    
    \caption{PopQA confidence distributions under Scheme A (+1,-1), Scheme B (+1,-1,+0.4), and I-CALM (+1,-1,+0.4) (Scheme B with norms is equivalent to I-CALM.). Within each panel, the left histogram shows first-round answered cases and the right histogram shows second-round best guesses after an initial ``I don't know.''} 
    \label{fig:popqa_conf_comparison_llama_gemini_qwen}
\end{figure}

\paragraph{TriviaQA}

Because the TriviaQA test labels are unavailable, we evaluate on its 17,944-question validation set. Transfer is more dependent on baseline headroom here: GPT-5 mini, GPT-4o mini, and Gemini-3.1-Flash-Lite already have low Pure Eval $\mathrm{FAR}_{\mathrm{answered}}$ values of $0.145$, $0.148$, and $0.078$, leaving limited room for further risk reduction. Nevertheless, $J_{\mathrm{abs}}$ increases monotonically from Scheme~A to Scheme~B to I-CALM for all five models, reaching $0.167$ for GPT-5 mini, $0.108$ for GPT-4o mini, $0.104$ for Llama, $0.067$ for Gemini, and $0.181$ for Qwen. Under I-CALM, coverage remains between $0.835$ and $0.981$, while CAR remains between $0.928$ and $0.988$, indicating that the additional error capture is obtained with limited loss of correct answers. The largest surfaced-risk reductions occur for models with greater baseline error: I-CALM lowers $\mathrm{FAR}_{\mathrm{answered}}$ from $0.434$ to $0.280$ for Llama and from $0.593$ to $0.463$ for Qwen, whereas GPT-5 mini improves modestly and GPT-4o mini and Gemini remain essentially unchanged. Calibration gains are clearest for Qwen and smaller for Llama and Gemini, while GPT-5 mini and GPT-4o mini show mixed changes from already strong or stable baselines. Consistent with these results, Figure~\ref{fig:trivia_conf_comparison_all} shows little low-confidence mass to reroute for the strongest baseline models, but a clearer movement of risky candidates into the best-guess channel for Qwen and, to a lesser extent, Llama. TriviaQA therefore supports the method while showing that its largest risk reductions arise when the baseline leaves sufficient error headroom.

\begin{table}[H]
\centering
{\fontsize{9pt}{9pt}\selectfont
\setlength{\tabcolsep}{1.5pt}

\begin{tabular}{clccccc}
\toprule

\textbf{Model} &
\textbf{Scheme} &
\textbf{C} &
\textbf{FAR} &
\textbf{ECR} &
\textbf{CAR} &
\textbf{$J_{\mathrm{abs}}$} \\

\midrule

\multirow{4}{*}{GPT-5 mini}
& Pure Eval & 0.998 & \mbox{0.145 {\scriptsize (0.005)}} & -- & -- & -- \\
& A         & 0.982 & \mbox{0.140 {\scriptsize (0.005)}} & 0.087 & 0.994 & 0.081 \\
& B         & 0.963 & \mbox{0.139 {\scriptsize (0.005)}} & 0.157 & 0.986 & 0.143 \\
& I-CALM    & 0.950 & \mbox{0.136 {\scriptsize (0.005)}} & 0.191 & 0.976 & 0.167 \\

\midrule

\multirow{4}{*}{GPT-4o mini}
& Pure Eval & 1.000 & \mbox{0.148 {\scriptsize (0.005)}} & -- & -- & -- \\
& A         & 0.983 & \mbox{0.156 {\scriptsize (0.005)}} & 0.066 & 0.994 & 0.060 \\
& B         & 0.977 & \mbox{0.155 {\scriptsize (0.005)}} & 0.090 & 0.991 & 0.081 \\
& I-CALM    & 0.970 & \mbox{0.156 {\scriptsize (0.006)}} & 0.120 & 0.988 & 0.108 \\

\midrule

\multirow{4}{*}{\begin{tabular}[c]{@{}c@{}}
Meta-Llama-3-\\
8B-Instruct
\end{tabular}}
& Pure Eval & 1.000 & \mbox{0.434 {\scriptsize (0.007)}} & -- & -- & -- \\
& A         & 0.995 & \mbox{0.287 {\scriptsize (0.007)}} & 0.013 & 0.999 & 0.012 \\
& B         & 0.977 & \mbox{0.284 {\scriptsize (0.007)}} & 0.060 & 0.992 & 0.052 \\
& I-CALM    & 0.952 & \mbox{0.280 {\scriptsize (0.007)}} & 0.121 & 0.983 & 0.104 \\

\midrule

\multirow{4}{*}{\begin{tabular}[c]{@{}c@{}}
Gemini-3.1-\\
Flash-Lite
\end{tabular}}
& Pure Eval & 1.000 & \mbox{0.078 {\scriptsize (0.004)}} & -- & -- & -- \\
& A         & 0.988 & \mbox{0.082 {\scriptsize (0.004)}} & 0.050 & 0.992 & 0.042 \\
& B         & 0.989 & \mbox{0.083 {\scriptsize (0.004)}} & 0.053 & 0.993 & 0.046 \\
& I-CALM    & 0.981 & \mbox{0.080 {\scriptsize (0.004)}} & 0.080 & 0.987 & 0.067 \\

\midrule

\multirow{4}{*}{\begin{tabular}[c]{@{}c@{}}
Qwen3-4B-\\
Instruct-2507
\end{tabular}}
& Pure Eval & 0.999 & \mbox{0.593 {\scriptsize (0.007)}} & -- & -- & -- \\
& A         & 0.893 & \mbox{0.488 {\scriptsize (0.008)}} & 0.158 & 0.948 & 0.106 \\
& B         & 0.877 & \mbox{0.481 {\scriptsize (0.008)}} & 0.186 & 0.946 & 0.132 \\
& I-CALM    & 0.835 & \mbox{0.463 {\scriptsize (0.008)}} & 0.253 & 0.928 & 0.181 \\

\bottomrule
\end{tabular}
}

\caption{TriviaQA results ($N_{\mathrm{total}}=14{,}267$): selective answering performance metrics for GPT-5 mini, GPT-4o mini, Meta-Llama-3-8B-Instruct, Gemini-3.1-Flash-Lite, Qwen3-4B-Instruct-2507 under Pure Eval, Scheme A, Scheme B, and I-CALM. C stands for coverage; FAR stands for $\mathrm{FAR}_{\mathrm{answered}}$. Parentheses report 95\% CI half-widths.}
\label{tab:combined_performance_triviaqa}
\end{table}

\begin{table}[H]
\centering
{\fontsize{9pt}{9pt}\selectfont
\setlength{\tabcolsep}{0.5pt}

\begin{tabular}{clcccc}
\toprule

\multirow{2}{*}{\textbf{Model}} &
\multirow{2}{*}{\textbf{Scheme}} &
\multicolumn{2}{c}{$\widehat{\textbf{ECE}}$} &
\multicolumn{2}{c}{\textbf{Brier Score}} \\

\cmidrule(lr){3-4}
\cmidrule(lr){5-6}

& &
\textbf{Answered} &
\textbf{Overall} &
\textbf{Answered} &
\textbf{Overall} \\

\midrule

\multirow{3}{*}{GPT-5 mini}
& A      & 0.026 & 0.026 & \mbox{0.086 {\scriptsize (0.004)}} & \mbox{0.088 {\scriptsize (0.004)}} \\
& B      & 0.037 & 0.034 & \mbox{0.087 {\scriptsize (0.004)}} & \mbox{0.090 {\scriptsize (0.004)}} \\
& I-CALM & 0.041 & 0.040 & \mbox{0.087 {\scriptsize (0.004)}} & \mbox{0.092 {\scriptsize (0.004)}} \\

\midrule

\multirow{3}{*}{GPT-4o mini}
& A      & 0.118 & 0.118 & \mbox{0.132 {\scriptsize (0.005)}} & \mbox{0.131 {\scriptsize (0.005)}} \\
& B      & 0.112 & 0.113 & \mbox{0.128 {\scriptsize (0.005)}} & \mbox{0.130 {\scriptsize (0.005)}} \\
& I-CALM & 0.121 & 0.122 & \mbox{0.132 {\scriptsize (0.005)}} & \mbox{0.132 {\scriptsize (0.005)}} \\

\midrule

\multirow{3}{*}{\begin{tabular}[c]{@{}c@{}}
Meta-Llama-3-\\
8B-Instruct
\end{tabular}}
& A      & 0.270 & 0.270 & \mbox{0.276 {\scriptsize (0.007)}} & \mbox{0.276 {\scriptsize (0.007)}} \\
& B      & 0.264 & 0.264 & \mbox{0.271 {\scriptsize (0.007)}} & \mbox{0.271 {\scriptsize (0.007)}} \\
& I-CALM & 0.263 & 0.266 & \mbox{0.269 {\scriptsize (0.007)}} & \mbox{0.271 {\scriptsize (0.007)}} \\

\midrule

\multirow{3}{*}{\begin{tabular}[c]{@{}c@{}}
Gemini-3.1-\\
Flash-Lite
\end{tabular}}
& A      & 0.079 & 0.079 & \mbox{0.079 {\scriptsize (0.004)}} & \mbox{0.080 {\scriptsize (0.004)}} \\
& B      & 0.078 & 0.078 & \mbox{0.079 {\scriptsize (0.004)}} & \mbox{0.080 {\scriptsize (0.004)}} \\
& I-CALM & 0.076 & 0.077 & \mbox{0.077 {\scriptsize (0.004)}} & \mbox{0.079 {\scriptsize (0.004)}} \\

\midrule

\multirow{3}{*}{\begin{tabular}[c]{@{}c@{}}
Qwen3-4B-\\
Instruct-2507
\end{tabular}}
& A      & 0.395 & 0.374 & \mbox{0.397 {\scriptsize (0.008)}} & \mbox{0.373 {\scriptsize (0.007)}} \\
& B      & 0.394 & 0.374 & \mbox{0.393 {\scriptsize (0.008)}} & \mbox{0.368 {\scriptsize (0.007)}} \\
& I-CALM & 0.374 & 0.348 & \mbox{0.376 {\scriptsize (0.008)}} & \mbox{0.344 {\scriptsize (0.007)}} \\

\bottomrule
\end{tabular}
}

\caption{TriviaQA results: calibration metrics for GPT-5 mini, GPT-4o mini, Meta-Llama-3-8B-Instruct, Gemini-3.1-Flash-Lite, and Qwen3-4B-Instruct-2507 under Scheme A $(+1,-1)$, Scheme B $(+1,-1,+0.4)$, and I-CALM $(+1,-1,+0.4)$. Parentheses report 95\% CI half-widths.}
\label{tab:combined_calibration_triviaqa}
\end{table}

\begin{figure}[H]
    \centering
    \begin{subfigure}{\columnwidth}
        \centering
        \includegraphics[width=0.7\columnwidth]{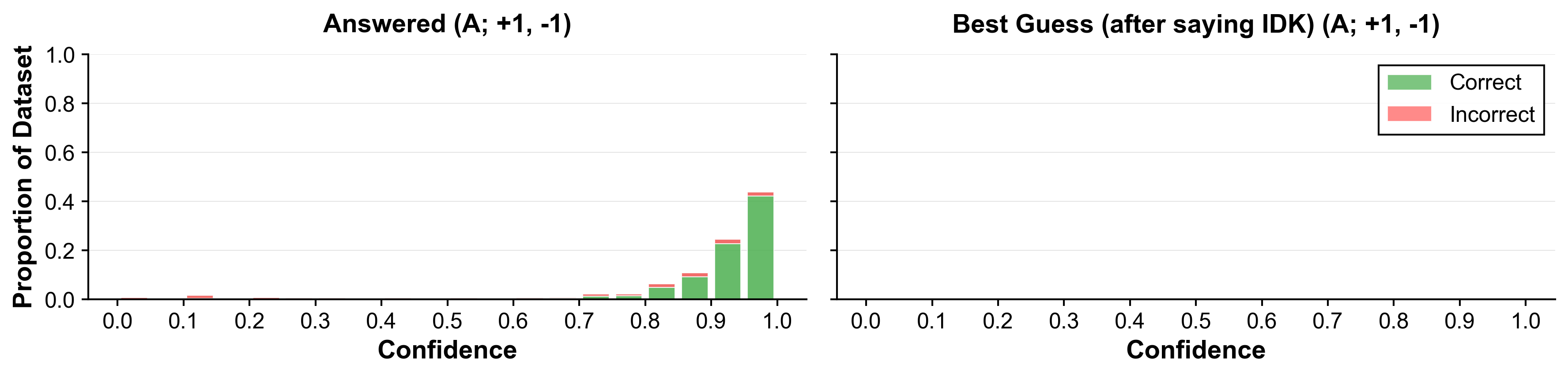}
        \includegraphics[width=0.7\columnwidth]{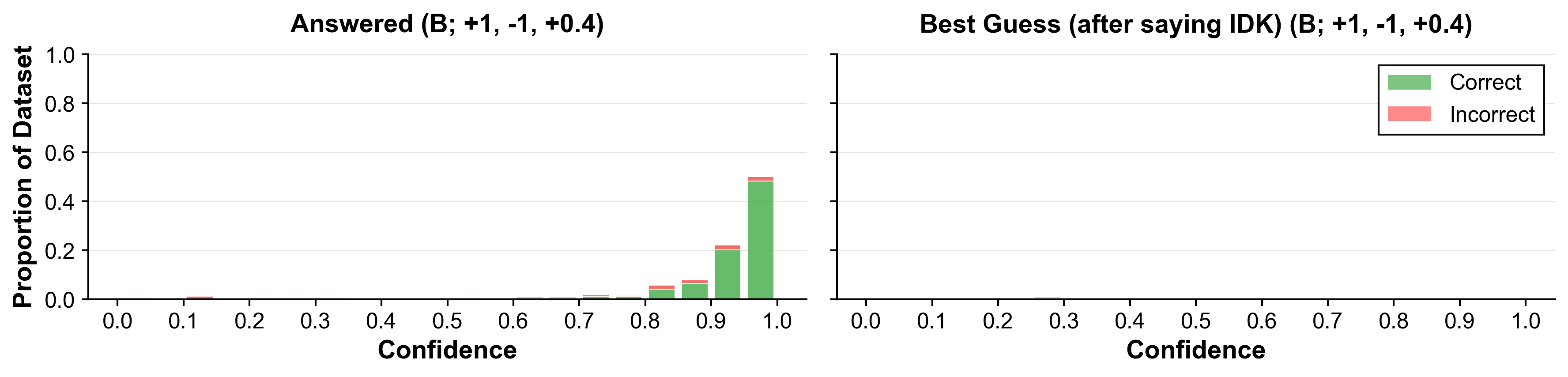}
        \includegraphics[width=0.7\columnwidth]{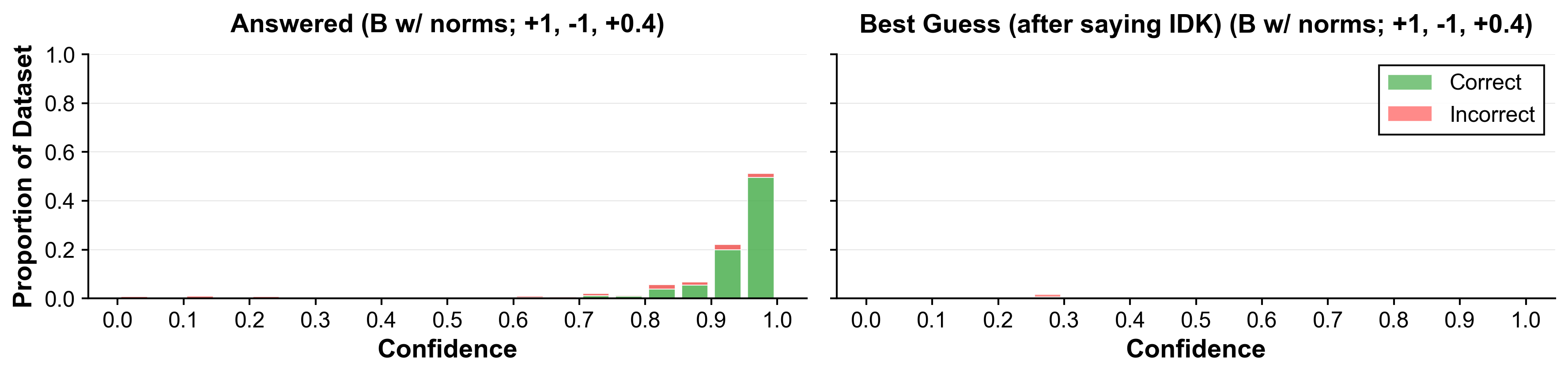}
        \caption{GPT-5 mini}
    \end{subfigure}
\end{figure}

\begin{figure}[H]
    \ContinuedFloat
    \centering
    
    \begin{subfigure}{\columnwidth}
        \centering
        \includegraphics[width=0.7\columnwidth]{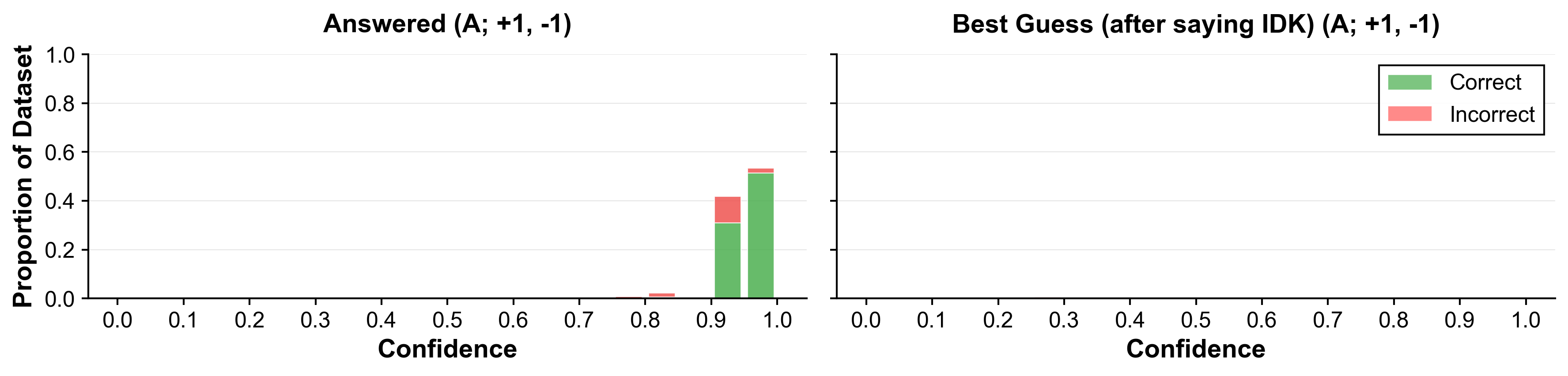}
        \includegraphics[width=0.7\columnwidth]{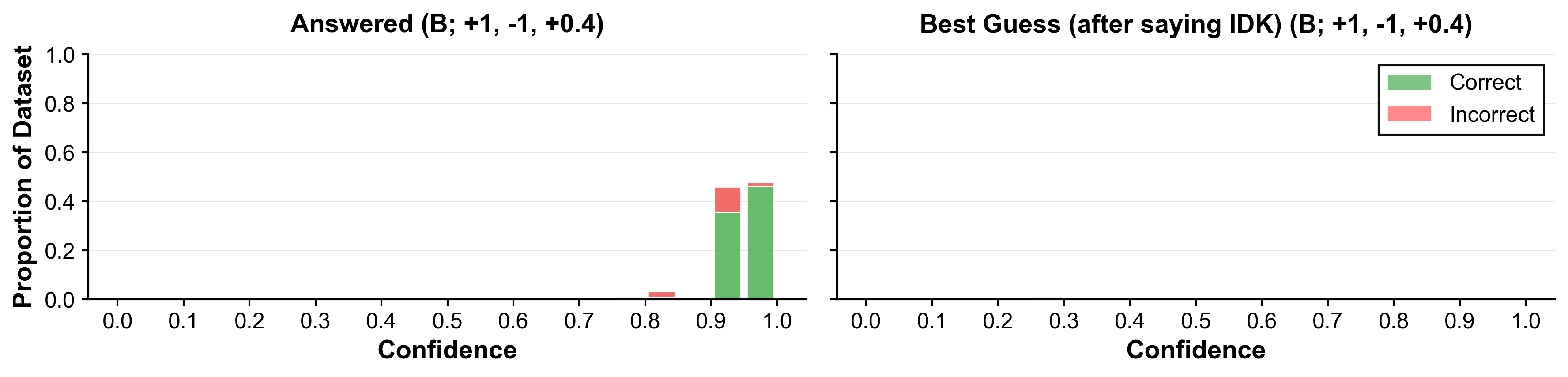}
        \includegraphics[width=0.7\columnwidth]{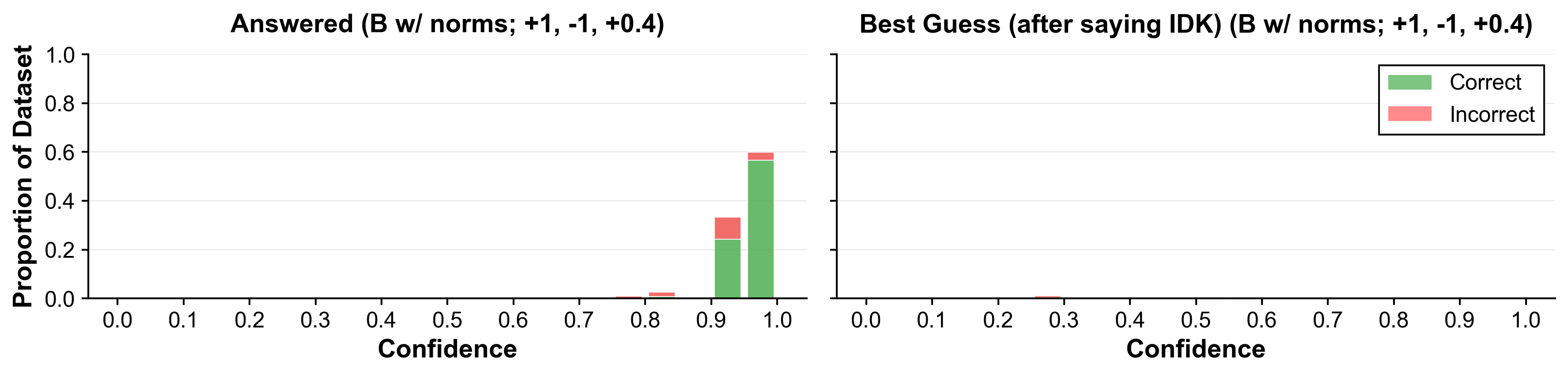}
        \caption{GPT-4o mini}
    \end{subfigure}   
\end{figure}

\begin{figure}[H]
    \ContinuedFloat
    \centering

    \begin{subfigure}{\columnwidth}
        \centering
        \includegraphics[width=0.7\columnwidth]{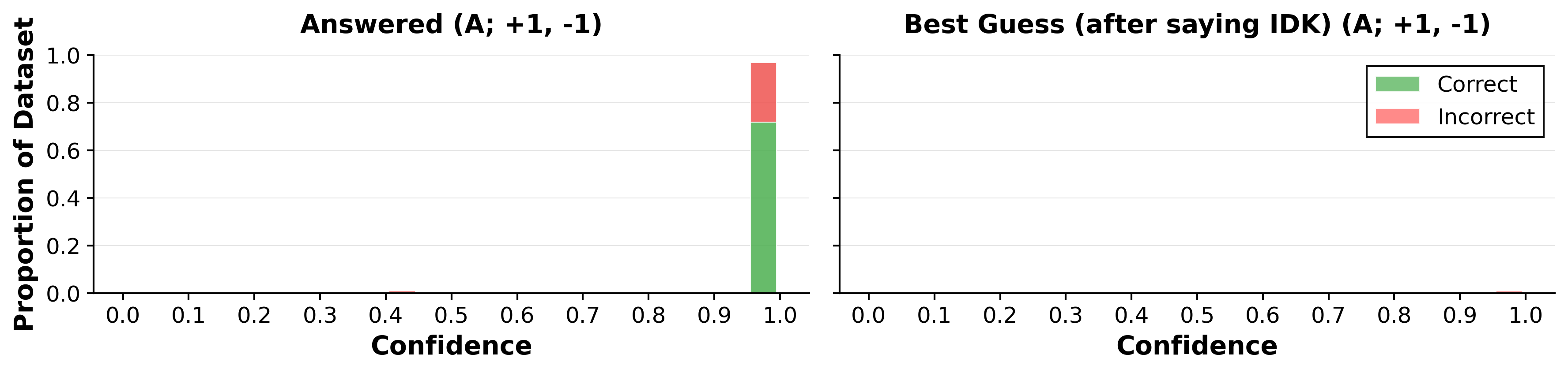}
        \includegraphics[width=0.7\columnwidth]{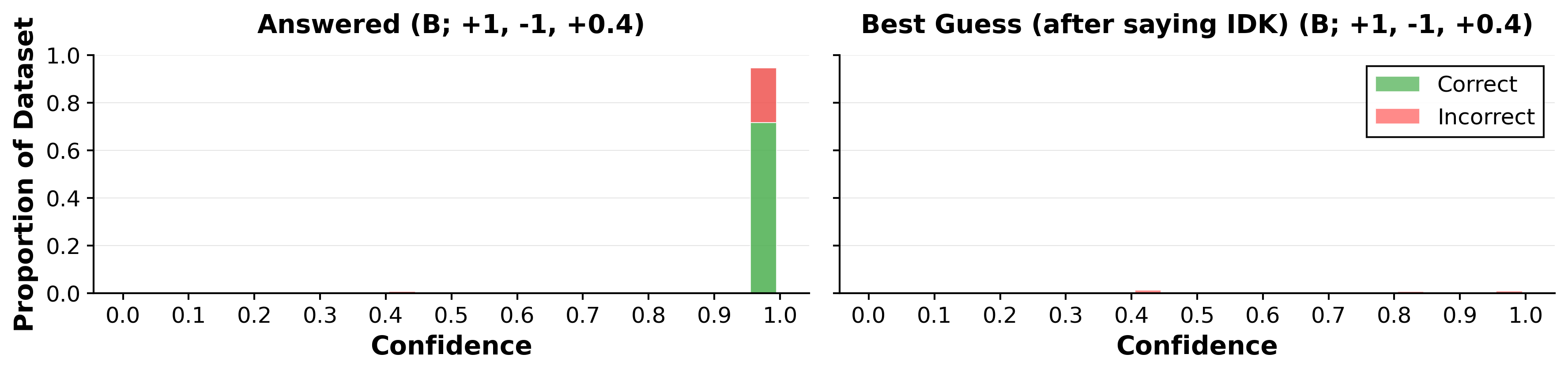}
        \includegraphics[width=0.7\columnwidth]{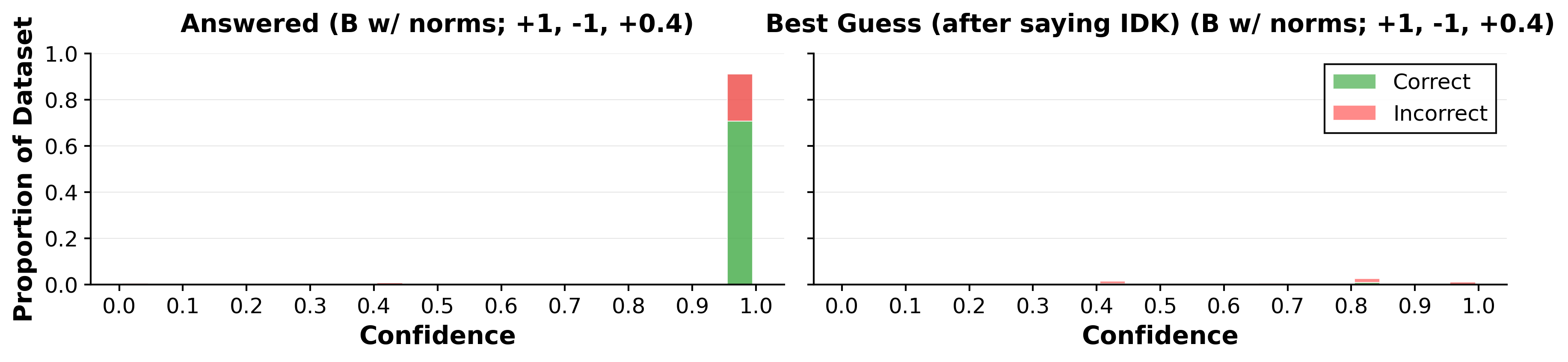}
        \caption{Meta-Llama-3-8B-Instruct}
    \end{subfigure}
\end{figure}

\begin{figure}[H]
    \ContinuedFloat
    \centering
    \begin{subfigure}{\columnwidth}
        \centering
        \includegraphics[width=0.7\columnwidth]{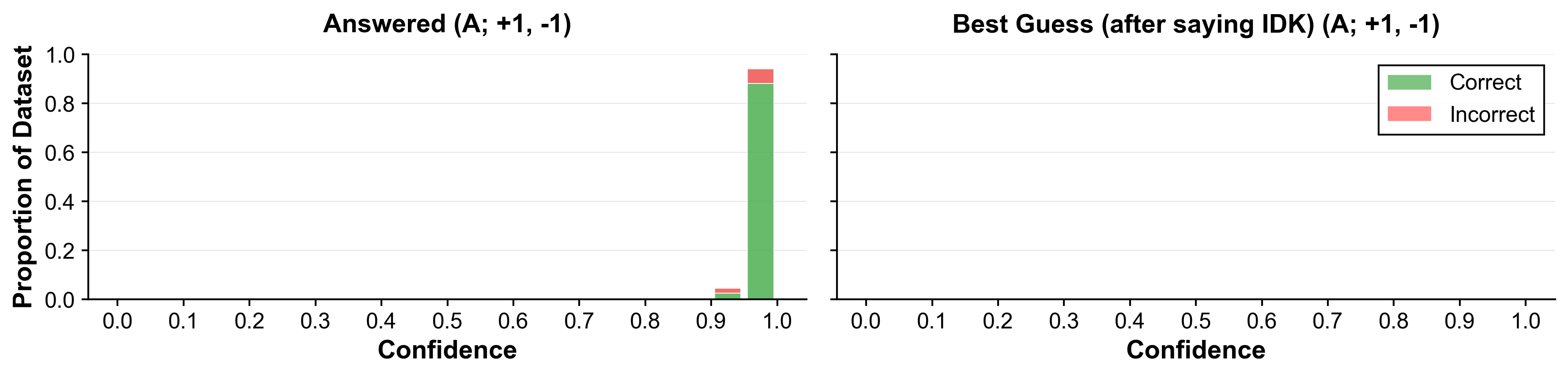}
        \includegraphics[width=0.7\columnwidth]{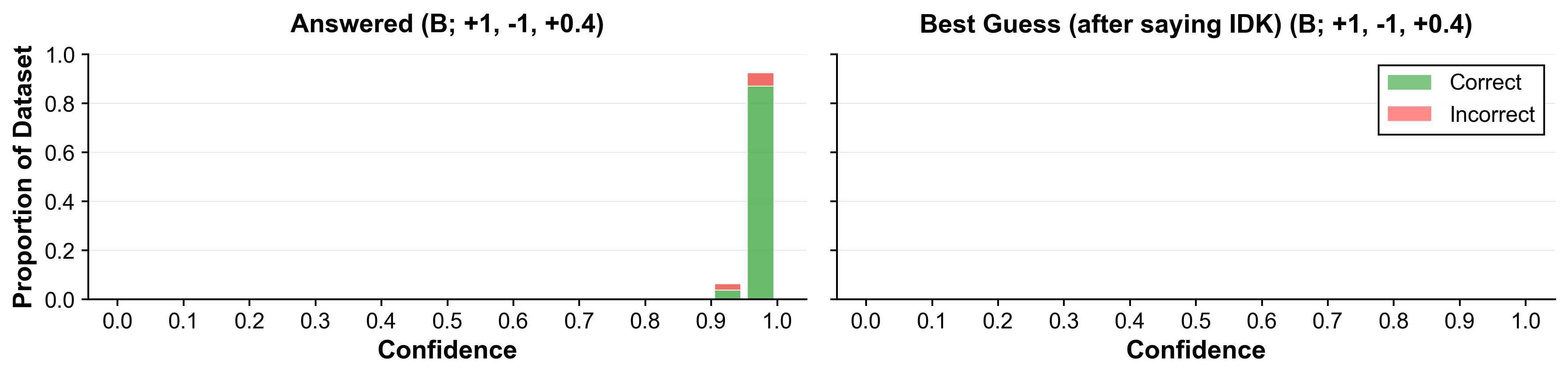}
        \includegraphics[width=0.7\columnwidth]{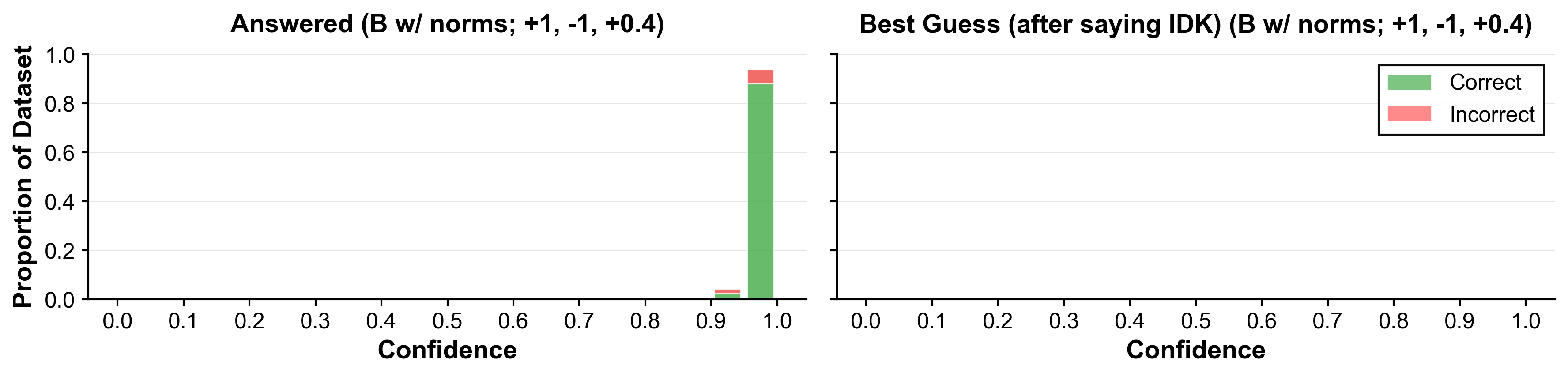}
        \caption{Gemini-3.1-Flash-Lite}
    \end{subfigure}
\end{figure}

\begin{figure}[H]
    \ContinuedFloat
    \centering

    \begin{subfigure}{\columnwidth}
        \centering
        \includegraphics[width=0.7\columnwidth]{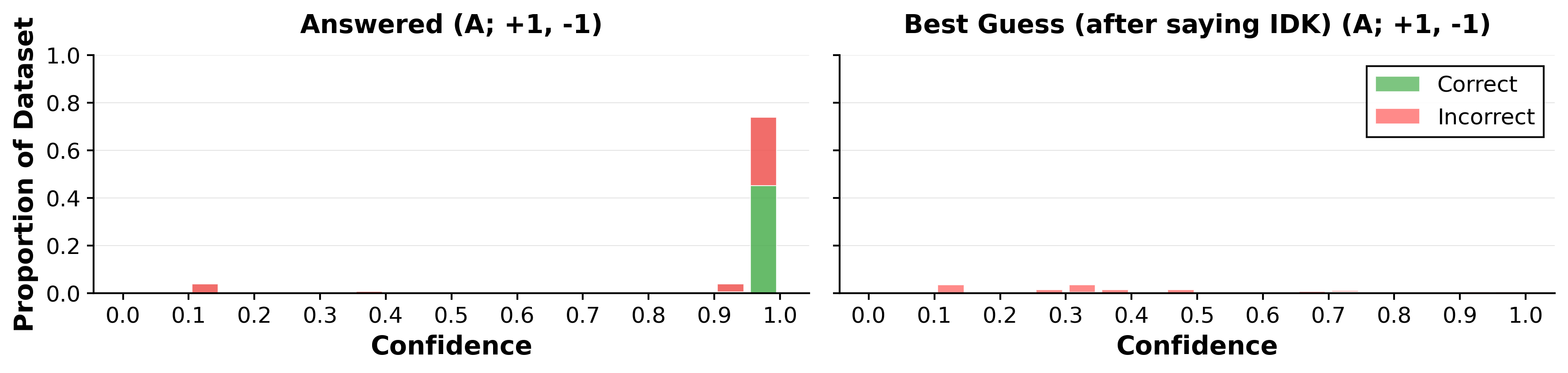}
        \includegraphics[width=0.7\columnwidth]{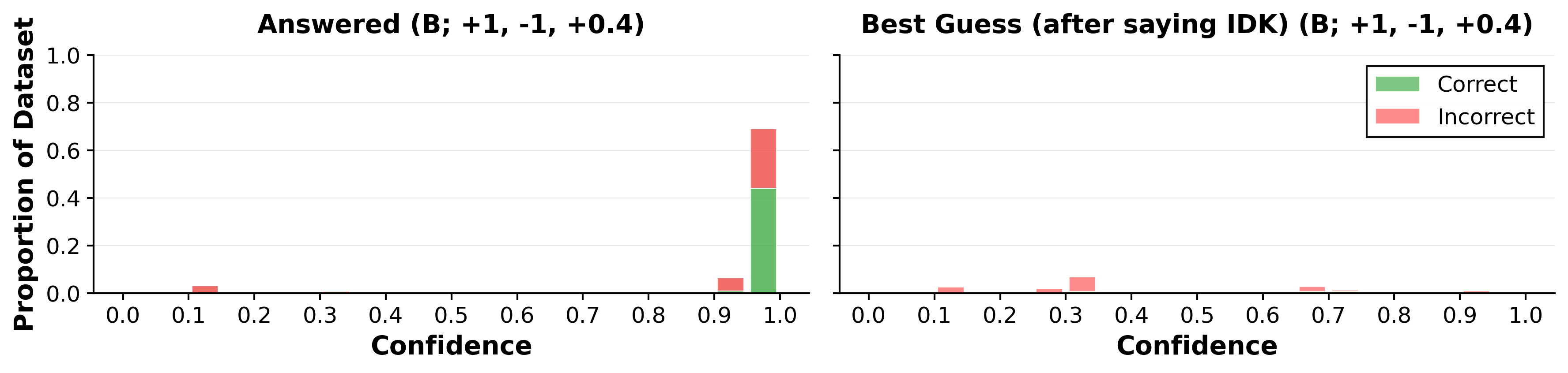}
        \includegraphics[width=0.7\columnwidth]{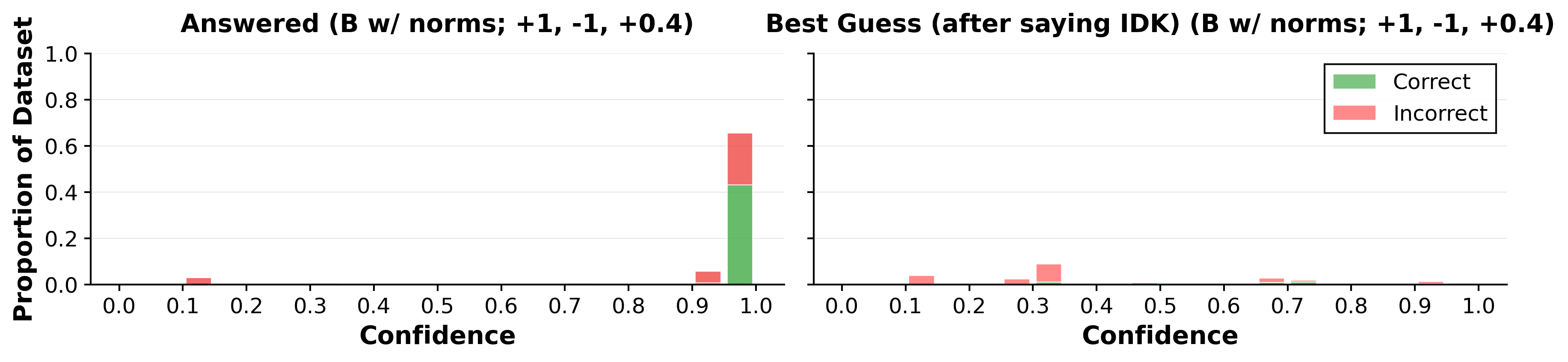}
        \caption{Qwen3-4B-Instruct-2507}
    \end{subfigure}
    \caption{TriviaQA confidence distributions under Scheme A (+1,-1), Scheme B (+1,-1,+0.4), and I-CALM (+1,-1,+0.4) (Scheme B with norms is equivalent to I-CALM.). Within each panel, the left histogram shows first-round answered cases and the right histogram shows second-round best guesses after an initial ``I don't know.''}
    \label{fig:trivia_conf_comparison_all}
\end{figure}

\paragraph{SimpleQA-Verified}

SimpleQA-Verified is the hardest transfer setting, consistent with its construction from questions selected to expose factual failures and its refinement into a challenging 1,000-question benchmark \citep{wei2024simpleqa,haas2025simpleqa}. Pure Eval $\mathrm{FAR}_{\mathrm{answered}}$ ranges from $0.726$ to $0.959$, reflecting severe underlying knowledge limitations. Even in this regime, I-CALM numerically lowers surfaced risk relative to Pure Eval for all five models: from $0.873$ to $0.871$ for GPT-5 mini, $0.917$ to $0.895$ for GPT-4o mini, $0.945$ to $0.908$ for Llama, $0.726$ to $0.699$ for Gemini, and $0.959$ to $0.934$ for Qwen. It also attains the highest ECR for every model. More importantly, $J_{\mathrm{abs}}$ increases monotonically and peaks under I-CALM for four of the five models: $0.077 \rightarrow 0.104 \rightarrow 0.108$ for GPT-5 mini, $0.060 \rightarrow 0.218 \rightarrow 0.273$ for Llama, $0.056 \rightarrow 0.078 \rightarrow 0.092$ for Gemini, and $0.034 \rightarrow 0.045 \rightarrow 0.057$ for Qwen. Llama shows the strongest result, combining the largest reduction in surfaced risk with a substantial improvement in abstention quality. GPT-4o mini is the sole exception: adding normative guidance raises ECR from $0.348$ to $0.450$, but the larger CAR decline from $0.833$ to $0.670$ reduces $J_{\mathrm{abs}}$ from $0.181$ to $0.120$, making Scheme~B the better operating point for this model--dataset pair. Overall calibration still improves from Scheme~A to I-CALM for every model; I-CALM attains the best overall $\widehat{\mathrm{ECE}}$ and Brier score for GPT-4o mini, Llama, and Qwen, whereas Scheme~B is slightly better on one or both overall metrics for GPT-5 mini and Gemini. Together with the routing patterns in Figure~\ref{fig:simpleqa-verified_conf_comparison_all}, these results show that I-CALM can improve error interception even when the model's factual competence is severely constrained, while also identifying a limit: normative guidance can become overly conservative when additional error capture requires withholding too many correct candidates.

\begin{table}[H]
\centering
{\fontsize{9pt}{9pt}\selectfont
\setlength{\tabcolsep}{1.5pt}

\begin{tabular}{clccccc}
\toprule

\textbf{Model} &
\textbf{Scheme} &
\textbf{C} &
\textbf{FAR} &
\textbf{ECR} &
\textbf{CAR} &
\textbf{$J_{\mathrm{abs}}$} \\

\midrule

\multirow{4}{*}{GPT-5 mini}
& Pure Eval & 0.999 & \mbox{0.873 {\scriptsize (0.020)}} & -- & -- & -- \\
& A         & 0.882 & \mbox{0.889 {\scriptsize (0.020)}} & 0.126 & 0.951 & 0.077 \\
& B         & 0.805 & \mbox{0.871 {\scriptsize (0.022)}} & 0.207 & 0.897 & 0.104 \\
& I-CALM    & 0.736 & \mbox{0.871 {\scriptsize (0.023)}} & 0.276 & 0.832 & 0.108 \\

\midrule

\multirow{4}{*}{GPT-4o mini}
& Pure Eval & 1.000 & \mbox{0.917 {\scriptsize (0.017)}} & -- & -- & -- \\
& A         & 0.752 & \mbox{0.902 {\scriptsize (0.021)}} & 0.260 & 0.881 & 0.141 \\
& B         & 0.668 & \mbox{0.888 {\scriptsize (0.024)}} & 0.348 & 0.833 & 0.181 \\
& I-CALM    & 0.561 & \mbox{0.895 {\scriptsize (0.025)}} & 0.450 & 0.670 & 0.120 \\

\midrule

\multirow{4}{*}{\begin{tabular}[c]{@{}c@{}}
Meta-Llama-3-\\
8B-Instruct
\end{tabular}}
& Pure Eval & 0.996 & \mbox{0.945 {\scriptsize (0.014)}} & -- & -- & -- \\
& A         & 0.898 & \mbox{0.931 {\scriptsize (0.016)}} & 0.106 & 0.954 & 0.06 \\
& B         & 0.704 & \mbox{0.916 {\scriptsize (0.023)}} & 0.310 & 0.908 & 0.218 \\
& I-CALM    & 0.476 & \mbox{0.908 {\scriptsize (0.028)}} & 0.540 & 0.733 & 0.273 \\

\midrule

\multirow{4}{*}{\begin{tabular}[c]{@{}c@{}}
Gemini-3.1-\\
Flash-Lite
\end{tabular}}
& Pure Eval & 1.000 & \mbox{0.726 {\scriptsize (0.027)}} & -- & -- & -- \\
& A         & 0.895 & \mbox{0.708 {\scriptsize (0.029)}} & 0.121 & 0.935 & 0.056 \\
& B         & 0.867 & \mbox{0.686 {\scriptsize (0.031)}} & 0.156 & 0.922 & 0.078 \\
& I-CALM    & 0.826 & \mbox{0.699 {\scriptsize (0.032)}} & 0.200 & 0.892 & 0.092 \\

\midrule

\multirow{4}{*}{\begin{tabular}[c]{@{}c@{}}
Qwen3-4B-\\
Instruct-2507
\end{tabular}}
& Pure Eval & 0.997 & \mbox{0.959 {\scriptsize (0.012)}} & -- & -- & -- \\
& A         & 0.662 & \mbox{0.935 {\scriptsize (0.018)}} & 0.340 & 0.694 & 0.034 \\
& B         & 0.558 & \mbox{0.930 {\scriptsize (0.020)}} & 0.445 & 0.600 & 0.045 \\
& I-CALM    & 0.438 & \mbox{0.934 {\scriptsize (0.023)}} & 0.565 & 0.492 & 0.057 \\

\bottomrule
\end{tabular}
}

\caption{SimpleQA Verified results ($N_{\mathrm{total}}=14{,}267$): selective answering performance metrics for GPT-5 mini, GPT-4o mini, Meta-Llama-3-8B-Instruct, Gemini-3.1-Flash-Lite, Qwen3-4B-Instruct-2507 under Pure Eval, Scheme A, Scheme B, and I-CALM. C stands for coverage; FAR stands for $\mathrm{FAR}_{\mathrm{answered}}$. Parentheses report 95\% CI half-widths.}
\label{tab:combined_performance_simpleqa}
\end{table}

\begin{table}[H]
\centering
{\fontsize{9pt}{9pt}\selectfont
\setlength{\tabcolsep}{0.5pt}

\begin{tabular}{clcccc}
\toprule

\multirow{2}{*}{\textbf{Model}} &
\multirow{2}{*}{\textbf{Scheme}} &
\multicolumn{2}{c}{$\widehat{\textbf{ECE}}$} &
\multicolumn{2}{c}{\textbf{Brier Score}} \\

\cmidrule(lr){3-4}
\cmidrule(lr){5-6}

& &
\textbf{Answered} &
\textbf{Overall} &
\textbf{Answered} &
\textbf{Overall} \\

\midrule

\multirow{3}{*}{GPT-5 mini}
& A      & 0.483 & 0.438 & \mbox{0.374 {\scriptsize (0.032)}} & \mbox{0.349 {\scriptsize (0.030)}} \\
& B      & 0.484 & 0.404 & \mbox{0.379 {\scriptsize (0.033)}} & \mbox{0.345 {\scriptsize (0.031)}} \\
& I-CALM & 0.520 & 0.414 & \mbox{0.416 {\scriptsize (0.036)}} & \mbox{0.345 {\scriptsize (0.030)}} \\

\midrule

\multirow{3}{*}{GPT-4o mini}
& A      & 0.792 & 0.664 & \mbox{0.723 {\scriptsize (0.032)}} & \mbox{0.575 {\scriptsize (0.031)}} \\
& B      & 0.777 & 0.620 & \mbox{0.707 {\scriptsize (0.034)}} & \mbox{0.521 {\scriptsize (0.031)}} \\
& I-CALM & 0.787 & 0.569 & \mbox{0.722 {\scriptsize (0.037)}} & \mbox{0.472 {\scriptsize (0.031)}} \\

\midrule

\multirow{3}{*}{\begin{tabular}[c]{@{}c@{}}
Meta-Llama-3-\\
8B-Instruct
\end{tabular}}
& A      & 0.898 & 0.835 & \mbox{0.886 {\scriptsize (0.023)}} & \mbox{0.829 {\scriptsize (0.023)}} \\
& B      & 0.892 & 0.715 & \mbox{0.880 {\scriptsize (0.024)}} & \mbox{0.714 {\scriptsize (0.029)}} \\
& I-CALM & 0.875 & 0.594 & \mbox{0.863 {\scriptsize (0.035)}} & \mbox{0.592 {\scriptsize (0.035)}} \\

\midrule

\multirow{3}{*}{\begin{tabular}[c]{@{}c@{}}
Gemini-3.1-\\
Flash-Lite
\end{tabular}}
& A      & 0.674 & 0.621 & \mbox{0.659 {\scriptsize (0.031)}} & \mbox{0.613 {\scriptsize (0.031)}} \\
& B      & 0.646 & 0.578 & \mbox{0.630 {\scriptsize (0.032)}} & \mbox{0.570 {\scriptsize (0.031)}} \\
& I-CALM & 0.666 & 0.585 & \mbox{0.654 {\scriptsize (0.032)}} & \mbox{0.581 {\scriptsize (0.030)}} \\

\midrule

\multirow{3}{*}{\begin{tabular}[c]{@{}c@{}}
Qwen3-4B-\\
Instruct-2507
\end{tabular}}
& A      & 0.723 & 0.579 & \mbox{0.797 {\scriptsize (0.032)}} & \mbox{0.559 {\scriptsize (0.032)}} \\
& B      & 0.780 & 0.572 & \mbox{0.785 {\scriptsize (0.036)}} & \mbox{0.504 {\scriptsize (0.032)}} \\
& I-CALM & 0.814 & 0.530 & \mbox{0.800 {\scriptsize (0.032)}} & \mbox{0.437 {\scriptsize (0.032)}} \\

\bottomrule
\end{tabular}
}

\caption{SimpleQA Verified results: calibration metrics for GPT-5 mini, GPT-4o mini, Meta-Llama-3-8B-Instruct, Gemini-3.1-Flash-Lite, and Qwen3-4B-Instruct-2507 under Scheme A $(+1,-1)$, Scheme B $(+1,-1,+0.4)$, and I-CALM $(+1,-1,+0.4)$. Parentheses report 95\% CI half-widths.}
\label{tab:combined_calibration_simpleqa}
\end{table}

\begin{figure}[H]
    \centering
    \begin{subfigure}{\columnwidth}
        \centering
        \includegraphics[width=0.7\columnwidth]{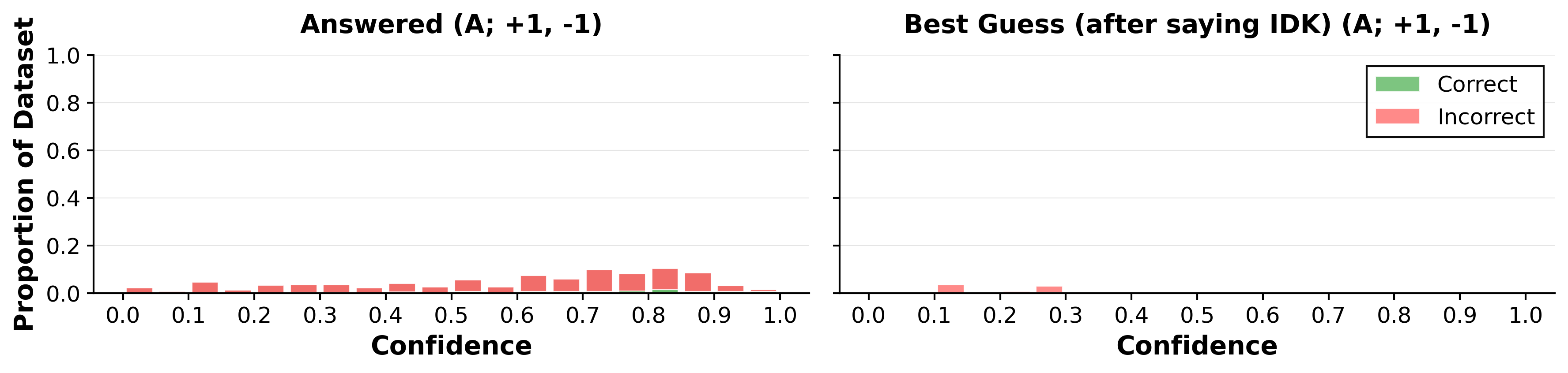}
        \includegraphics[width=0.7\columnwidth]{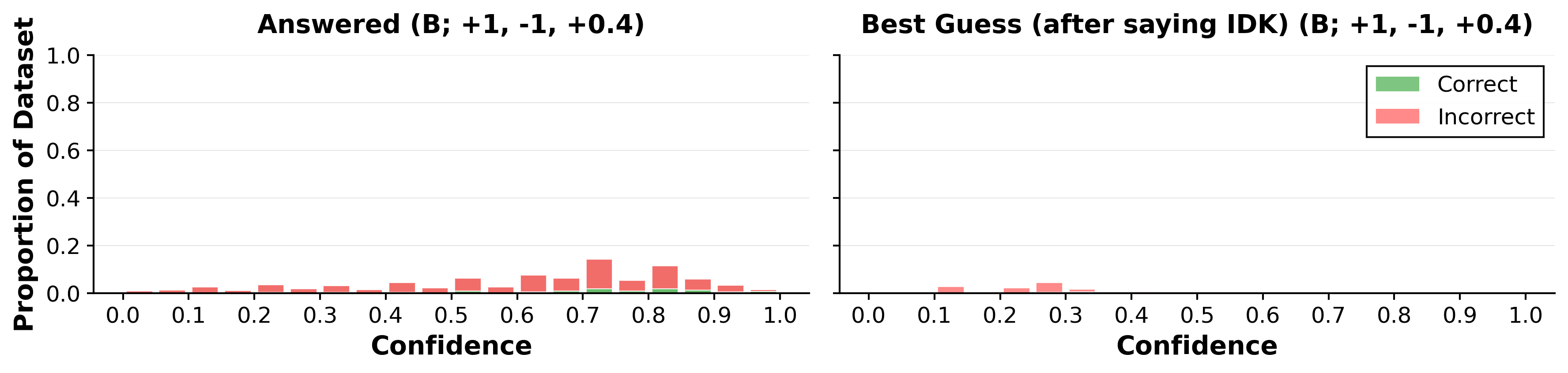}
        \includegraphics[width=0.7\columnwidth]{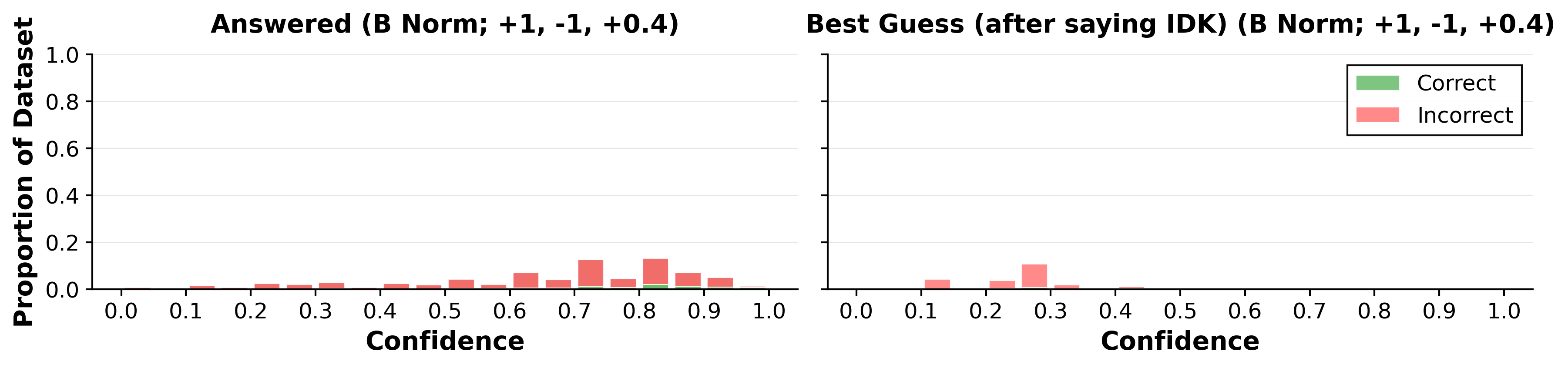}
        \caption{GPT-5 mini}
    \end{subfigure}
\end{figure}
   
\begin{figure}[H] 
    \ContinuedFloat
    \centering
    \begin{subfigure}{\columnwidth}
        \centering
        \includegraphics[width=0.7\columnwidth]{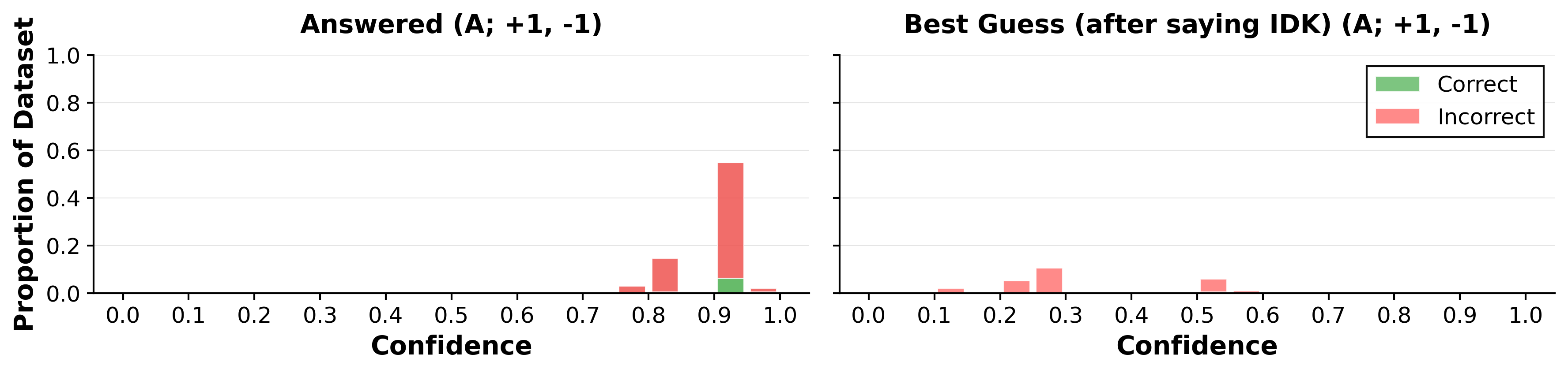}
        \includegraphics[width=0.7\columnwidth]{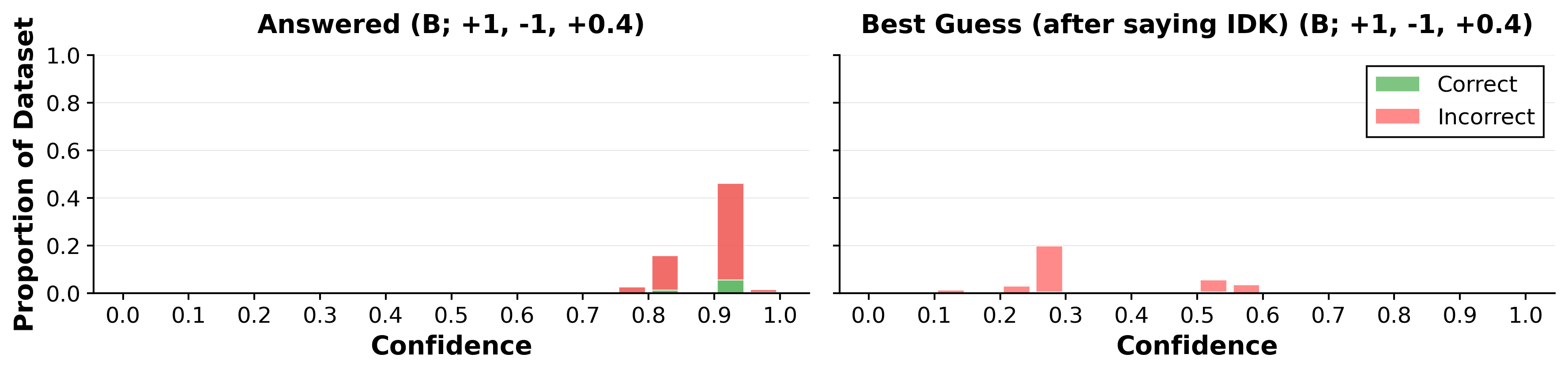}
        \includegraphics[width=0.7\columnwidth]{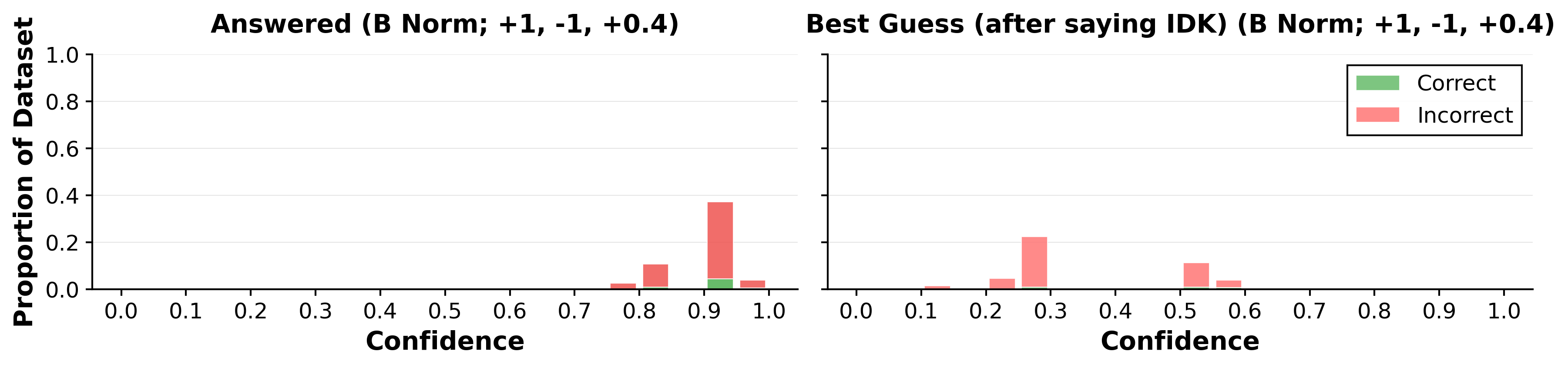}
        \caption{GPT-4o mini}
    \end{subfigure}   
\end{figure}

\begin{figure}[H]
    \ContinuedFloat
    \centering

    \begin{subfigure}{\columnwidth}
        \centering
        \includegraphics[width=0.7\columnwidth]{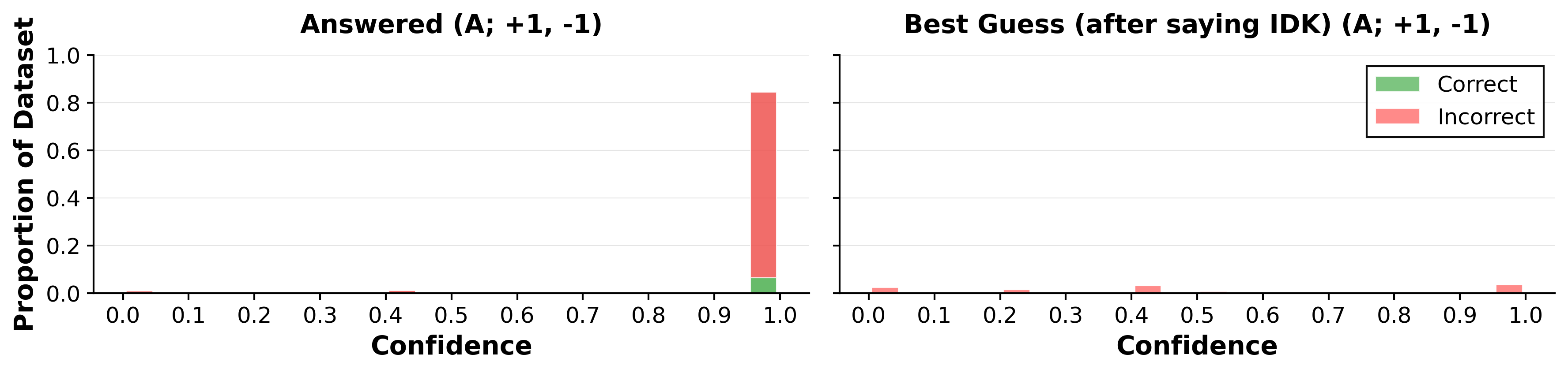}
        \includegraphics[width=0.7\columnwidth]{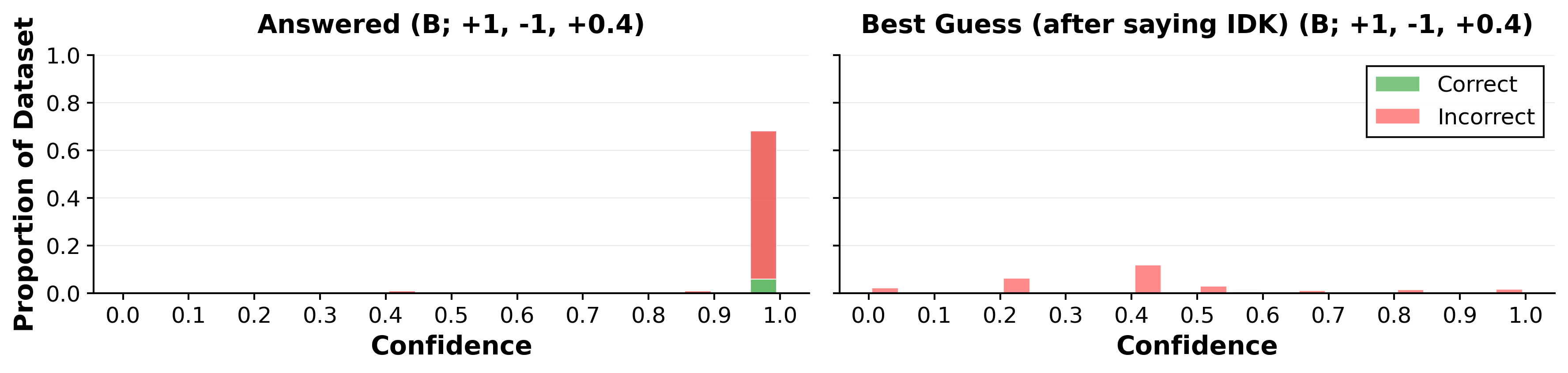}
        \includegraphics[width=0.7\columnwidth]{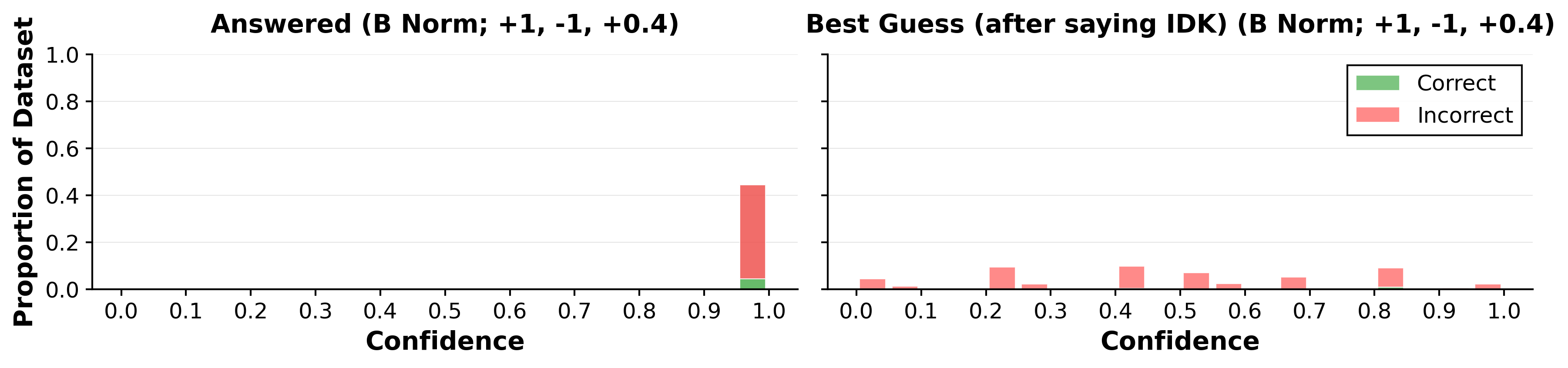}
        \caption{Meta-Llama-3-8B-Instruct}
    \end{subfigure}
\end{figure}

\begin{figure}[H]  
    \ContinuedFloat
    \centering
    \begin{subfigure}{\columnwidth}
        \centering
        \includegraphics[width=0.7\columnwidth]{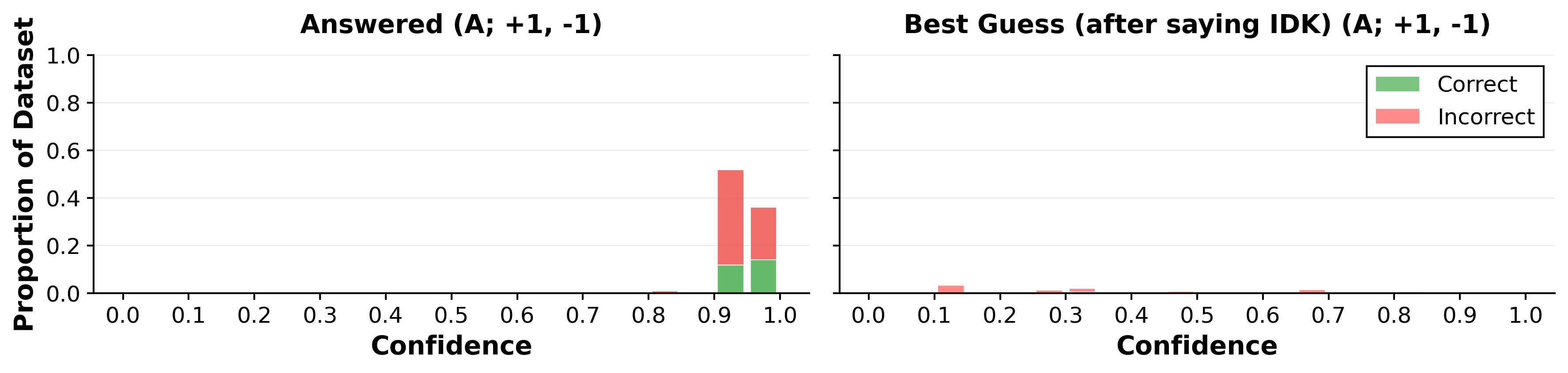}
        \includegraphics[width=0.7\columnwidth]{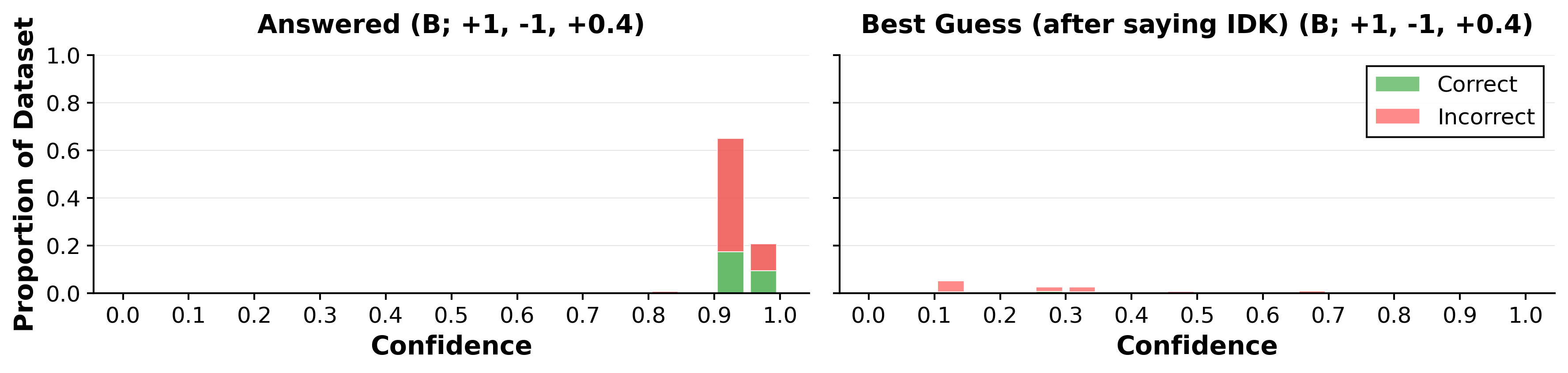}
        \includegraphics[width=0.7\columnwidth]{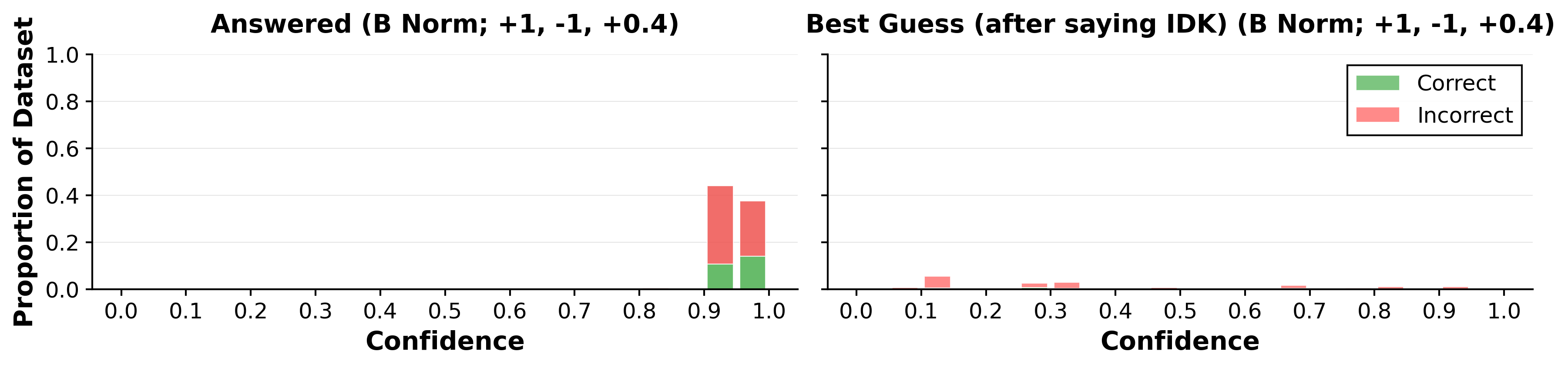}
        \caption{Gemini-3.1-Flash-Lite}
    \end{subfigure}
\end{figure}

\begin{figure}[H]
    \ContinuedFloat
    \centering

    \begin{subfigure}{\columnwidth}
        \centering
        \includegraphics[width=0.7\columnwidth]{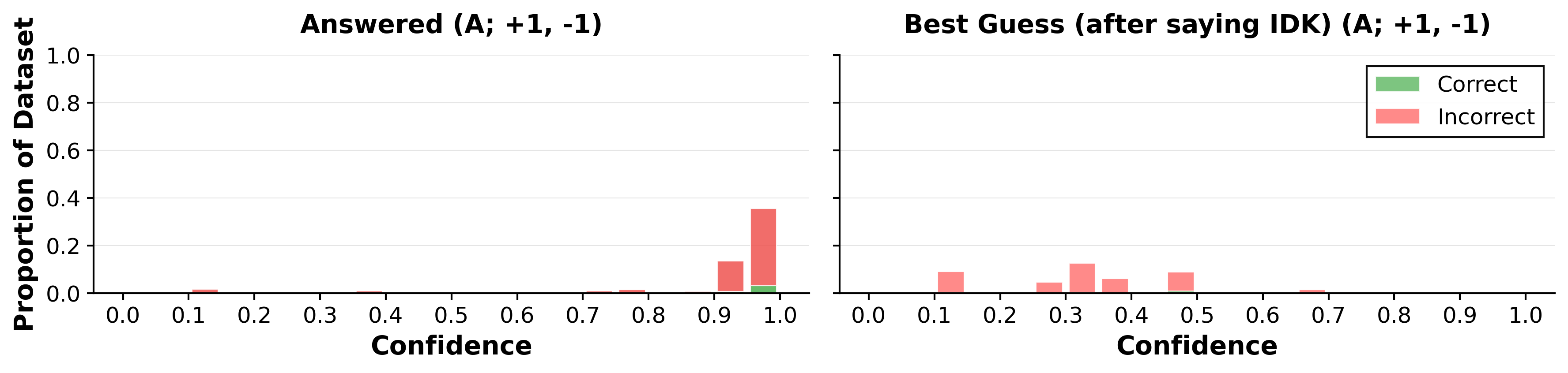}
        \includegraphics[width=0.7\columnwidth]{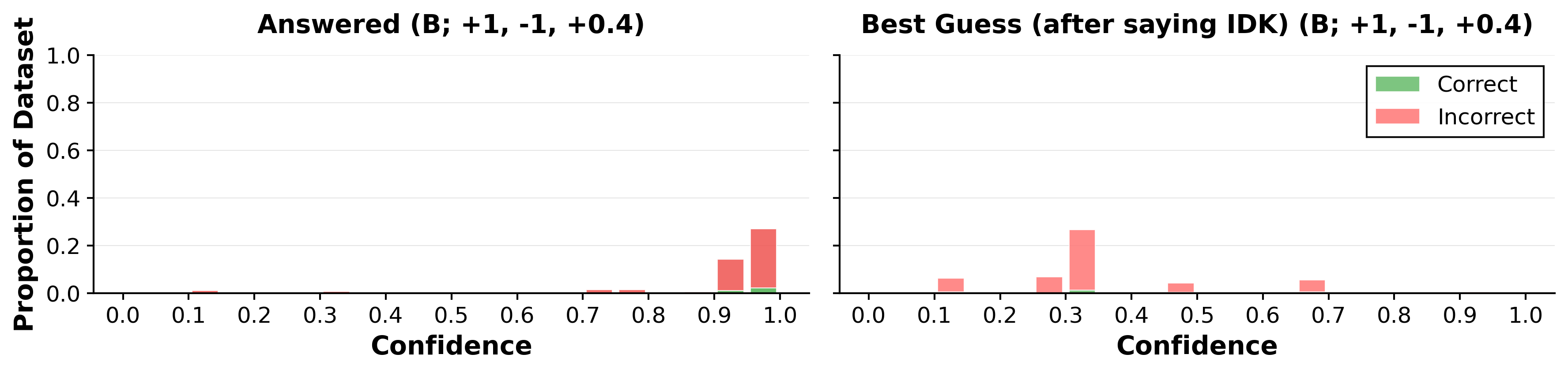}
        \includegraphics[width=0.7\columnwidth]{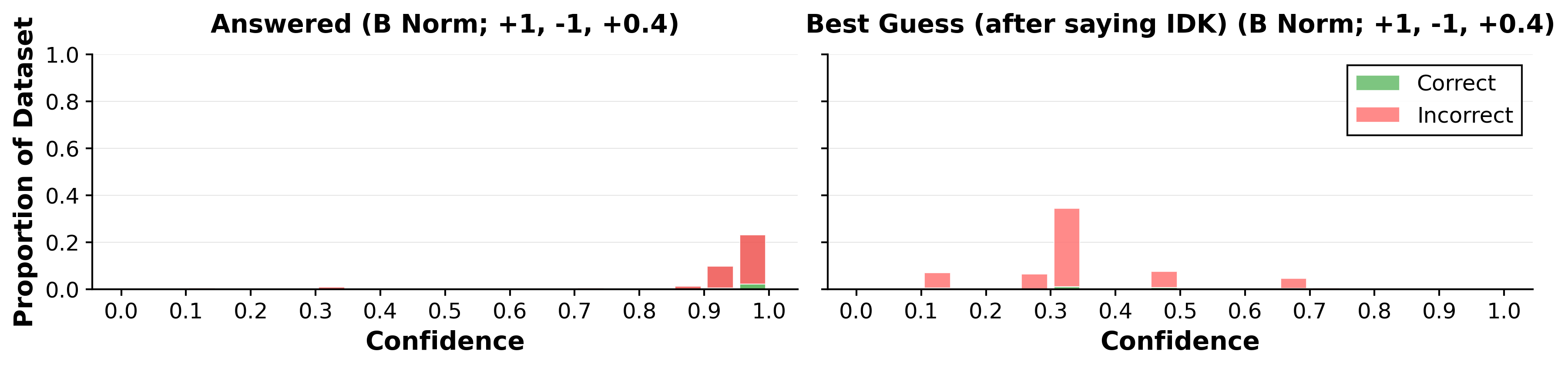}
        \caption{Qwen3-4B-Instruct-2507}
    \end{subfigure}
    \caption{SimpleQA Verified confidence distributions under Scheme A (+1,-1), Scheme B (+1,-1,+0.4), and I-CALM (+1,-1,+0.4) (Scheme B with norms is equivalent to I-CALM.). Within each panel, the left histogram shows first-round answered cases and the right histogram shows second-round best guesses after an initial ``I don't know.''}
    \label{fig:simpleqa-verified_conf_comparison_all}
\end{figure}

\subsection{Deployment Extension: Post-Hoc Confidence Thresholding with Certified FAR Control}
\label{subsec:new_false_answer_rate_control}

The goal of this subsection is to turn the model's self-reported confidence into a threshold rule that, with high probability over a held-out calibration sample, controls the population FAR among accepted answers at or below a user-specified target. When a held-out calibration/validation split is used, we randomly partition the data using a fixed random seed of 1.
We use the cumulative false-answer rate (CFAR) curve as a descriptive diagnostic, and consider two finite-sample certified rules based on one-sided Clopper--Pearson upper confidence bounds (CP-UCB) \citep{clopper1934use}: a conservative Bonferroni CP-UCB baseline and a multistart fixed-sequence ordered-testing alternative.
This places the subsection in the selective-prediction and risk-controlled refusal literature, where the central deployment question is how much coverage can be obtained at a controlled error level \citep{mao2025calibrating,wang2025safer,oehri2025trusted}.

As in Section~\ref{sec:problem_protocol}, each query is paired with a final answer, a correctness label, and a self-reported confidence. When the model initially abstains, the final answer is its elicited best guess; 
thus this section studies a forced-candidate regime: every query is paired with a candidate answer and confidence score, and the interface decides whether to surface or withhold that candidate.
Let $u:=1-\tau^{self}\in[0,1]$ denote reported uncertainty. For a fixed uncertainty cutoff $u$, the system accepts all answers with reported uncertainty at most $u$, equivalently all answers with confidence at least $1-u$. On a calibration set of size $n$, define
\begin{equation}
n(u)=\sum_{i=1}^n \mathbf{1}\{u_i\le u\},
\qquad
k(u)=\sum_{i=1}^n \mathbf{1}\{u_i\le u\}e_i,
\label{eq:cp_counts}
\end{equation}
where $e_i=\mathbf{1}\{\text{the final answer is incorrect}\}$. The empirical cumulative false-answer rate (CFAR) curve is
$$
\widehat{\mathrm{CFAR}}(u)=
\begin{cases}
k(u)/n(u), & n(u)>0,\\
0, & n(u)=0.
\end{cases}
$$
Equivalently, $\widehat{\mathrm{CFAR}}(u)$ is the empirical FAR among answers that would be accepted by the uncertainty threshold $u$.
When plotting $\widehat{\mathrm{CFAR}}(u)$, we additionally show pointwise 95\% confidence intervals as descriptive uncertainty bands. These intervals are not used directly for threshold selection: because they are pointwise and two-sided, their upper endpoints do not provide a valid post-selection guarantee after scanning multiple thresholds.

\subsubsection{Bonferroni CP-UCB threshold selection}

To obtain a certified rule, we fix a prespecified finite grid of uncertainty thresholds
$$
\mathcal U=\{0,0.01,\ldots,1.00\},
$$
independent of the calibration data, and we set the overall failure probability to $\delta=0.05$. 
The certification is with respect to this prespecified grid; a finer grid gives finer threshold resolution but induces a slightly more conservative Bonferroni correction.
For each $u\in\mathcal U$, we compute the one-sided CP-UCB
\begin{equation}
\mathrm{UCB}_{\mathrm{CP}}(u)=
\bigl\{
\begin{array}{@{}l@{}}
1, \text{ if } k(u)=n(u),\\[2pt]
\operatorname{Beta}^{-1}\!\Bigl(1-\delta/|\mathcal U|;\,
k(u)+1,\\
\hspace{3.5em} n(u)-k(u)\Bigr),
\text{ if } k(u)<n(u).
\end{array}
\label{eq:cp_ucb}
\end{equation}
which yields a Bonferroni-corrected simultaneous upper band over the entire grid. 
$\operatorname{Beta}^{-1}(q;a,b)$ denotes the $q$--quantile of the $\mathrm{Beta}(a,b)$ distribution.
Proposition~\ref{prop:cp_ucb} gives the guarantee. 
For any target FAR $r\in[0,1]$, we then choose the most permissive certified threshold
$$
\hat u_r=\max\{u\in\mathcal U:\mathrm{UCB}_{\mathrm{CP}}(u)\le r\},
$$
whenever the feasible set is nonempty; otherwise the rule rejects all answers. We define the corresponding confidence threshold as $\hat\tau_r=1-\hat u_r$. At deployment time, the system accepts the model output iff $\tau^{self}\ge\hat\tau_r$; otherwise the answer is rejected or held for further review. The CP-UCB calibration procedure is summarized in Algorithm 1 below.

\begin{proposition}[Finite-sample FAR control via CP-UCB]
\label{prop:cp_ucb}
Let $(U,E)$ denote the reported uncertainty and error indicator, respectively, for a fresh example drawn from the same distribution, where $E=1$ iff the final answer is incorrect. 
Assume the calibration examples and future examples are i.i.d. draws from a common distribution over $(U,E)$.
Let $\mathcal U=\{u_1,\dots,u_M\}\subset[0,1]$ be a prespecified finite grid of uncertainty thresholds, independent of the calibration data. For each calibration example $i=1,\dots,n$, let $u_i=1-\tau_i^{self}$ be the reported uncertainty and let $e_i\in\{0,1\}$ indicate whether the corresponding final answer is incorrect. For each $u\in\mathcal U$, define
$$
n(u)=\sum_{i=1}^n \mathbf{1}\{u_i\le u\},
\qquad
k(u)=\sum_{i=1}^n \mathbf{1}\{u_i\le u\}e_i.
$$
Define
$$
R(u)=\Pr(E=1\mid U\le u),
$$
with the convention $R(u)=0$ when $\Pr(U\le u)=0$.
Define the one-sided Clopper--Pearson upper bound
$$
\mathrm{UCB}_{\mathrm{CP}}(u)
=
\left\{
\begin{array}{@{}l@{}}
1,\quad \text{if } k(u)=n(u),\\[2pt]
\operatorname{Beta}^{-1}\!\Bigl(1-\delta/M;\,k(u)+1,\,n(u)-k(u)\Bigr),\\
\qquad \text{if } k(u)<n(u).
\end{array}
\right.
$$
For a target FAR $r\in[0,1]$, let
$\hat u_r=\max\{u\in\mathcal U:\mathrm{UCB}_{\mathrm{CP}}(u)\le r\}$,
whenever the feasible set is nonempty; otherwise the rule rejects all answers. Then
$$
\Pr\!\Big(\forall u\in\mathcal U,\ R(u)\le \mathrm{UCB}_{\mathrm{CP}}(u)\Big)\ge 1-\delta.
$$
Consequently, whenever the feasible set is nonempty, the selected threshold satisfies
$R(\hat u_r)\le r$.
If the feasible set is empty, the reject-all rule trivially satisfies the risk constraint.
\end{proposition}

\begin{proof}[Proof sketch]
For a fixed $u$, let
$$
I_i(u)=\mathbf{1}\{U_i \le u\}.
$$
Since $(U_i,E_i)$ are i.i.d. and $I_i(u)$ depends only on $U_i$, we have
$$
E_i \mid I_i(u)=1 \sim \mathrm{Bernoulli}(R(u)).
$$
Therefore, conditional on there being $m$ accepted examples at threshold $u$, the corresponding error indicators are i.i.d. $\mathrm{Bernoulli}(R(u))$, so the number of accepted errors is distributed as $\mathrm{Binomial}(m,R(u))$.
The one-sided Clopper--Pearson construction gives
$$
\Pr\!\bigl(R(u)\le \mathrm{UCB}_{\mathrm{CP}}(u)\bigr)\ge 1-\delta/M.
$$
A union bound over the $M$ prespecified grid points yields the simultaneous event with probability at least $1-\delta$. The selection rule only chooses thresholds whose certified upper bound is at most $r$, which yields $R(\hat u_r)\le r$ whenever the feasible set is nonempty.
\end{proof}

\begin{tcolorbox}[
  enhanced jigsaw,
  colback=gray!5,
  colframe=gray!40,
  title=Algorithm 1: CP-UCB threshold selection,
  boxsep=3pt,
  left=3pt,
  right=3pt,
  top=3pt,
  bottom=3pt,
  parskip=0pt,
  breakable,
  break at=-\baselineskip/0pt,
  height fixed for=middle,
  pad at break=1mm
]
\textbf{Input:} calibration set $\{(u_i,e_i)\}_{i=1}^n$, target FAR $r\in[0,1]$, overall failure probability $\delta\in(0,1)$, prespecified grid $\mathcal U=\{0,0.01,\ldots,1.00\}$.

\textbf{For each} $u\in\mathcal U$:
\begin{enumerate}[leftmargin=*,itemsep=0.2em]
    \item Compute the number of accepted answers:
    \[
    n(u)=\sum_{i=1}^n \mathbf{1}\{u_i\le u\}.
    \]
    
    \item Compute the number of incorrect accepted answers:
    \[
    k(u)=\sum_{i=1}^n \mathbf{1}\{u_i\le u\}e_i.
    \]
    
    \item Set the per-threshold significance level:
    \[
    \alpha=\delta/|\mathcal U|.
    \]
    
    \item Compute the one-sided CP upper bound:
    \[
    \mathrm{UCB}_{\mathrm{CP}}(u)=1
    \quad \text{if } k(u)=n(u),
    \]
    otherwise,
    \[
    \mathrm{UCB}_{\mathrm{CP}}(u)=
    \operatorname{Beta}^{-1}\!\Bigl(
      1-\alpha;\,k(u)+1,\,
      n(u)-k(u)
    \Bigr).
    \]
\end{enumerate}

\textbf{Return the calibrated threshold}
$$
\hat u_r=\max\{u\in\mathcal U:\mathrm{UCB}_{\mathrm{CP}}(u)\le r\},
$$
if the feasible set is nonempty; otherwise return a reject-all rule.

Define the corresponding confidence threshold as
$\hat\tau_r=1-\hat u_r$
whenever $\hat u_r$ exists.

\textbf{Deployment rule:} accept the model output iff $\tau^{self}\ge \hat\tau_r$ (equivalently, $u\le \hat u_r$); otherwise reject or hold the answer for further review.
\end{tcolorbox}

The guarantee is a population statement obtained entirely from the calibration split; the validation split is used only as a holdout check on how conservative the certified rule is in finite samples. Figure~\ref{fig:gpt5_penalty} is descriptive only and uses the full PopQA dataset for GPT-5 mini. As $u$ increases, the acceptance set expands to include more uncertain outputs, so the empirical CFAR curves generally rise. The curves are close at very low uncertainty, separate over much of the mid-uncertainty range---where  I-CALM is typically lowest, followed by Scheme B and then Scheme A---and reconverge near $u=1$, where $\widehat{\mathrm{CFAR}}(1)$ equals the overall forced-answer FAR. This is consistent with Appendix~\ref{subsec:E2_full_result}, which shows similar $\text{FAR}_{\mathrm{overall}}$ across schemes.

\begin{figure}[H]
\centering
\includegraphics[width=0.7\linewidth]{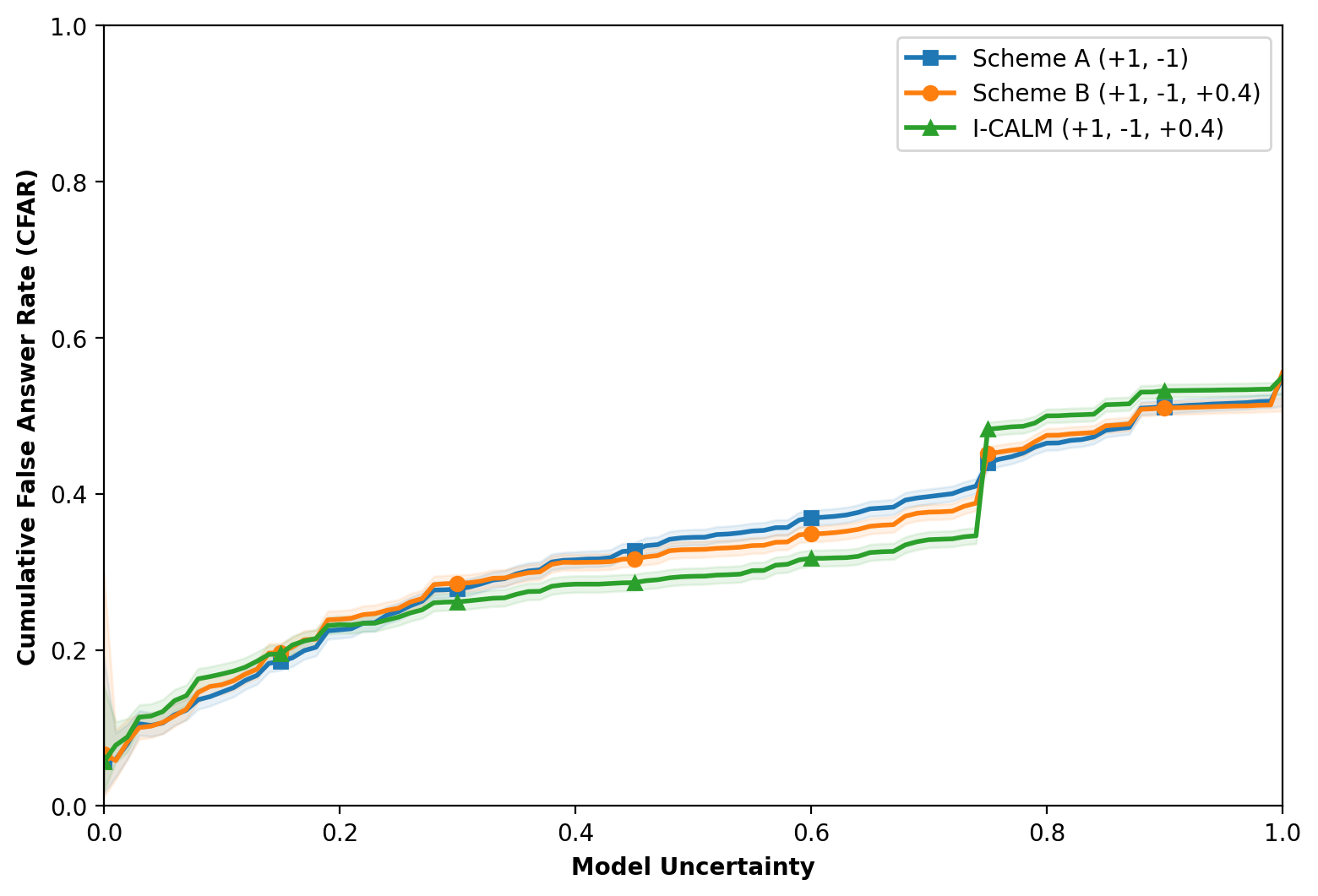}
\caption{Empirical CFAR as a function of reported uncertainty $u = 1 - \tau^{\mathrm{self}}$ for GPT-5 mini on the full PopQA dataset under Scheme A, Scheme B, and  I-CALM. Shaded bands show pointwise 95\% confidence intervals.}
\label{fig:gpt5_penalty}
\end{figure}

For actual certification, Figure~\ref{fig:cfar_train} plots the calibration-split empirical CFAR curves (solid) together with their Bonferroni CP-UCB envelopes (dashed). For the illustrative target $r=0.3$, the selected thresholds are
$$
\hat u_{0.3}^A=0.24,\qquad
\hat u_{0.3}^B=0.25,\qquad
\hat u_{0.3}^{B\text{ w/ norms}}=0.26,
$$
and the validation markers all lie below the target line. Table~\ref{tab:risk_targets} extends this comparison to $r\in\{0.1,0.2,0.3,0.4\}$. On this split, all validation FAR values among accepted answers fall below their targets. The $r=0.1$ row returns reject-all for all schemes, reflecting the conservativeness induced by the 20\% calibration split together with the Bonferroni correction.

\begin{figure}[H]
\centering
\includegraphics[width=0.7\linewidth]{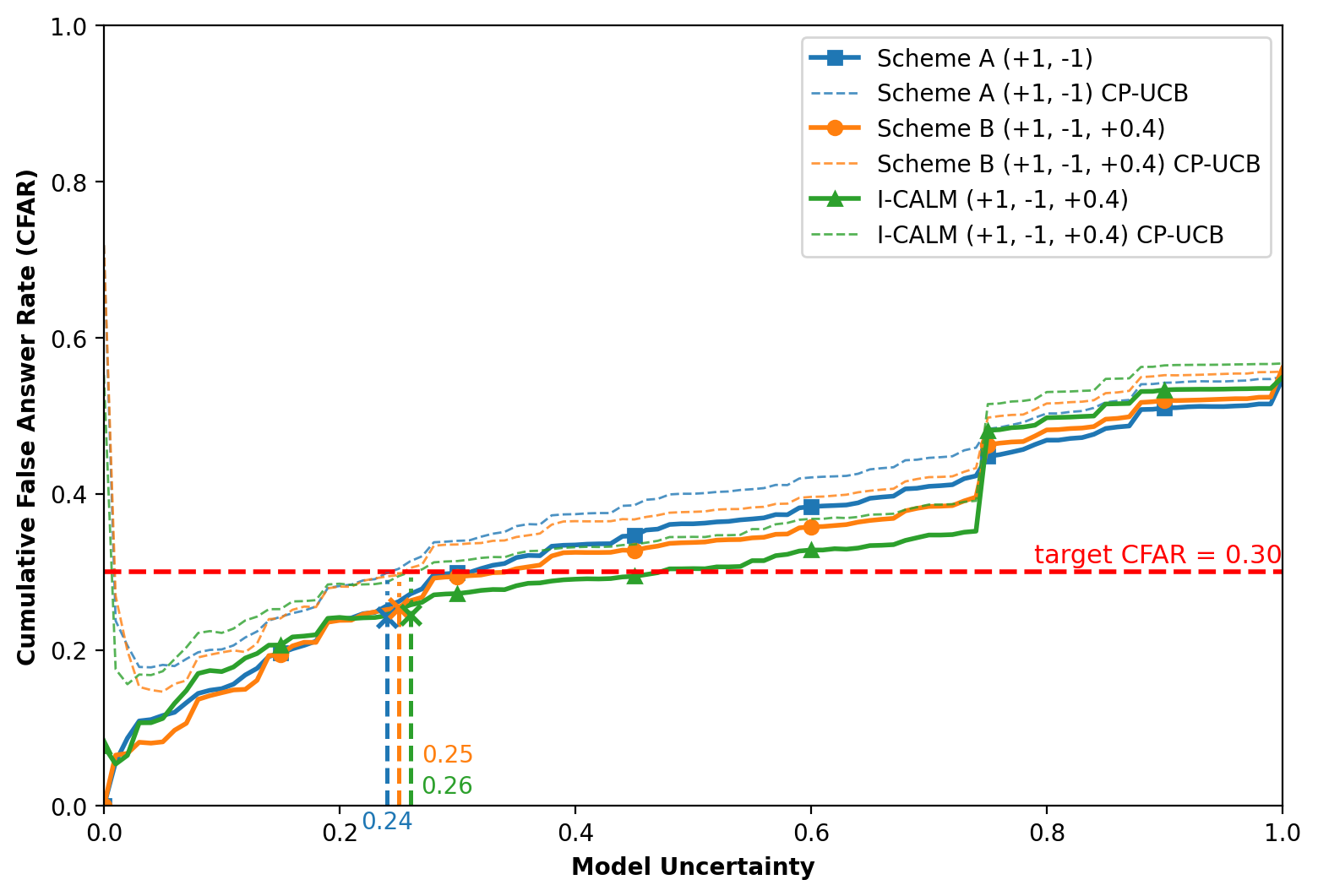}
\caption{Calibration-split empirical CFAR curves (solid) and Bonferroni CP-UCB envelopes (dashed) for GPT-5 mini on PopQA at the illustrative risk target $r=0.3$. Vertical dashed lines denote the selected uncertainty thresholds $\hat u_r$; x-shaped markers show the validation FAR among accepted answers at those thresholds.}
\label{fig:cfar_train}
\end{figure}

Under this formulation, the relevant deployment comparison is coverage at fixed risk. Both the conditional risk curve $R(u)$ and the distribution of reported uncertainty $U$ matter: lowering the CFAR/CP-UCB curve can support a larger certified threshold within a scheme, but cross-scheme coverage also depends on how much mass the prompting intervention shifts across uncertainty levels. In Table~\ref{tab:risk_targets}, Scheme B achieves the highest validation acceptance at each nontrivial risk target, while  I-CALM often matches or slightly exceeds the certified threshold without increasing acceptance. 
The practical takeaway is that prompt-level incentives change the coverage--risk operating point of the confidence filter, so threshold size alone is not a sufficient proxy for deployment utility.

\begin{table}[H]
\centering
{\fontsize{9pt}{9pt}\selectfont
\setlength{\tabcolsep}{4pt}

\begin{tabular}{ccccc}
\toprule

\shortstack{\textbf{Risk}\\\textbf{Target}}
& \shortstack{\textbf{Calibrated}\\\textbf{Uncertainty}\\\textbf{Threshold}}
& \shortstack{\textbf{Calibration}\\\textbf{Acceptance}\\\textbf{Rate}}
& \shortstack{\textbf{Validation}\\\textbf{Acceptance}\\\textbf{Rate}}
& \shortstack{\textbf{Validation}\\$\widehat{\textbf{CFAR}}$} \\

\midrule

\multicolumn{5}{l}{\textbf{Scheme A (+1, -1)}} \\

\midrule

0.1 & reject-all & 0.00 & \mbox{0.00 {\scriptsize (0.00)}} & -- \\
0.2 & 0.08 & 0.21 & \mbox{0.20 {\scriptsize (0.01)}} & \mbox{0.13 {\scriptsize (0.01)}} \\
0.3 & 0.24 & 0.44 & \mbox{0.43 {\scriptsize (0.01)}} & \mbox{0.24 {\scriptsize (0.01)}} \\
0.4 & 0.48 & 0.62 & \mbox{0.60 {\scriptsize (0.01)}} & \mbox{0.34 {\scriptsize (0.01)}} \\

\midrule

\multicolumn{5}{l}{\textbf{Scheme B (+1, -1, +0.4)}} \\

\midrule

0.1 & reject-all & 0.00 & \mbox{0.00 {\scriptsize (0.00)}} & -- \\
0.2 & 0.12 & 0.26 & \mbox{0.25 {\scriptsize (0.01)}} & \mbox{0.17 {\scriptsize (0.01)}} \\
0.3 & 0.25 & 0.44 & \mbox{0.44 {\scriptsize (0.01)}} & \mbox{0.25 {\scriptsize (0.01)}} \\
0.4 & 0.63 & 0.62 & \mbox{0.61 {\scriptsize (0.01)}} & \mbox{0.35 {\scriptsize (0.01)}} \\

\midrule

\multicolumn{5}{l}{\textbf{I-CALM (+1, -1, +0.4)}} \\

\midrule

0.1 & reject-all & 0.00 & \mbox{0.00 {\scriptsize (0.00)}} & -- \\
0.2 & 0.06 & 0.17 & \mbox{0.16 {\scriptsize (0.01)}} & \mbox{0.14 {\scriptsize (0.02)}} \\
0.3 & 0.26 & 0.44 & \mbox{0.43 {\scriptsize (0.01)}} & \mbox{0.24 {\scriptsize (0.01)}} \\
0.4 & 0.74 & 0.60 & \mbox{0.58 {\scriptsize (0.01)}} & \mbox{0.34 {\scriptsize (0.01)}} \\

\bottomrule
\end{tabular}
}

\caption{Bonferroni CP-UCB threshold selection for GPT-5 mini on PopQA using a 20\%/80\% calibration--validation split. Validation FAR is computed among answers accepted by the threshold selected on the calibration split.}
\label{tab:risk_targets}
\end{table}

\subsubsection{Multistart fixed-sequence CP threshold selection}

The Bonferroni baseline in Algorithm~1 permits arbitrary post-selection over the full grid $\mathcal U$, but can be conservative when neighboring thresholds have highly correlated accepted sets. In our setting the acceptance sets are nested in $u$, and the empirical CFAR curves are often approximately nondecreasing: as $u$ increases, the acceptance set expands and the cumulative accepted-set FAR typically rises. This ordered structure motivates a fixed-sequence ordered-testing alternative \citep{wiens2003fixed,angelopoulos2025learn}. Rather than scanning only from the strictest threshold $u_1$, we use a multistart version that prespecifies several starting indices and runs a forward scan from each one.

Write the grid in increasing order as $\mathcal U=\{u_1<\cdots<u_M\}$, so that $u_1$ is the smallest uncertainty level and therefore the strictest acceptance rule. For a target FAR $r\in[0,1]$, define the null hypotheses
$$
H_j:R(u_j)\ge r,\qquad j=1,\ldots,M,
$$
where $R(u)=\Pr(E=1\mid U\le u)$ is as in Proposition~\ref{prop:cp_ucb}. The ordered-testing idea is to scan thresholds from lower to higher uncertainty and stop along a given path at the first non-rejection. To reduce sensitivity to poorly supported thresholds near $u=0$, we prespecify a small set of starting indices
$$
\mathcal J=\{j_1<\cdots<j_L\}\subseteq \{1,\ldots,M\},
$$
for example a coarse evenly spaced subset of the grid, and split the overall failure probability equally across starts. From each start $j_\ell\in\mathcal J$, we test the path
$$
u_{j_\ell},u_{j_\ell+1},\ldots,u_M
$$
at level $\delta/L$, continuing only while rejections occur.

For each $u_j$, the exact one-sided binomial p-value is
\begin{equation}
p_j=
\begin{cases}
1, & n(u_j)=0,\\
\Pr(B\le k(u_j)), & n(u_j)>0,
\end{cases}
\label{eq:ms_pvalue}
\end{equation}
\[
B\sim\mathrm{Binomial}(n(u_j),r).
\]
where $n(u_j)$ and $k(u_j)$ are defined in \eqref{eq:cp_counts}. Small $p_j$ indicates that the observed number of accepted errors is unusually low under the null $R(u_j)\ge r$, and therefore provides evidence that threshold $u_j$ is safe.

As in the Bonferroni baseline, the same decision can be expressed through
one-sided CP upper bounds. Since the one-sided CP interval inverts the exact
binomial test, define
\[
U_j := \mathrm{UCB}_{\mathrm{CP}}^{\mathrm{MS}}(u_j).
\]
Then $p_j\le \delta/L$ is equivalent to the following:
if $k(u_j)=n(u_j)$, then
\begin{equation}
U_j = 1;
\end{equation}
if $k(u_j)<n(u_j)$, then
\begin{multline}
U_j =
\operatorname{Beta}^{-1}\!\Bigl(
1-\delta/L;\, k(u_j)+1,\\
n(u_j)-k(u_j)
\Bigr).
\label{eq:ms_cp_ucb}
\end{multline}
A threshold $u_j$ is certified when $U_j\le r$.

Let $C_\ell(r)\subseteq \mathcal U$ denote the set of thresholds certified
along start $j_\ell$. The multistart certified set is
\[
\mathcal C_r^{\mathrm{MS}} = \bigcup_{\ell=1}^L C_\ell(r),
\]
and the selected deployment threshold is
\[
\hat u_r^{\mathrm{MS}} = \max \mathcal C_r^{\mathrm{MS}}.
\]
whenever the certified set is nonempty; otherwise the rule rejects all answers. The corresponding confidence threshold is $\hat\tau_r^{\mathrm{MS}}=1-\hat u_r^{\mathrm{MS}}$.

Because different starts can certify different forward segments, $\mathcal C_r^{\mathrm{MS}}$ need not form a single contiguous prefix of the grid. This does not affect validity: the deployment rule depends only on the selected threshold $\hat u_r^{\mathrm{MS}}$, and that selected threshold is itself one of the certified thresholds.

Under the idealized monotone-risk picture $R(u_1)\le \cdots \le R(u_M)$, each start explores the same safe-to-unsafe ordering, but from a different initialization point. Exact monotonicity is not required for validity; it only explains why multistart scans can recover useful thresholds when the earliest grid points are too sparse to support a single forward scan from $u_1$.

\begin{proposition}[Finite-sample FAR control via multistart fixed-sequence CP]
\label{prop:ms_cp}
Assume the setting of Proposition~\ref{prop:cp_ucb}, and write the prespecified grid as $\mathcal U=\{u_1<\cdots<u_M\}$. Let $\mathcal J=\{j_1<\cdots<j_L\}\subseteq\{1,\ldots,M\}$ be a prespecified set of starting indices, independent of the calibration data. For a target FAR $r\in[0,1]$, define $H_j:R(u_j)\ge r$, where $R(u_j)=\Pr(E=1\mid U\le u_j)$. For each start $j_\ell$, run fixed-sequence testing on the path
$$
H_{j_\ell},H_{j_\ell+1},\ldots,H_M
$$
at level $\delta/L$, continuing only while rejections occur, and let $C_\ell(r)$ be the set of thresholds certified along that path. Let
$$
\mathcal C_r^{\mathrm{MS}}=\bigcup_{\ell=1}^{L} C_\ell(r),
\qquad
\hat u_r^{\mathrm{MS}}=\max \mathcal C_r^{\mathrm{MS}}
$$
whenever $\mathcal C_r^{\mathrm{MS}}\neq\varnothing$, with reject-all otherwise. Then
$$
\Pr\!\bigl(\forall u\in \mathcal C_r^{\mathrm{MS}},\ R(u)\le r\bigr)\ge 1-\delta.
$$
Consequently, whenever a threshold is selected,
$$
R(\hat u_r^{\mathrm{MS}})\le r.
$$
If no threshold is selected, the reject-all rule trivially satisfies the risk constraint.
\end{proposition}

\begin{proof}[Proof sketch]
Fix a start $j_\ell$. If all null hypotheses on that path are false, then every threshold on the path is safe and there can be no false rejection. Otherwise, let $j_\ell^\star$ be the index of the first true null on the path
$$
H_{j_\ell},H_{j_\ell+1},\ldots,H_M.
$$
For each fixed $j$, conditional on $n(u_j)$, we have
$$
k(u_j)\mid n(u_j)\sim \mathrm{Binomial}(n(u_j),R(u_j)),
$$
so $p_j$ is a valid one-sided p-value for $H_j$. Under fixed-sequence testing along the $\ell$-th path, any false rejection implies rejection of $H_{j_\ell^\star}$. Therefore
$$
\Pr(\text{any false rejection on path } \ell)\le \Pr(p_{j_\ell^\star}\le \delta/L)\le \delta/L.
$$
Applying a union bound across the $L$ prespecified starts yields
$$
\Pr(\text{any false rejection on any path})
\le
\sum_{\ell=1}^{L} \delta/L
=
\delta.
$$
Hence, with probability at least $1-\delta$, every threshold certified on any path is safe, which implies $R(\hat u_r^{\mathrm{MS}})\le r$ whenever a threshold is selected.
\end{proof}

\begin{tcolorbox}[
  enhanced jigsaw,
  colback=gray!5,
  colframe=gray!40,
  title=Algorithm 2: Multistart fixed-sequence CP threshold selection,
  boxsep=3pt,
  left=3pt,
  right=3pt,
  top=3pt,
  bottom=3pt,
  parskip=0pt,
  breakable,
  break at=-\baselineskip/0pt,
  height fixed for=middle,
  pad at break=1mm
]
\textbf{Input:} calibration set $\{(u_i,e_i)\}_{i=1}^n$, target FAR $r\in[0,1]$, overall failure probability $\delta\in(0,1)$, ordered grid $\mathcal U=\{u_1<\cdots<u_M\}$, and prespecified start set $\mathcal J=\{j_1<\cdots<j_L\}\subseteq\{1,\ldots,M\}$.

Initialize the certified set as $\mathcal C\leftarrow \varnothing$.

\textbf{For each start} $j_\ell\in\mathcal J$:
\begin{enumerate}[leftmargin=1.8em,itemsep=0.2em]
    \item Set $j\leftarrow j_\ell$.
    \item While $j\le M$:
    \begin{enumerate}[leftmargin=1.8em,itemsep=0.2em]
        \item Compute the number of accepted calibration answers
        $n(u_j)=\sum_{i=1}^n \mathbf{1}\{u_i\le u_j\}$.
        
        \item Compute the number of incorrect accepted calibration answers
        $k(u_j)=\sum_{i=1}^n \mathbf{1}\{u_i\le u_j\}e_i$.
        
        \item Compute either the exact one-sided p-value $p_j$ in \eqref{eq:ms_pvalue} or, equivalently, $\mathrm{UCB}_{\mathrm{CP}}^{\mathrm{MS}}(u_j)$ in \eqref{eq:ms_cp_ucb}.
        \item If $p_j\le \delta/L$ (equivalently, $\mathrm{UCB}_{\mathrm{CP}}^{\mathrm{MS}}(u_j)\le r$), add $u_j$ to $\mathcal C$ and set $j\leftarrow j+1$.
        \item Otherwise stop the current path and move to the next start.
    \end{enumerate}
\end{enumerate}

\textbf{Return}
$$
\hat u_r^{\mathrm{MS}}=\max \mathcal C
$$
if $\mathcal C\neq\varnothing$; otherwise return a reject-all rule.

Define the corresponding confidence threshold as
$\hat\tau_r^{\mathrm{MS}}=1-\hat u_r^{\mathrm{MS}}$
whenever $\hat u_r^{\mathrm{MS}}$ exists.

\textbf{Deployment rule:} accept the model output iff $\tau^{self}\ge \hat\tau_r^{\mathrm{MS}}$ (equivalently, $u\le \hat u_r^{\mathrm{MS}}$); otherwise reject or hold the answer for further review.
\end{tcolorbox}

Relative to Algorithm~1, the multistart rule trades arbitrary post-selection over the full grid for a family of prespecified ordered scans. It can recover certified thresholds even when the smallest uncertainty levels are too sparsely supported to justify a forward scan from $u_1$. The trade-off is an additional start-wise correction $\delta/L$: as $L$ increases, the rule becomes more conservative, although for $L\ll M$ it is still typically much less conservative than the Bonferroni baseline in Algorithm~1. In practice, a small coarse set of starts (for example, 5--10 roughly evenly spaced grid points) is a reasonable default, and we report one representative choice below.

Because the underlying empirical CFAR curves are unchanged from Figures~\ref{fig:gpt5_penalty}--\ref{fig:cfar_train}, we summarize the multistart rule in table form rather than adding a second calibration plot. Table~\ref{tab:msf_risk_targets} reports results for a representative choice $L=10$ on the same 20\%/80\% calibration--validation split as Table~\ref{tab:risk_targets}. This provides a moderate number of starts without making the per-start level $\delta/L$ too small.

On this split, the validation FAR among accepted answers remains below the target in every nontrivial row. Relative to the Bonferroni baseline in Table~\ref{tab:risk_targets}, the multistart rule increases coverage most clearly at stricter risk targets for Scheme A and Scheme B. For example, under Scheme A at $r=0.2$, the validation acceptance rate increases from $0.1965$ to $0.2322$; under Scheme B at $r=0.2$, it increases from $0.2522$ to $0.2697$. At $r=0.3$ and $r=0.4$, the gains are smaller but still generally positive.  I-CALM benefits less from multistart scanning: at $r=0.2$ it still returns reject-all, and at $r=0.4$ it matches the Bonferroni threshold exactly. The practical takeaway is that multistart scanning can recover additional low-risk coverage when the strictest uncertainty region is statistically thin, but its benefit depends on the shape of the risk curve and on how much calibration mass lies just beyond the earliest grid points.

\begin{table}[H]
\centering
{\fontsize{9pt}{9pt}\selectfont
\setlength{\tabcolsep}{4pt}

\begin{tabular}{ccccc}
\toprule

\shortstack{\textbf{Risk}\\\textbf{Target}}
& \shortstack{\textbf{Calibrated}\\\textbf{Uncertainty}\\\textbf{Threshold}}
& \shortstack{\textbf{Calibration}\\\textbf{Acceptance}\\\textbf{Rate}}
& \shortstack{\textbf{Validation}\\\textbf{Acceptance}\\\textbf{Rate}}
& \shortstack{\textbf{Validation}\\$\widehat{\textbf{CFAR}}$} \\

\midrule

\multicolumn{5}{l}{\textbf{Scheme A (+1, -1)}} \\

\midrule

0.1 & reject-all & 0.00 & \mbox{0.00 {\scriptsize (0.00)}} & -- \\
0.2 & 0.11 & 0.24 & \mbox{0.23 {\scriptsize (0.01)}} & \mbox{0.15 {\scriptsize (0.01)}} \\
0.3 & 0.25 & 0.45 & \mbox{0.44 {\scriptsize (0.01)}} & \mbox{0.25 {\scriptsize (0.01)}} \\
0.4 & 0.56 & 0.63 & \mbox{0.62 {\scriptsize (0.01)}} & \mbox{0.35 {\scriptsize (0.01)}} \\

\midrule

\multicolumn{5}{l}{\textbf{Scheme B (+1, -1, +0.4)}} \\

\midrule

0.1 & reject-all & 0.00 & \mbox{0.00 {\scriptsize (0.00)}} & -- \\
0.2 & 0.13 & 0.27 & \mbox{0.27 {\scriptsize (0.01)}} & \mbox{0.18 {\scriptsize (0.01)}} \\
0.3 & 0.26 & 0.45 & \mbox{0.45 {\scriptsize (0.01)}} & \mbox{0.26 {\scriptsize (0.01)}} \\
0.4 & 0.67 & 0.64 & \mbox{0.63 {\scriptsize (0.01)}} & \mbox{0.36 {\scriptsize (0.01)}} \\

\midrule

\multicolumn{5}{l}{\textbf{I-CALM (+1, -1, +0.4)}} \\

\midrule

0.1 & reject-all & 0.00 & \mbox{0.00 {\scriptsize (0.00)}} & -- \\
0.2 & reject-all & 0.00 & \mbox{0.00 {\scriptsize (0.00)}} & -- \\
0.3 & 0.27 & 0.45 & \mbox{0.43 {\scriptsize (0.01)}} & \mbox{0.25 {\scriptsize (0.01)}} \\
0.4 & 0.74 & 0.60 & \mbox{0.58 {\scriptsize (0.01)}} & \mbox{0.34 {\scriptsize (0.01)}} \\

\bottomrule
\end{tabular}
}

\caption{Multistart fixed-sequence CP threshold selection with $L = 10$ for GPT-5 mini on PopQA using a 20\%/80\% calibration--validation split. Validation FAR is computed among answers accepted by the threshold selected on the calibration split.}
\label{tab:msf_risk_targets}
\end{table}

\end{document}